\documentclass[12pt]{article}
\usepackage{amsmath}
\usepackage{amssymb}
\usepackage{mathtools}
\usepackage{amsthm}
\usepackage{graphicx}
\usepackage{natbib}
\usepackage[resetlabels]{multibib}
\usepackage{url} 

\newcites{Supp}{Supplementary References}

\usepackage{subfigure}
\usepackage{algorithm}
\theoremstyle{plain}
\newtheorem{theorem}{Theorem}[section]

\theoremstyle{definition}

\theoremstyle{remark}
\newtheorem{remark}[theorem]{Remark}
\usepackage{thmtools,thm-restate} 

\usepackage{xcolor}

\usepackage[noend]{algorithmic}
\newcommand{\blind}{0}

\begin{document}

\def\spacingset#1{\renewcommand{\baselinestretch}%
{#1}\small\normalsize} \spacingset{1}


\if0\blind
{
  \title{\bf Bayesian Federated Learning with Hamiltonian Monte Carlo: Algorithm and Theory}


  \author{Jiajun Liang\thanks{ByteDance Inc, China}
    \and
    Qian Zhang \thanks{Department of Statistics, Purdue University, West Lafayette, IN}
    \and
    Wei Deng \thanks{Machine Learning Research, Morgan Stanley, New York, NY}
    \and
    Qifan Song \thanks{Department of Statistics, Purdue University, West Lafayette, IN}
    \and 
    Guang Lin \thanks{Department of Mathematics \& School of Mechanical Engineering, Purdue University, West Lafayette, IN. G.L. gratefully acknowledges the support of the National Science Foundation (DMS-2053746, DMS-2134209, ECCS-2328241, and OAC-2311848), and U.S. Department of Energy (DOE) Office of Science Advanced Scientific Computing Research program DE-SC0023161, and DOE – Fusion Energy Science,  under grant number: DE-SC0024583. }
   }

  \maketitle
} \fi


\if1\blind
{
  \bigskip
  \bigskip
  \bigskip
  \begin{center}
    {\LARGE\bf Title}
\end{center}
  \medskip
} \fi

\bigskip
\begin{abstract}
This work introduces a novel and efficient Bayesian federated learning algorithm, namely, the Federated Averaging stochastic Hamiltonian Monte Carlo (FA-HMC), for parameter estimation and uncertainty quantification. We establish rigorous convergence guarantees of FA-HMC on non-iid distributed data sets, under the strong convexity and Hessian smoothness assumptions. Our analysis investigates the effects of parameter space dimension, noise on gradients and momentum, and the frequency of communication (between the central node and local nodes) on the convergence and communication costs of FA-HMC. Beyond that, we establish the tightness of our analysis by showing that the convergence rate cannot be improved even for continuous FA-HMC process. Moreover, 
extensive empirical studies demonstrate that FA-HMC outperforms the existing Federated Averaging-Langevin Monte Carlo (FA-LD) algorithm.
\end{abstract}

\noindent%
{\it Keywords:}  Hamiltonian Monte Carlo, Federated Learning, Bayesian sampling, Federated averaging, Stochastic Gradient Langevin Dynamics
\vfill

\newpage
\spacingset{1.5} 

\section{Introduction}
Standard learning algorithms usually require centralizing the training data, in the sense that the learning machine can directly access all pieces of the data. Federated learning (FL), on the other hand, enables multiple parties to collaboratively train a consensus model without directly sharing confidential data \citep{konevcny2015federated,konevcny2016federated,bonawitz2019towards,li2020federated}. The framework of FL is quite appealing to applications where data confidentiality is of vital importance, such as aggregating user app data from mobile phones to learn a shared predictive model \citep[e.g.,][]{tran2019federated,chen2020joint} or analyzing medical data from multiple healthcare stakeholders (e.g., hospitals, research centers, life science companies) \citep[e.g.,][]{li2020multi,rieke2020future}.

FL shares a similar algorithmic architecture to parallel optimization. First, parallel algorithms are commonly based on the divide-and-combine strategy, i.e., the learning system assigns (usually i.i.d.) training samples to each worker node, say via simple random sampling. As such,  the training data sets are similar in nature across worker nodes. But under the FL framework, the data sets of each worker node are generated or collected locally and are not homogeneous, which poses challenges for convergence analysis. Secondly, parallel computing is commonly practiced in the same physical location, such as a data center, where high throughput computer networking communications are available between worker nodes. In contrast, FL has either a vast number of worker nodes (e.g., mobile devices) or geographically separated worker nodes (e.g., hospitals), which limits the connectivity between the central nodes and worker nodes. Due to the unavailability of fast or frequent communication, FL needs to be communication-efficient.

Federated Averaging \citep[FedAvg,][]{mmr+17} is one of the most widely used  FL optimization algorithms. It trains a global model by synchronously averaging multi-step local stochastic gradient descent (SGD) updated parameters of all the worker nodes.  Various attempts have been made to enhance the robustness and efficiency of FedAvg 
\citep[e.g.,][]{LS20,wang2020federated}. 
However, optimization-based approaches often fail to provide proper uncertainty quantification for their estimations. Reliable uncertainty quantification, such as interval estimations or hypothesis testing, provides a vital diagnostic for both developers and users of an AI system. 

The Bayesian counterpart naturally integrates an inference component, thus it provides a unified solution for both estimations and uncertainty quantification. This paper studies a Bayesian computing algorithm aiming to obtain samplers from the global posterior distribution by infrequently aggregating samples drawn from local posterior distributions. 
Unlike existing results that utilize stochastic gradient Langevin dynamics \citep{Welling11}, this work considers (stochastic gradient) Hamiltonian Monte Carlo \citep[HMC,][]{Neal12}.  While the second-order nature of HMC poses more theoretical difficulties, it has been demonstrated to be more computationally efficient through numerous empirical studies \citep[see, e.g.,][]{girolami2011riemann,chen2014stochastic}.
Readers can refer to Section A 
in Supplementary Material for a review of related literature on federated sampling and HMC.


The contributions of the presented work are three-fold:


(1) We propose the  Federated Averaging Hamiltonian Monte Carlo (FA-HMC) algorithm which is effective  for global posterior inferences in federated learning. It utilizes stochastic gradient HMC on individual local nodes and combines the local samples obtained infrequently to yield global samples.

 (2) Under strong log-concavity and proper smoothness assumptions, we have proven a non-asymptotic convergence result under the Wasserstein metric for various training settings. 
 Furthermore, we demonstrate that this upper bound of the convergence rate of the FA-HMC sampling algorithm  is tight (i.e., best achievable for certain sampling problems).

(3) We conduct simulation and real data experiments to validate our theoretical findings. Additionally, the numerical studies  show that FA-HMC is easy to tune, improves communication efficiency, and can outperform FA-LD in different settings.
\paragraph{Roadmap:}  The paper is organized as follows: In Section \ref{prelim}, we summarize   the problem setup and provide the necessary background on HMC. In Section  \ref{section_fa-HMC}, we present the FA-HMC algorithm and the assumptions used for its analysis. In Section \ref{sec:theory}, we provide the key theoretical findings and examine the effects of SGD noise and the correlation between momentum. Furthermore, we prove that our analysis is tight and cannot be improved for certain sampling problems, even for continuous FA-HMC. In Section \ref{sec:simu}, we compare the FA-HMC algorithm with the FA-LD algorithm through extensive simulations and real-data experiments. Finally, in Section \ref{section_conclude}, we conclude our work and suggest potential future directions.


\section{Preiminary}
\label{prelim}

\subsection{Problem Setup}\label{sec:fald}

Let $z^c_{i}, 1\leq i\leq n_c$ be the available data of the $c$-th node and $\ell(\theta; z^c_{i})$ be a user-specified negative log-likelihood function. Define $n=\sum n_c$, $w_c=n_c/n$, and $f^{(c)}(\theta):=n\sum_{i=1}^{n_c}\ell(\theta; z^c_{i})/n_c$ is the local loss function of parameter $\theta\in \mathbb{R}^d$ accessible to the $c$-th local node (e.g., the normalized negative log-likelihood function based on the data set available at $c$-th local node) for $1\leq c\leq N$. The goal is to simulate the global target distribution $\pi(\theta)\propto \exp(-{f(\theta)})$, where $f(\theta)=\sum_{c=1}^N w_cf^{(c)}(\theta)$, $w_c\geq 0$ and $\sum_c w_c = 1$.

\subsection{Hamilton's Equations and HMC}
Hamiltonian (Hybrid) Monte Carlo (HMC) was first proposed by \cite{Duane87} for simulations of quantum chromodynamics and was then extended to molecular dynamics and neural networks \cite{Neal12}. To alleviate the random-walk behavior in the vanilla Langevin dynamics, HMC simulates the trajectory of a particle according to Hamiltonian dynamics and obtains a much faster convergence rate than Langevin dynamics \cite{Mangoubi18_leapfrog}. In specific,  HMC introduces a set of auxiliary momentum variables $p\in\mathbb{R}^d$ to capture second-order information, whereas Langevin Monte Carlo is only a first-order method. In this way, HMC generates samples from the following joint distribution
\begin{equation*}
    \pi(\theta, p)\propto \exp(-{f(\theta)}- \frac{1}{2}p'\Sigma^{-1}p),
\end{equation*}
where ${f(\theta)}+ p'\Sigma^{-1}p/2$ is the Hamiltonian function and quantifies the total energy of a physical system.
To further generate more efficient proposals, HMC simulates according to the following Hamilton's equations
\begin{align}\label{eq:hmc}
\begin{split}
    \frac{d\theta(t)}{dt}=\Sigma^{-1/2}{p(t)},\quad\frac{dp(t)}{dt}=-\Sigma^{1/2}\nabla_\theta f(\theta(t)),
\end{split}
\end{align}
which satisfy the conservation law and are time reversible. Such properties leave the distribution invariant and the nature of Hamiltonian conservation always makes the proposal accepted ideally. Note that commonly, one chooses $\Sigma =\mathbb{I}_d$ such that the momentum follows the standard multivariate normal distribution.

To numerically implement the continuous HMC process, a popular numerical integrator is the ``leapfrog'' approximation, see Algorithm \ref{alg:disHMC}. Here, to enhance the computational efficiency, $\nabla\widetilde{f}(\theta_k,\xi_k)$ and $\nabla\widetilde{f}(\theta_{k+1},\xi_{k+1/2})$ are the stochastic versions of $\nabla f(\theta_k)$ and $\nabla f(\theta_{k+1})$, respectively. The arguments $\xi_k$ and $\xi_{k+1/2}$ denote random variables that control the randomness of the stochastic gradients. For example, given $f(\theta)=\sum_{i=1}^{n}\ell(\theta;z_i)$ with data $\{z_i\}_{i=1}^n$, we let $\nabla \widetilde{f}(\theta,\xi_k)=n\sum_{i \in S(\xi_k)}\nabla\ell(\theta;z_i)/|S(\xi_k)|+Z(\xi_k)$ where  $S(\xi_k)$ is a random index subset, $Z(\xi_k)$ is an injected Gaussian noise, and $\xi_k$ is the random seed. When the exact gradients are used, it holds that $\nabla\widetilde{f}(\theta_k,\xi_k)=\nabla f(\theta_k)$ and $\nabla\widetilde{f}(\theta_{k+1},\xi_{k+1/2})=\nabla f(\theta_{k+1})$. Note that throughout this paper, when the exact gradient is used instead of a stochastic gradient, the algorithm is referred to as the vanilla version, e.g., {\it vanilla} FA-HMC.

\begin{algorithm}[tb]
	\caption{Stochastic gradient leapfrog approximation $\widetilde{h}_{\rm LF}$ }\label{alg:disHMC}
	\begin{algorithmic}
		\STATE {\bfseries Input:} Energy function $f(\cdot) $; Initial parameters $\theta_0$, momentum $p_0$; learning rate $\eta$; leapfrog step $K$; $k=0$
	       \WHILE{$k\leq K$:}
   \STATE    $\theta_{k+1}=\theta_k+\eta p_k-\frac{\eta^2}{2}\nabla\widetilde{f}(\theta_k,\xi_k)$  
        
   \STATE     $p_{k+1}=p_k-\frac{\eta}{2} \nabla\widetilde{f}(\theta_k,\xi_k)-\frac{\eta}{2}\nabla\widetilde{f}(\theta_{k+1},\xi_{k+\frac{1}{2}})$
     \STATE      $k=k+1;$
 \ENDWHILE
 \STATE {\bfseries Output:} $\widetilde{h}_{\rm LF}(f,\theta_0,p_0,\eta,K)=\theta_K$
	\end{algorithmic}
\end{algorithm}
For convenience in analysis, the leapfrog method without Metropolis correction (see   Algorithm \ref{HMC_Standard}), is commonly studied in the literature \citep{Mangoubi18_leapfrog, chen2019optimal, zou2021convergence}. One may also add an additional accept/reject step according to the Metropolis ratio \citep{chen2020fast}.
\begin{algorithm}[ ]
	\caption{HMC algorithm (without Metropolis correction)}\label{HMC_Standard}
	\begin{algorithmic}
		\STATE {\bfseries Input:} Energy function $f(\cdot)$; Initial point $\theta_0$; Stepsize function $\eta_t=\eta(t)$; Leapfrog step $K$; $t=0$;
       \WHILE{ the stopping rule is not satisfied }
   \STATE {\bf sample }momentum $p_t\sim N(0,\mathbb{I}_d)$
        \STATE {\bf update}         	$\theta_{t+1}=\widetilde{h}_{\rm LF}(f,\theta_t,p_t, \eta_t, K)$, $t=t+1$;
 \ENDWHILE
 \STATE {\bfseries Output:} $\{\theta_i\}_{i=1}^{t}$
	\end{algorithmic}
\end{algorithm}

Note that in the literature, \cite{chen2014stochastic} proposed a different HMC algorithm, based on Euler integrator of Hamilton dynamics. Their implementation includes variance adjustment to counteract the noise of the stochastic gradient, which can negatively impact the stationary distribution. This adjustment eventually leads to an underdamped Langevin Monte Carlo algorithm with stochastic gradient \citep[see also e.g.,][]{ma2015complete, zou2019stochastic,chau2019stochastic, akyildiz2020nonasymptotic, nemeth2021stochastic}. 

\section{FA-HMC Algorithm and Assumptions}\label{section_fa-HMC}
Ensuring the confidentiality of the data utilized for training a model is a vital concern in federated learning. 
To safeguard against potential gradient leakage \citep{NEURIPS2019_60a6c400} and breaches of local data privacy,  it is preferable to use  noisy gradients and less-correlated momentum among local nodes \citep[see][]{deng2021convergence,Maxime2021}. This could  make it more difficult to recover local data information through accumulated communication. 

With these considerations, we propose  Federated Averaging via HMC algorithm that utilizes general stochastic gradients and non-necessarily identical  momentum across nodes. We let all local devices run HMC (Algorithm \ref{alg:disHMC}), and synchronize their model parameters every $T$ iteration. All devices may use stochastic gradients and share part of the initial momentum of leapfrog approximation. Note that in practice, correlated momentum between devices can be easily achieved by sending a common random seed to all devices for momentum generation. This FA-HMC algorithm is formalized in Algorithm \ref{alg:disFAHMC}.
\begin{algorithm}[tb]
	\caption{FA-HMC algorithm}\label{alg:disFAHMC}
	\begin{algorithmic}
		\STATE {\bfseries Input:} $\theta^{(c)}_0=\theta_0$, $t=0$; stepsize function $\eta_t=\eta(t)$; Local update step $T$; leapfrog update step $K$; 
		\WHILE{the stopping rule is not satisfied}
  \STATE {\bf sample }momentum $p_t^{(c)}$
  \IF{$t\equiv 0 (\mathrm{mod }\ T)$}
  \STATE Broadcast $\theta_t:=\sum_{c=1}^Nw_c\theta^{(c)}_t$ and set $\theta^{(c)}_{t+1,0}=\theta_t$
  \ELSE
  \STATE $\theta^{(c)}_{t+1,0}=\theta^{(c)}_t$
  \ENDIF
  \STATE {\bf update} $\theta^{(c)}_{t+1}=\widetilde{h}_{\rm LF}(f^{(c)},\theta^{(c)}_{t+1,0},p_t^{(c)},\eta_t,K) $ in parallel for all devices, $t=t+1$
  \ENDWHILE
	\end{algorithmic}
\end{algorithm}
It is worth mentioning that when leapfrog step $K=1$, the leapfrog approximation of the unadjusted HMC algorithm (i.e., Algorithm \ref{alg:disHMC}) reduces to 
$\theta_{t+1}=\theta_{t}-(\eta_t^2/2)\nabla_{\theta}\widetilde{f}(\theta_t,\xi_t)+\eta_t N(0,\mathbb{I}_d)$, 
which is exactly the unadjusted Langevin Monte Carlo with dynamic learning rate $\eta_t^2/2$. And the FA-HMC reduces to FA-LD \cite{deng2021convergence}.

\subsection{Assumptions}
To establish the convergence performance of the aggregated model with respect to $\theta_{t}$, we adopted the following assumptions.

\begin{restatable}[$\mu$-Strongly Convex]{assum}{assumconvex}\label{assum:convex}
For each $c=1,2,\ldots, N$, $f^{(c)}$ is $\mu$-strongly convex for some $\mu>0$, i.e., $\forall x, y \in \mathbb{R}^{d}$, 
$f^{(c)}(y)\geq f^{(c)}(x)+\langle\nabla f^{(c)}(x), y-x\rangle+\frac{\mu}{2}\|y-x\|_{2}^{2}.$ 

\end{restatable}
\vskip -0.2in
\begin{restatable}[$L$-Smoothness]{assum}{assumsmooth}\label{assum:smooth}
 For each $c=1,2,\ldots, N$, $f^{(c)}$ is $L$-smooth for some $L>0$, i.e., $\forall x, y \in \mathbb{R}^{d}$, $\|\nabla f^{(c)}(y)-\nabla f^{(c)}(x)\|\leq L\|x-y\|.$
\end{restatable}
\vskip -0.2in
\begin{restatable}[$L_H$-Hessian Smoothness]{assum}{assumHsmooth}\label{assum:Hsmooth} For each $c=1,2,\ldots, N$, $f^{(c)}$ is $L_H$ Hessian smoothness, i.e., for any  $\theta_1,\theta_2, p\in\mathbb{R}^d$,
$
\|\big(\nabla^2 f^{(c)}(\theta_1)-\nabla^2 f^{(c)}(\theta_2)\big)p\|^2\leq L_H^2\|\theta_1-\theta_2\|^2\|p\|_{\infty}^2.
$
\end{restatable}

Assumptions~\ref{assum:convex}-\ref{assum:smooth} are commonly used for the convergence analysis of gradient-based MCMC algorithms \citep[e.g.,][ and references therein]{dk17,Mangoubi18_leapfrog,dk19,erdogdu2021convergence}. 
The strong convexity condition, in some theoretical literature of stochastic Langevin Monte Carlo, has also been relaxed to the dissipativity condition \citep[e.g.,][]{Maxim17,zou2021faster} for non-log-concave target distributions. But such an extension is beyond the scope of this paper and will be investigated in future works.
Assumption \ref{assum:Hsmooth} ensures second-order smoothness of energy functions beyond gradient Lipchitzness.  Similar Hessian smoothness conditions are used in the literature. For example, \cite{dk19,chen2020fast,zou2021faster} required the Hessian matrix of energy function to be Lipchitz under $\ell_2$ operator norm. In comparison, Assumption \ref{assum:Hsmooth} is a stronger requirement since $\ell_{\infty}$ norm appears on the RHS.
Our assumption is somewhat comparable to Assumption 1 of \cite{Mangoubi18_leapfrog} which defines a semi-norm with respect to a set of pre-specified unit vectors. 

 We require an additional assumption to model stochastic gradients. Denote $\theta^{(c)}_{t,k}$ as the position parameter of the $c$-th local node at iteration $t$ and leapfrog step $k$, and $\xi^{(c)}_{t,x}$ ($x=k-1/2,k$) as the corresponding variable that controls the randomness of gradient.

\begin{restatable}[$\sigma_g$-Bounded Variance]{assum}{assumboundvar}\label{assum:boundvar}For local device $c=1,2,\ldots,N$, and leapfrog step $k=1,2,\ldots,K$, $t=1,2,\ldots$, we have $\max_{x=k-1/2,k}\mathrm{tr(Var}(\nabla \widetilde{f}^{(c)}(\theta^{(c)}_{t,k},\xi^{(c)}_{t,x})|\theta^{(c)}_{t,k}))\leq \sigma_g^2Ld$,
 for some $\sigma_g>0$.
\end{restatable}
This is a common assumption in the literature \citep[see][]{gurbuzbalaban2021decentralized,Maxime2021,deng2021convergence}. It is worth noting that in practice, the stochastic gradient is computed based on a random subsample of the whole dataset, thus the variability of the stochastic gradient can be naturally controlled by adjusting the batch sizes.

Under our framework, we can also relax the above assumption to 
\[
\max_{x=k-1/2,k}\text{tr(Var}(\nabla \widetilde{f}^{(c)}(\theta^{(c)}_{t,k},\xi^{(c)}_{t,x})|\theta^{(c)}_{t,k}))\leq \sigma_g^2(G_{t,k}^{(c)}+d),
\]
without significant changes to our proof, where $G_{t,k}^{(c)}$ denote $\|\nabla f^{(c)}(\theta^{(c)}_{t,k})\|^2$.  
The extension of the proof to accommodate this assumption is discussed in Section K 
in the appendix.


Before presenting our main result, we emphasize that this paper examines the convergence of the FA-HMC sampling algorithm, specifically in regard to dimension $d$ and error $\epsilon$. It also explores ways to adjust the algorithm to maintain its effectiveness when considering variations in gradient and momentum noise.  Adapting the FA-HMC algorithm to more general settings like non-convexity will be our future study. 

\section{Theoretical Results}\label{sec:theory}
In Section \ref{sec:theo:main}, we describe the general convergence rate of FA-HMC on different settings and point out the setting where FA-HMC achieves the fastest speed and least communication cost. In Section C of the supplementary material, 
we argue that that the upper bound on the nearly ideal case is tight by giving a matching lower bound result. In Section \ref{sec:theo:conanalysis}, we present a detailed result of the convergence behavior of the FA-HMC algorithm.

\subsection{Main Results}\label{sec:theo:main}
Define $\theta^*:=\mathrm{argmin}_{\theta} f(\theta)$ and denote the marginal distribution of $\theta_t$ by $\pi_t$. 
Given two probability measures $\mu$ and $\nu$, the $2$-Wasserstein distance is  $\mathcal{W}_2(\mu,\nu)=\inf_{X\sim \mu,Y\sim\nu}(\mathbb{E}\|X-Y\|^2)^{1/2}$. 
The following theorem describes the general convergence rate of FA-HMC.
\begin{restatable}[]{theorem}{thmconvergratesg}\label{thm:convergrate:sg}
Assume \ref{assum:convex}-\ref{assum:boundvar}, and ${\cal W}_2(\pi_0,\pi)^2=O\footnote{\label{fn:bigO} As $d\rightarrow \infty$, we say $f=O(g)$ if $f\leq C g$ for some constant $C$, and say $f=\widetilde{O}(g)$ for $C$ being a  polynomial of $\log(d)$. } (d)$ and $\sum_{c=1}^Nw_c\|\nabla f^{(c)}(\theta^*)\|^2=O(d)$. For a given local iteration step $T$, there exists some constant $C$ depending on $L,L/\mu,L^2_H/L^3$ such that if we choose $\eta(t)\equiv \eta$  and (denote $\gamma=(K\eta)^2$)
\[
\eta^2 = \frac{\gamma}{K^2} =C \min\Bigl\{\frac{1}{K^2L},\frac{\epsilon}{K^2\sqrt{d}T},\frac{\epsilon^2}{K^2dT^2(1-\rho)N},\frac{\epsilon^2}{Kd\sum_{c=1}^Nw_c^2\sigma^2_g}\Bigr\}
\]
then  ${\cal W}_2(\pi_{t_\epsilon},\pi)\leq \epsilon$ for any $\epsilon>0$, with iteration number
\[
t_\epsilon=\frac{d\log(d/\epsilon^2)}{\epsilon^2}\widetilde{O}^{\ref{fn:bigO}}\Bigl(T^2\big(\gamma+(1-\rho)N\big)+\frac{\sum_{c=1}^Nw_c^2\sigma^2_g}{K}\Bigr)
\]
and corresponding communication times
\[
\frac{t_{\epsilon}}{T}=\frac{d\log(d/\epsilon^2)}{\epsilon^2}\widetilde{O}\Bigl(T\big(\gamma+(1-\rho)N\big)+\frac{\sum_{c=1}^Nw_c^2\sigma^2_g}{KT}\Bigr).
\]
\end{restatable}
When one uses small batch stochastic gradients (i.e., large $\sigma_g$)  or less correlated momentum (i.e, small $\rho$) to improve computational feasibility and protect privacy, the proposed $\gamma$ is negligible. Under this scenario, the required number of iterations is of rate $\widetilde{O}( d/\epsilon^2)$ with respect to the dimension $d$ and precision level $\epsilon$. 

\begin{remark}
Regarding the stopping rule of algorithms \ref{HMC_Standard} and \ref{alg:disFAHMC}, Theorem 4.1 does provide a nonasymptotic choice of $t_\epsilon$ to achieve an $\epsilon$-$W_2$ error in theory. 
But this bound is impractical, as it relies on the unknown distributional properties of the target distribution. For more practical rules, various suggestions have been made in the literature \citep[e.g., ][]{gelman1995bayesian}. For example,
(i) From a visual inspection perspective, we can randomly pick some dimensions and visually compare the trace plots between two parts of a single chain (by splitting one chain in half) or between two chains. We keep running the chains until they become ``approximately" stationary; (ii) From a quantitative perspective, we can compute the between- and within-sequence variances following the potential scale reduction factor 
$\widehat R$ defined in Eq.(11.4) of \citet{gelman1995bayesian}, the stopping rule can be triggered when $\widehat R \approx 1$.
Note that it is beyond the scope of this paper to design a stopping rule with statistical guarantees.

\end{remark}


The result of Theorem \ref{sec:theo:main} also shows that for a fixed $\epsilon$, under proper tuning, the communication cost $t_\epsilon/T$ may initially decrease and then increase as the number of local HMC iteration steps $T$ increases (i.e., a `U' curve w.r.t, $T$). Therefore, there is a trade-off between communication and divergence,  and  an optimal choice for local iteration can be made. Similar discoveries were also argued by \cite{deng2021convergence} for Bayesian Federated Averaging Langevin system. The above results provide a certain level of direction for optimizing the performance of FA-HMC algorithms, considering any well-defined federated learning loss that accounts for total running time, overall communication cost, and divergence. 

For instance, by reducing the noise of the stochastic gradients and improving correlation between momentum to a certain level, we can achieve significant improvement on the convergence speed from $\widetilde{O}(d/\epsilon^2)$ to $\widetilde{O}(\sqrt{d}/\epsilon)$, which is argued by the following proposition.
\begin{restatable}[]{proposition}{propconvergrate}\label{prop:convergrate}
With the assumptions as stated in Theorem \ref{thm:convergrate:sg}, 
if we choose $\eta(t)\equiv \eta$ and (denote $\gamma=(K\eta)^2$)
\begin{equation}\label{eq:vansetting}
  \gamma=C\min\Bigl\{\frac{1}{L}, \frac{\epsilon}{T\sqrt{d}}\Bigr\},\quad \rho=1-O(\frac{\gamma}{N}),\quad \sigma_g^2=O(K\gamma)  
\end{equation}
then 
it achieves that ${\cal W}_2(\pi_{t_\epsilon},\pi)\leq \epsilon$, where $\pi_t$ denotes the marginal distribution of $\theta_t$, with iteration $t$ and corresponding communication times $t_\epsilon/T$ as
\[
t_\epsilon=\frac{\sqrt{d}\log(d/\epsilon^2)}{\epsilon}\widetilde{O}(T),\qquad \frac{t_{\epsilon}}{T}=\widetilde{O}\big(\frac{\sqrt{d}\log(d/\epsilon^2)}{\epsilon}\big).
\]
\end{restatable}
Under the setting (\ref{eq:vansetting}), referred to as vanilla FA-HMC, the obtained convergence rate matches that of the underdamped Langevin Monte Carlo algorithm on a single device in \cite{ccbj18} and is superior to that of Federated Averaging of underdamped Langevin Monte Carlo algorithm under decentralized setting \citep[i.e., rate $\widetilde{O}(d/\epsilon^2)$ in][]{gurbuzbalaban2021decentralized}. It also matches existing results about Federated Langevin algorithm tackling heterogeneity under the federated learning framework \cite{plassier2022federated} and is better than those without hessian smoothness assumption \cite{deng2021convergence}. 

Furthermore, in Section C 
of the supplementary material, we establish a lower bound for $t_\epsilon =\Omega(\sqrt{d}T\log(d/\epsilon)/\epsilon)$ for some log-concave target distribution. In other words, our result in Proposition \ref{prop:convergrate} is tight w.r.t. dimension $d$ and local iteration $T$. This tight result implies that 
(1) Unlike the ``U'' curve with respect to $T$ discovered in Theorem \ref{thm:convergrate:sg}, when there are small stochastic gradients and large correlations between momentum, communication times have limited variations in $T$. Therefore, the tradeoff between communication and divergence will not exist for vanilla FA-HMC and it suggests a small local iteration $T$ to minimize unnecessary computation; and (2)In terms of rate dependency w.r.t. the dimension, under similar conditions on the Hessian matrix, the rate of single-device HMC is as low as $O(d^{1/4})$ \cite{Mangoubi18_leapfrog}, which is strictly better than our rate $O(d^{1/2})$ under the federated learning setting. This intrinsic gap is caused by (i) FA algorithm design and (ii) the use of stochastic gradient.

\subsection{Convergence Behaviour for  FA-HMC Algorithm}\label{sec:theo:conanalysis}

For correlated momentum, for simplicity of analysis, we consider the following setting
\[
p^{(c)}_t=\sqrt{\rho}\xi_t+\sqrt{1-\rho}{\xi_t^{(c)}}/{\sqrt{w_c}},\quad \text{for all } c\in [N], t\geq 1,
\]
where $\xi_t,\xi_t^{(c)}$ are independent standard Gaussian and the $\xi_t$ are the shared across all local nodes and $\xi_t^{(c)}$'s are private to each local node $c$. 

Here the factor $1/\sqrt{w_c}$ on $\xi_t^{(c)}$ is a scaling treatment such that the average momentum is a standard Gaussian. To see this, note that the average momentum $p_t= \sum_{c=1}^Nw_cp^{(c)}_t$ has a smaller variance due to the correlation between $\{p^{(c)}_t\}_c$. 
By direct calculations, we have
\[
\mathbb{E}\|p_t^{(c)}\|^2=(\rho+\frac{1-\rho}{w_c})d,\qquad\mathbb{E}\|p_t\|^2=d.
\]
Note that for FA-HMC, the momentum of each local device is not standard Gaussian. This is to ensure that the center momentum (i.e., $p_t = \sum_{c=1}^N w_c p_t^{(c)}$ aggregated from local momentum) is close to the standard Gaussian. 
 This is a special setting induced by distributed sampling and the goal of privacy preservation.

We define the aggregated global model $\theta_t:= \sum_{c=1}^N\theta^{(c)}_t$ for all $t\geq 1$. Note that $\theta_t$, in practice, is not accessible unless $t\equiv 0 (\mathrm{mod }\ T)$. For $t\geq 0$, we also define $\theta^\pi_{t+1}$ as the parameter resulting from the evolution over $K\eta_t$ time  following dynamic (\ref{eq:hmc}) with initial position  $\theta_{t}^\pi$ and momentum $p_t$. With the above preparations, to intuitively understand  the convergence of the distribution of $\theta_t$, 
we take the vanilla FA-HMC as an example. We can decompose $ \theta_{t+1}^{(c)}-\theta^\pi_{t+1}$ as follow:
\begin{align*}
    \theta_{t+1}^{(c)}-\theta^\pi_{t+1}=(\text{I}_1)-\eta^2\sum_{k=1}^{K-1}(K-k)(\text{I}_2)_k-(\text{I}_3),
\end{align*}
where
\begin{align*}
    &(\text{I}_1)=\theta^{(c)}_{t,0}-\frac{(K\eta)^2}{2}\nabla f^{(c)}(\theta^{(c)}_{t,0})-\frac{(K^3-K)\eta^3}{6}\nabla^2 f^{(c)}(\theta^{(c)}_{t,0}) \\
    &\quad\cdot p_t-\Bigl(\theta^\pi_t-\frac{(K\eta)^2}{2}\nabla f(\theta^\pi_t)-\frac{(K\eta)^3}{6}\nabla^2 f(\theta^\pi_t)p_t\Bigr);\\
    &(\text{I}_2)_k=\nabla f^{(c)}(\theta^{(c)}_{t,k})-\nabla f^{(c)}(\theta^{(c)}_{t,0})-\nabla^2 f^{(c)}(\theta^{(c)}_{t,0})\eta p_tk\\
    &(\text{I}_3)=\int_0^{K\eta}\int_0^s\nabla f(\theta^\pi_t(u))-\nabla f(\theta^\pi_t)-\nabla^2 f(\theta^\pi_t)p_tududs.
\end{align*}
Here $(\text{I}_1)$ represents second-order random approximation of $\theta_{t+1}^{(c)}-\theta^\pi_{t+1}$ through $\theta_{t}^{(c)}$ and $\theta^\pi_{t}$, and we expect that
\[
\mathbb{E}\|\sum_{c=1}^Nw_c(\text{I}_1)\|^2\leq \alpha_t  \mathbb{E}\|\theta_{t}-\theta^\pi_{t}\|^2 + \varepsilon_t^2,
\]
where the contraction factor $\alpha_t\in (0,1)$ and one-iteration divergence error $\varepsilon_t>0$.

On the other hand, $\|(\text{I}_2)_k\|$ and $\|(\text{I}_3)\|$ represent second-order approximation error and are expected to be $O((K\eta_t)^2\sqrt{d})$. 

By utilizing Lemma D.1, 
the overall behavior is summarized in the following theorem. 

\begin{restatable}[Convergence]{theorem}{thmdisFLHMC}\label{thm:disFLHMC}
Under Assumptions \ref{assum:convex}-\ref{assum:boundvar}, if we  set $\eta_{t''}\leq \eta_{t'}\leq 1/(K\sqrt{L})$ for any $t' \leq t''$ in Algorithm \ref{alg:disFAHMC},  then  $\{\theta_t\}_t$  satisfies 
\[
    \mathbb{E} \|\theta_{t+1}-\theta^\pi_{t+1}\|^2
     \leq(1-\frac{\mu(K\eta_t)^2}{4})^t\mathbb{E}\|\theta_0-\theta^\pi_0\|^2+\eta_t^2\Delta_t
\]
where there exist constants $C_1,C_2>0$ depending on $L,L/\mu,L^2_H/L^3$  and $c_d=\log^2(d)$, such that
\begin{align*}
 \Delta_t=& C_1T^2 K^2\sum_{c=1}^N\Bigl(\underbrace{\frac{w_cB^{(c)}_{\nabla}}{L}}_{\mathrm{B
 ias}}(K\eta_t)^2+\underbrace{\frac{1-\rho}{L}d}_{\mathrm{Correlation}}\Bigr)+C_2K\underbrace{\sum_{c=1}^Nw_c^2\sigma_g^2d}_{\mathrm{Stoc.\; Grad.}}
\end{align*}
with $B^{(c)}_{\nabla}:=\sup_t\mathbb{E}\|\nabla f^{(c)}(\theta^{(c)}_{t,0})\|^2$.
\end{restatable}

The proof is postponed to Section G 
in the supplementary material. 
The divergence error is made up of three main components: error resulting from bias across local nodes (which includes heterogeneity and sampling cost), momentum noise, and gradient noise. In the absence of stochastic gradients and when momentum is identical across nodes, the only errors present are lower-order biases.
 Similar intermediate contraction results have been derived in the literature on gradient-based sampling algorithms \citep[e.g.,][]{deng2021convergence, plassier2022federated}.

By the definition of Wasserstein metric, Theorem \ref{thm:disFLHMC} immediately establishes a convergence result of the marginal distribution of $\theta_t$, denoted by $\pi_t$, towards $\pi$ under Wasserstein-2 distance.
The convergence result involves a term $\sum_{c=1}^Nw_c\sup_t\mathbb{E}\|\nabla f^{(c)}(\theta^{(c)}_{t,0})\|^2/L$. In Lemma D.7 
in the appendix, we shows that uniformly, $\mathbb{E}\|\nabla f^{(c)}(\theta^{(c)}_{t,0})\|^2=\widetilde{O}(\sum_{c=1}^Nw_c\|\nabla f^{(c)}(\theta^*)\|^2+L\mathbb{E}\|\theta_0-\theta_0^\pi\|^2+d)$  omitting its dependency on constants $L, L/\mu$ and $L_H^2/L^3$, and in consequence,  solving the two inequalities
\[
(1-\mu(K\eta_t)^2/4)^t\mathbb{E}\|\theta_0-\theta^\pi_0\|^2\leq \epsilon^2/2,\qquad\eta_t^2\Delta_t\leq \epsilon^2/2,
\]
we obtain Theorem \ref{thm:convergrate:sg}.

On the other hand, in literature, people design settings for converging learning rate such that the extra logarithmic factor in the convergence result can be removed. We also obtain a similar result on a learning rate design as stated in the following proposition.
\begin{restatable}[Dynamic stepsize]{proposition}{propdynamic}\label{prop:dynamic}Under Assumptions \ref{assum:convex}-\ref{assum:boundvar}, there is a setting of  $\{\eta_t\}_t$ for Algorithm \ref{alg:disFAHMC} such that $\mathbb{E}\|\theta_t-\theta^\pi_t\|^2\leq \epsilon^2$ at some $t\leq C\log^2(d)d\big(T^2(\gamma+(1-\rho)N\big)+\sum_{c=1}^Nw_c^2\sigma^2_g/K\big)/\epsilon^2$,  with $\gamma=\min\{1/\sqrt{L},\epsilon/\sqrt{d}T,\epsilon^2/(dT^2(1-\rho)N),\epsilon^2K/(d\sum_{c=1}^Nw_c^2\sigma^2_g)\}$.
\end{restatable}

By this proposition, we see that the $\log(d/\epsilon^2)$ factors are removed in the convergent iteration compared to Theorem~\ref{thm:convergrate:sg}. One setting of $\eta_t$ that satisfies the claims in Proposition 4.6 is specified in the proof (i.e., Section H in the supplement file).

\section{Experiments} \label{sec:simu}
In this section, we first compare the empirical performance of FA-HMC and FA-LD on simulated data. Then we examine the relationship between dimension and communication round in our theoretical suggested setting of the learning rate. Last we present the performance of FA-HMC on the real datasets. We apply FA-HMC with constant stepsize $\eta$ and the same momentum initialization across devices. 
We conduct 
the synchronization of the model parameters 
every $T$ local leapfrog step in the implementation of FA-HMC. Due to the significant computational costs involved in evaluating performance at each cohort level, some results in this section are obtained from a single run and others are obtained by averaging multiple runs. 
We defer part of the results with error bars to Section L 
in the supplementary materials.

\subsection{Simulation: FA-HMC vs FA-LD} \label{sec:sim_data}


We first sample from the posterior of a Bayesian logistic regression on a simulated dataset of dimension $d=1000$~\citep{Mangoubi18_leapfrog}.  Specifically, we split the dataset of size $1000$ equally into $20$ local nodes; we run the experiments using both exact gradients (i.e., vanilla version) and stochastic gradients, where the later ones are simulated by adding an independent zero-mean Gaussian noise of variance $\sigma^2=100$ to each coordinate of the true gradients. Note that simulating the randomness of the stochastic gradient by a normal variable is consistent with the experiment setting in \cite{Mangoubi18_leapfrog}. We argue that Gaussian noise is a reasonable approximation when invoking the central limit theorem with a large enough batch size.



As the benchmark, we run Metropolis-adjusted HMC (MHMC) for a sufficient number of iterations. 
To evaluate the performance of FA-HMC, we use the computable metric $\frac{1}{d}\sum_{i=1}^d\mathcal{W}_1(\mu_i,\nu_i)$ as a measure  of marginal error (ME) of two sets of samples, as proposed by \cite{Mangoubi18_leapfrog,faes2011variational}. This metric compares the empirical distributions of the $i$-th coordinate of the two sets of samples, represented by $\mu_i$ and $\nu_i$, respectively. 

We first compare FA-HMC with FA-LD (i.e., FA-HMC with $K=1$). 
Noticing that the communication limit is a major bottleneck for federated learning, we suppose the
local computation cost is negligible compared with the communication cost. Therefore, the comparison between FA-HMC and FA-LD is based on the same number of communications, or equivalently, the same number of steps $t$.
Fixing local step $T=10$, we try different stepsizes $\eta$ and leapfrog steps $K$. For FA-HMC, we set $K=\lfloor \pi/(3\eta)\rfloor$ following \cite{Mangoubi18_leapfrog} when $\eta\le 0.01$ and tune $K$ (such that the performance is optimal w.r.t the choice of $K$) when $\eta\ge 0.02$. 
Each run consists of $2\times10^7$ steps and we collect the same number of samples from the last $10^7$ steps. 
We plot the curves of the calculated MEs against $\eta$ in Figure \ref{fig:sim-stepsize} (exact gradients, G) and \ref{fig:sim-stepsize-sto} (stochastic gradients, SG). 
We observe that in this task, where FA-LD is already a competitive baseline, 
\emph{FA-HMC still significantly outperforms FA-LD with around 5\% improvement on the performance}. 
Moreover, we realize that a wide range of stepsizes for FA-HMC yields pretty decent performance. As such, FA-HMC appears to be more robust w.r.t. its hyperparameters around the optimal choices, suggesting that FA-HMC is easier to tune than FA-LD, and a small stepsize usually leads to a good performance.

Next, we study the impact of local steps $T$ on communication efficiency in FA-HMC with SG. Fixing leapfrog step $K=100$ and stepsize $\eta=0.01$, we run FA-HMC with $T$ ranging from 1 to 100.
For each run, we collect one sample after a fixed number of communication rounds and calculate the MEs in an online manner. Then, we report the required rounds $R_{\epsilon}$ 
to achieve $\textup{ME}=\epsilon$ under different settings and present the results 
in Figure \ref{fig:sim-local-sto}. 
As we can see, the optimal local step $T$ is 70; setting $T$ too large or too small leads to more communication costs. 
We also notice that under the optimal local step,
a smaller $\epsilon$ leads to more improvement on the communication cost $R_\epsilon$ compared with the result of $T=100$. Moreover, compared with the communication efficiency of $T=1$, the optimal \emph{communication efficiency improves by more than 65 times when $\epsilon$ is around $0.101$.}

Furthermore, we reduce the dimension $d$ to 10 in the simulated data and run FA-HMC as well as FA-LD with different stepsizes $\eta$ on this new dataset fixing local step $T=10$. Apart from the dimension $d$, the other settings are the same as those in the experiments for Figure \ref{fig:sim-stepsize-sto}. To list a few, stochastic gradients are adopted, and we choose the leapfrog step $K=\lfloor \pi/(3\eta)\rfloor$ when $\eta\le 0.01$ and tune $K>1$ when $\eta\ge 0.02$ for FA-HMC. The curves of the MEs against $\eta$ are plotted in Figure \ref{fig:sim-dim10}. We observe that the general pattern in  Figure \ref{fig:sim-dim10} is similar to Figure \ref{fig:sim-stepsize-sto}. The optimal performance of FA-HMC is better than that of FA-LD, and the performance gap is larger for smaller step sizes. Comparing Figure \ref{fig:sim-dim10} with Figure \ref{fig:sim-stepsize-sto}, we comment that FA-HMC is more advantageous under high-dimensional settings. This observation is consistent with our theoretical results that FA-HMC has a better convergence rate in terms of the dimension.


\begin{figure*}[htbp]
    \centering
    \subfigure[Study of $\eta$ and $K$ (G)]{
    \begin{minipage}[t]{0.24\linewidth}
    \centering
    \label{fig:sim-stepsize}
    \includegraphics[width=\linewidth]{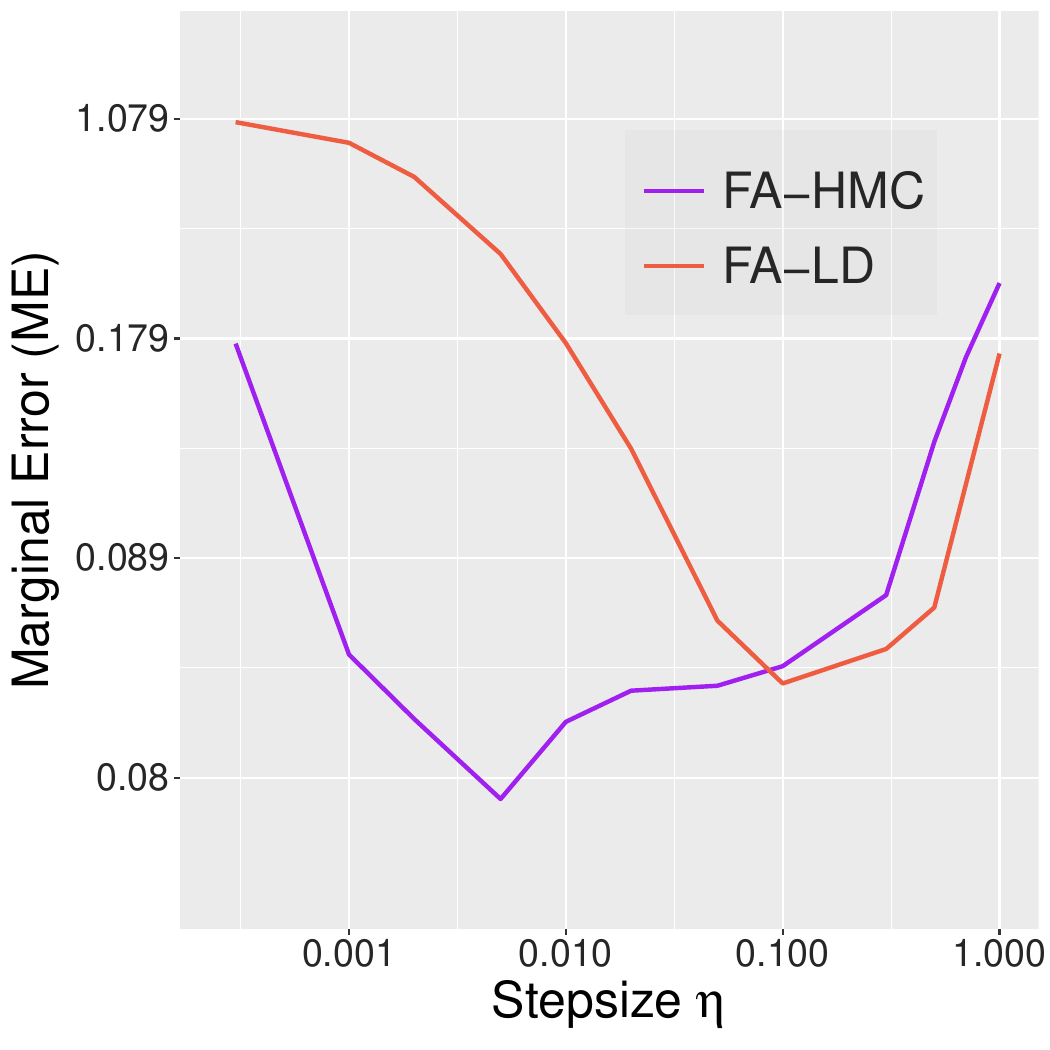}
    \end{minipage}%
    }%
    \subfigure[Study of $\eta$ and $K$ (SG)]{
    \begin{minipage}[t]{0.24\linewidth}
    \centering
    \label{fig:sim-stepsize-sto}
    \includegraphics[width=\linewidth]{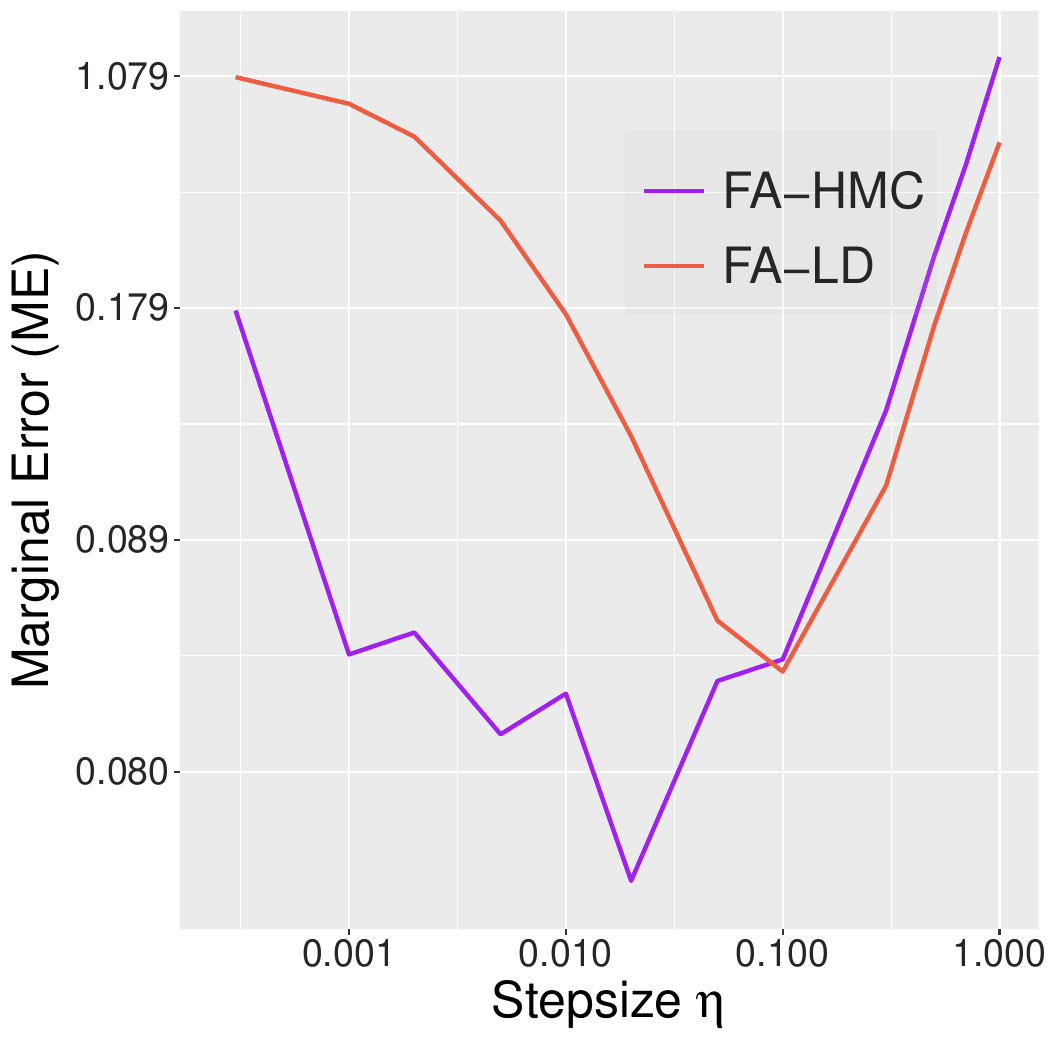}
    \end{minipage}%
    }%
    \subfigure[Study of $T$ (SG)]{
    \begin{minipage}[t]{0.24\linewidth}
    \centering
    \label{fig:sim-local-sto}
    \includegraphics[width=\linewidth]{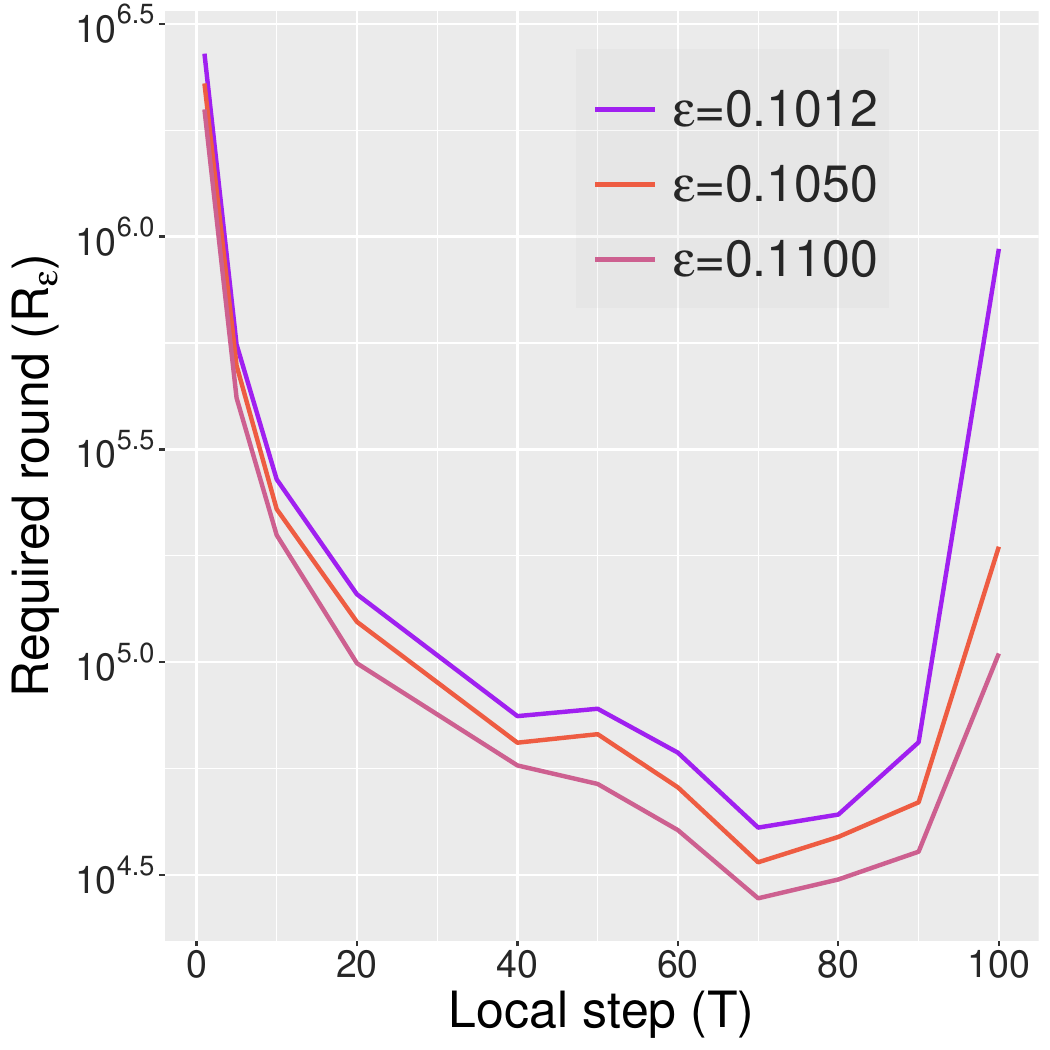}
    \end{minipage}%
    }%
    \subfigure[$d=10$ (SG)]{
    \begin{minipage}[t]{0.24\linewidth}
    \centering
    \label{fig:sim-dim10}
    \includegraphics[width=\linewidth]{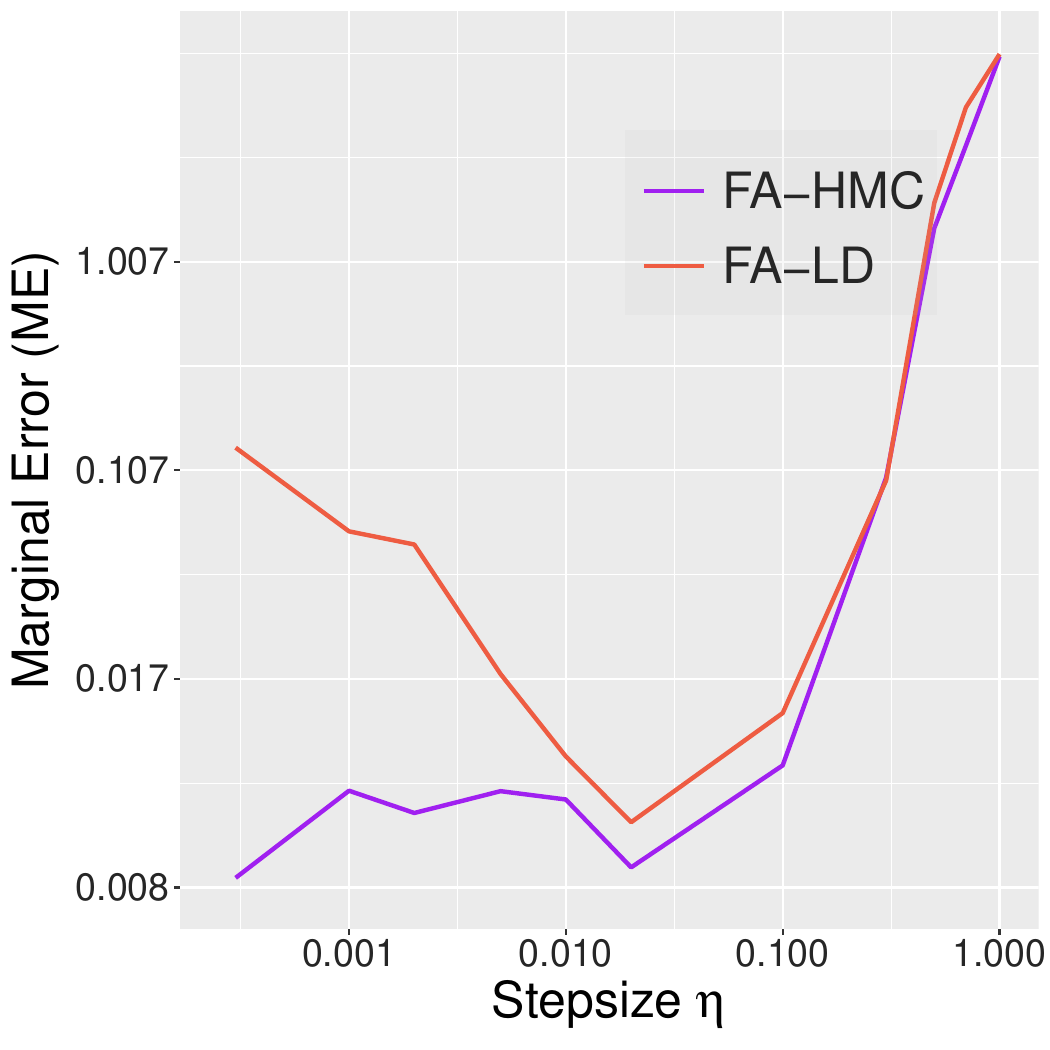}
    \end{minipage}%
    }%
  \vskip -0.1in
  \caption{Experimental results of FA-HMC and FA-LD on the simulated dataset using exact gradients (G) and stochastic gradients (SG). Dimension $d=1000$ in Figure (a)-(c) and $d=10$ in Figure (d).}
  \label{figure:sim}
\end{figure*}

\subsection{Simulation: Dimension vs Communication for FA-HMC}
In this experiment, under the suggested setting of learning rate in Proposition~\ref{prop:convergrate}, we examine the relationship between communication rounds $t_{\epsilon}/T$ required to achieve a ${\cal W}_2(\theta_{t_{\epsilon}},\theta^{\pi})^2<0.1$ and dimension $d$. 

To obtain an accurate computation of the ${\cal W}_2(\theta_{t_{\epsilon}},\theta^{\pi})$, we consider a distributed heterogeneous Gaussian model where the ${\cal W}_2(\theta_{t_{\epsilon}},\theta^{\pi})$ can be explicitly calculated in terms of the population mean and variance of the parameter. Specifically, we assume that the posterior distribution of half of the local nodes' parameters is $N(20\mathbf{1}_d,\mathbb{I}_d)$, and for the other half, it is $N(\mathbf{1}_d,2\mathbb{I}_d)$. One can check that the overall posterior distribution of parameters is $N(16.2\mathbf{1}_d,1.6\mathbb{I}_d)$. We use  leapfrog steps $K=5$, local steps $T = 10$, and a learning rate $\eta=0.02/d^{1/4}$. For different dimensions $d = 2, 50, 100, 150, ..., 950, 1000$, we repeat the experiment   $200 \cdot d(d-1)/2$ times and sample the parameter $\theta_t$ at the last iteration $t$ for each time. The sampled parameters allow us to estimate the population mean and variance on the calculation of ${\cal W}_2(\theta_{t_{\epsilon}},\theta^{\pi})$. 

The simulation results in Figure~\ref{fig:dim_comm} suggest that the square of communication round $(t_\epsilon/T)^2$ is approximately proportional to dimension $d$. This aligns well with our theoretical discovery in Proposition~\ref{prop:convergrate}, where under the suggested learning rate setting $t_\epsilon/T=O(\sqrt{d}\log(d/\epsilon^2)/\epsilon)$.

\begin{figure}[htbp]
    \centering
    \includegraphics[width=0.5\linewidth]{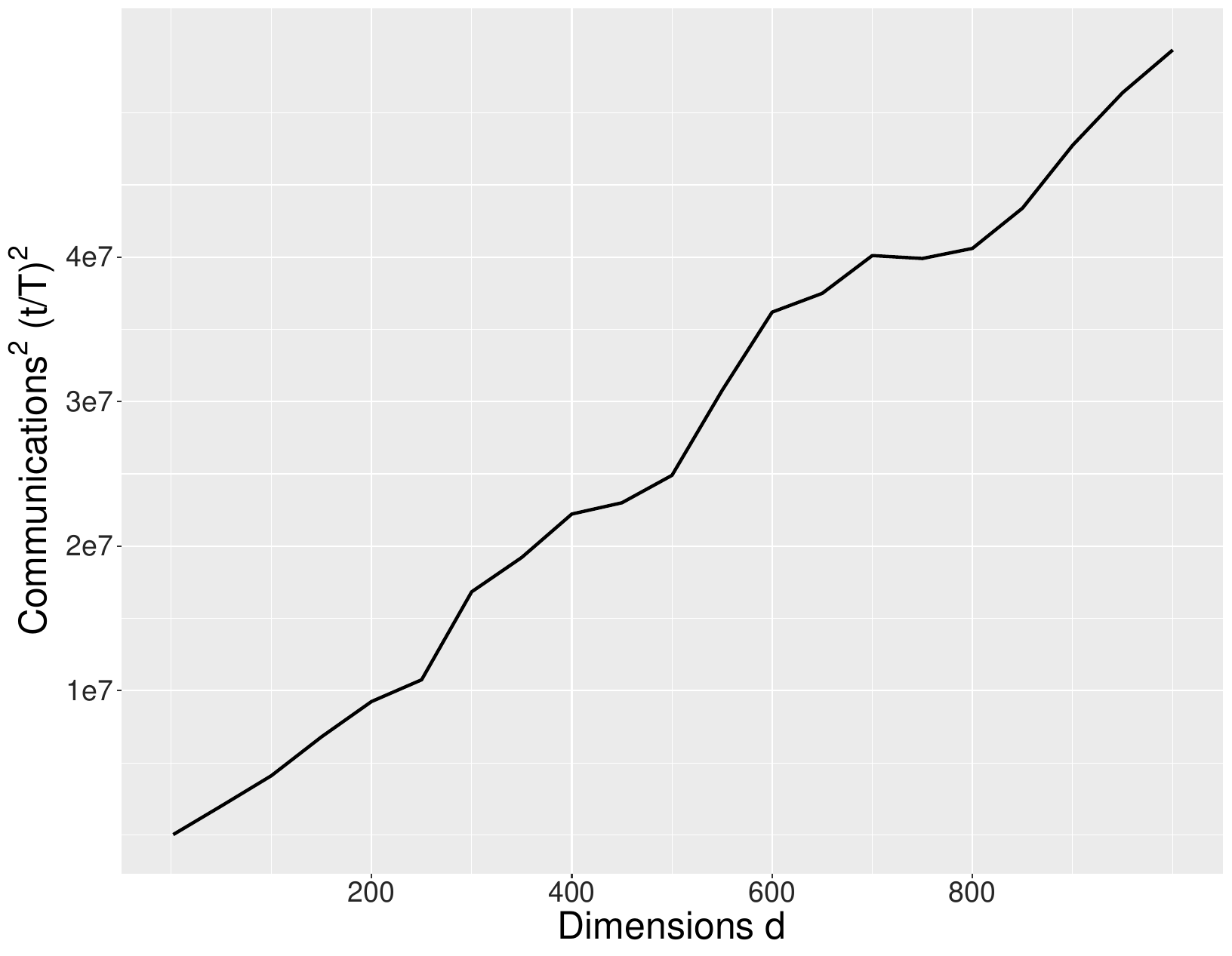}
    \caption{Experimental results of FA-HMC to achieve ${\cal W}_2<0.1$ at different dimensions $d$.}
    \label{fig:dim_comm}
\end{figure}

\subsection{Application: Logistic Regression Model for FMNIST} \label{sec:real_data}

In this section, we apply FA-HMC to train a logistic regression on the Fashion-MNIST dataset. The data points are randomly split into 10 subsets of equal size for $N=10$ clients. We run FA-HMC under different settings of local step $T$ and leapfrog step $K$ with stochastic gradients that are calculated using a batch size of $1000$ in each local device. 
In each run, one parameter sample is collected after a fixed number of communication rounds, and the predicted probabilities made by all the previously collected parameter samples are averaged to calculate four test statistics: prediction accuracy, Brier Score (BS) \citep{brier1950verification}, Expected Calibration Error (ECE) \citep{guo2017calibration}, and Negative Log Likelihood (NLL) on the test dataset. We tune the step size $\eta$ in each setting for the best test statistic. 
We conduct $5$ independent runs in each setting and report the average results of those chains. The standard deviations of the results across multiple runs are displayed in Section L 
in the supplementary materials.

Specifically, to study the impact of leapfrog step $K$ on the performance of FA-HMC, we fix local step $T=50$, run FA-HMC with $K=1$, 10, 50, and 100, and plot the curves of the calculated test statistics (accuracy, BS, ECE, and NLL) against communication rounds in Figure \ref{figure:Fashion_HMC}. 
As we can see, under the same budgets of communication and computation, FA-LD ($K=1$) performs the worst in terms of BS, ECE, and NLL and the second worst in terms of accuracy, which shows the superiority of FA-HMC with $K>1$ over FA-LD. 
Moreover, FA-HMC with $K=50$ performs the best in terms of accuracy, BS, and NLL and achieves a small ECE. In particular, the improvement on ECE and NLL over $K=1$ can be as large as 26\% and 2\% respectively, indicating that the optimal choice of leapfrog step is around 50 in this setting.

\begin{figure*}[htbp]
    \centering
    \subfigure[Accuracy]{
    \begin{minipage}[t]{0.24\linewidth}
    \centering
    \label{fig:HMC-accu}
    \includegraphics[width=\linewidth]{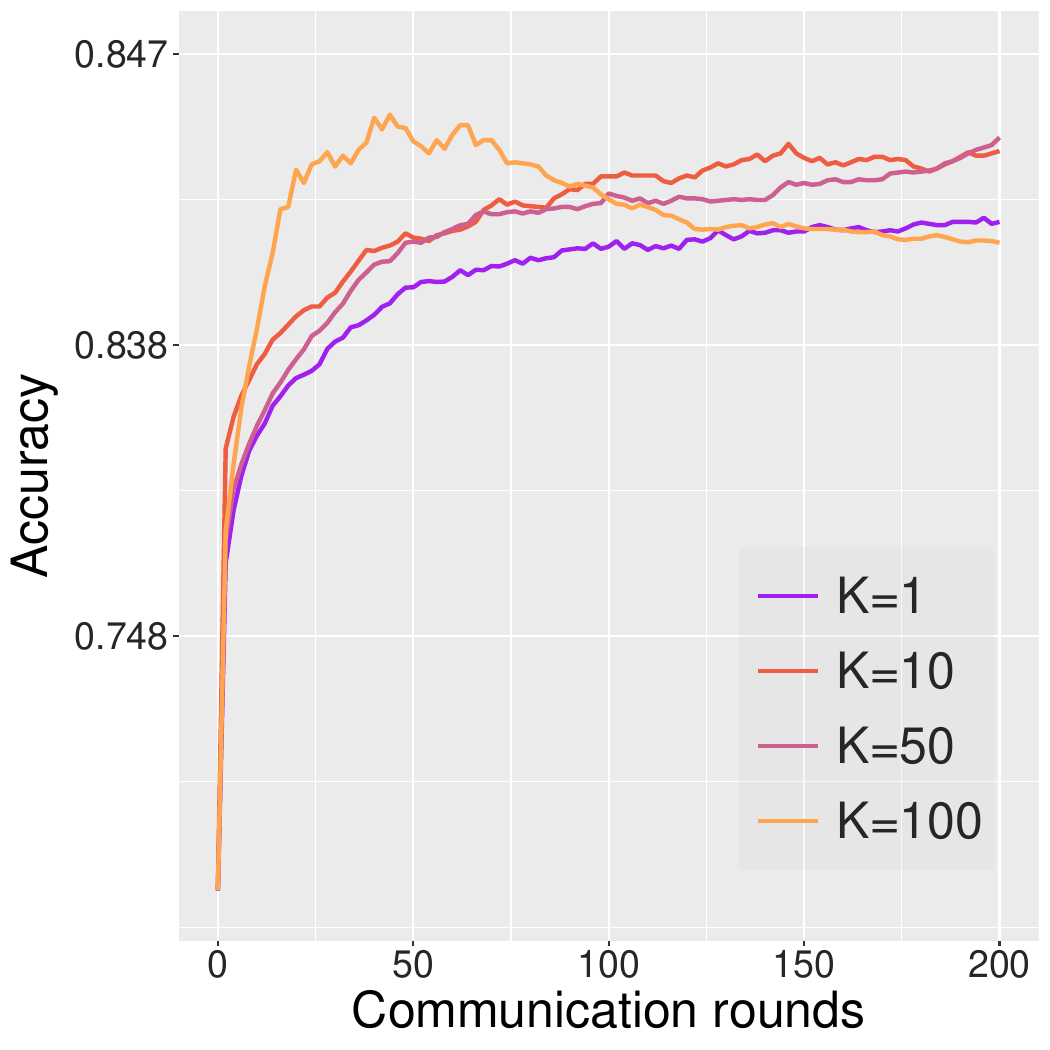}
    \end{minipage}%
    }%
    \subfigure[BS]{
    \begin{minipage}[t]{0.24\linewidth}
    \centering
    \label{fig:HMC-brier}
    \includegraphics[width=\linewidth]{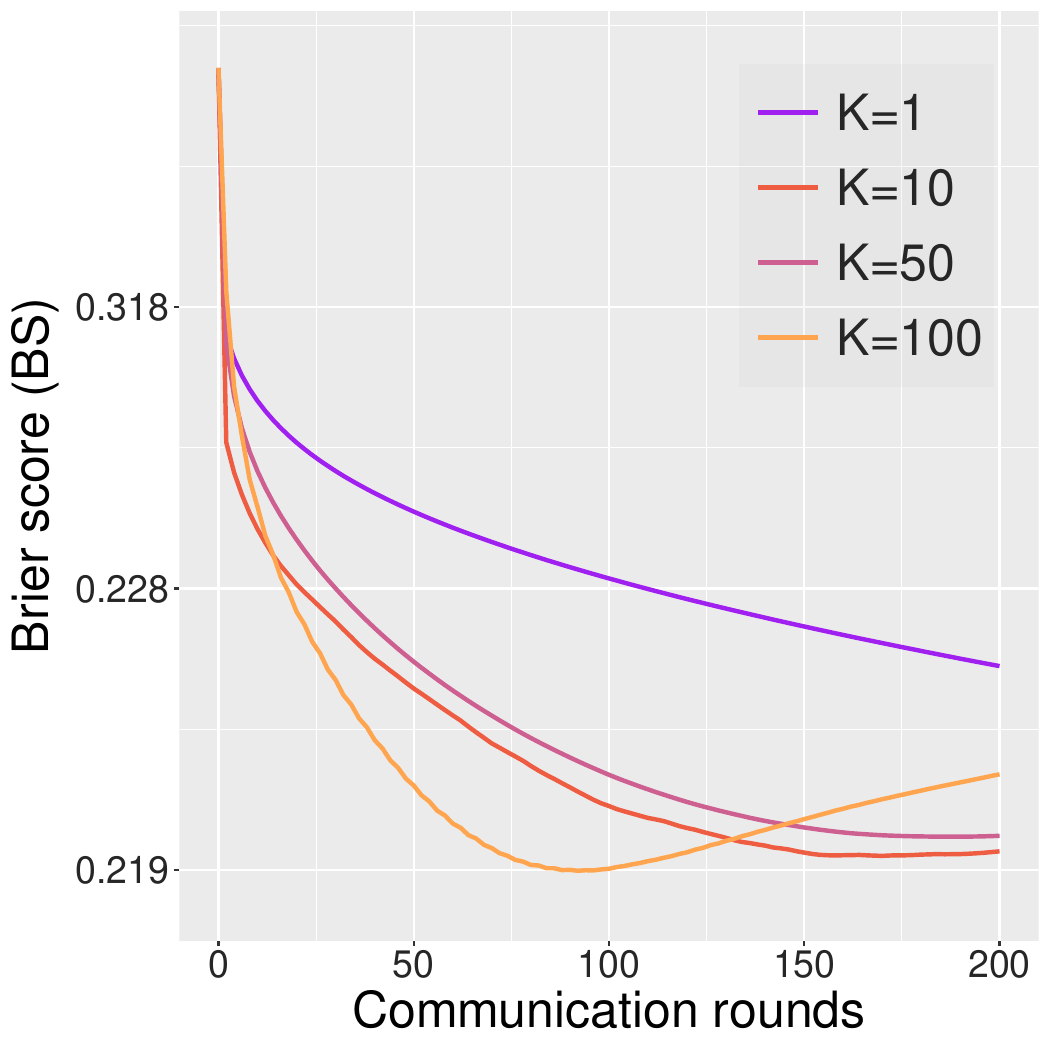}
    \end{minipage}%
    }%
    \subfigure[ECE]{
    \begin{minipage}[t]{0.24\linewidth}
    \centering
    \label{fig:HMC-ECE}
    \includegraphics[width=\linewidth]{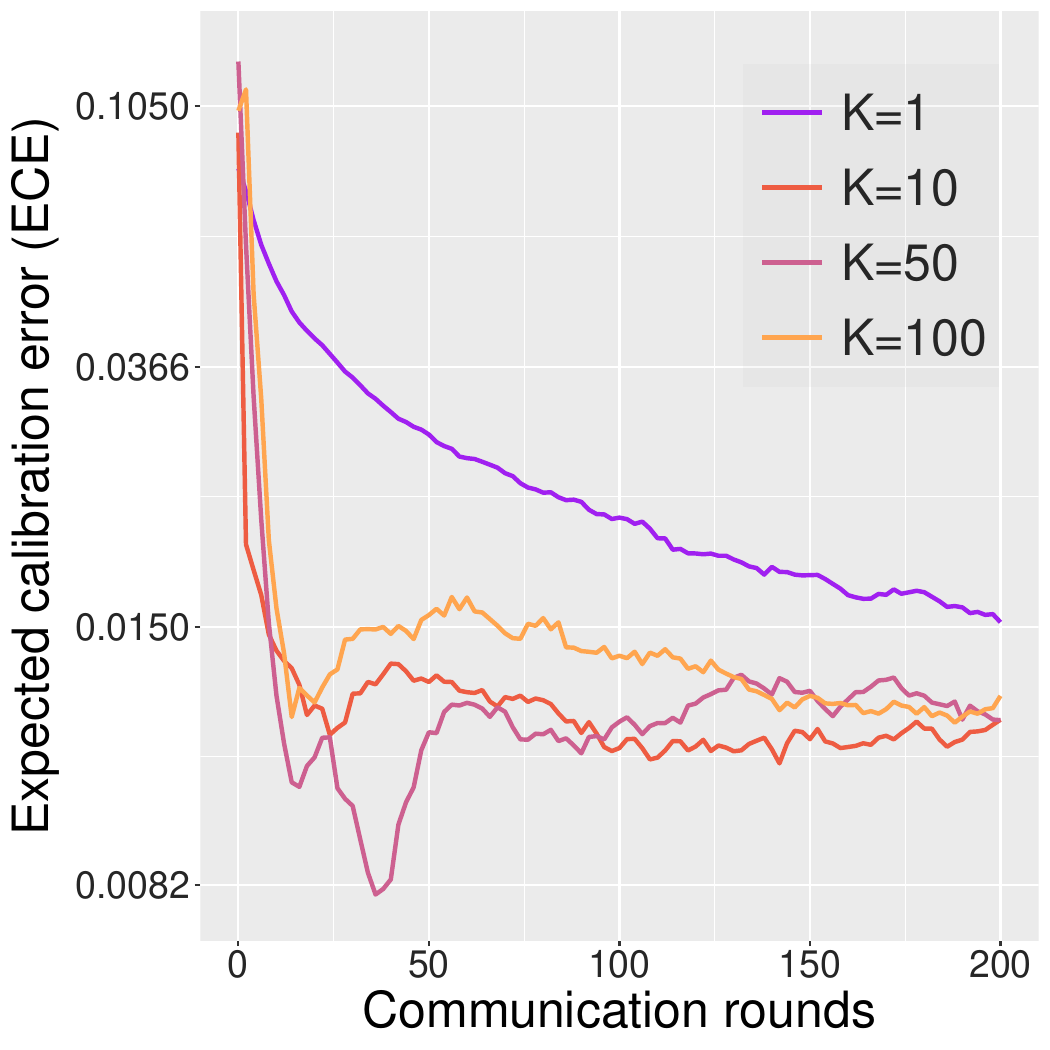}
    \end{minipage}%
    }%
    \subfigure[NLL]{
    \begin{minipage}[t]{0.24\linewidth}
    \centering
    \label{fig:HMC-NLL}
    \includegraphics[width=\linewidth]{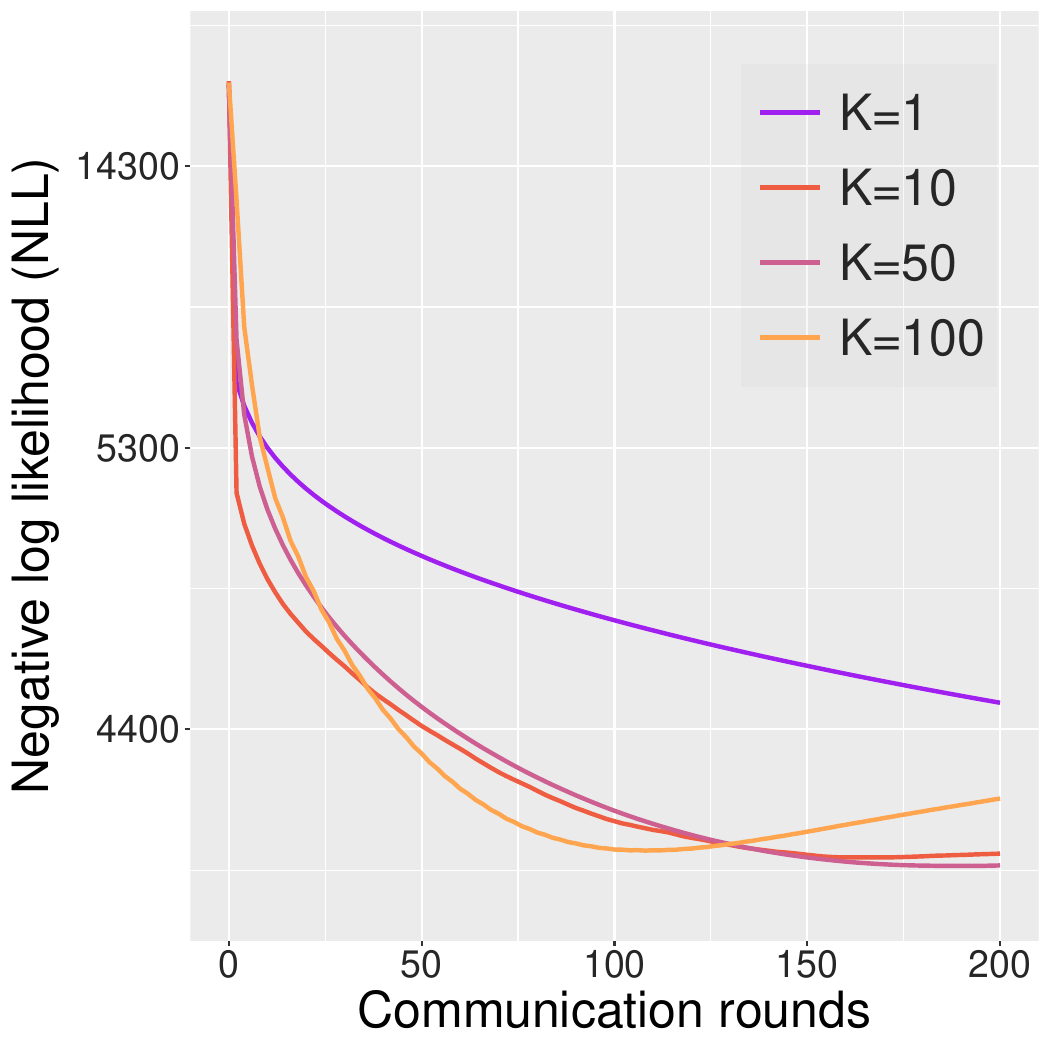}
    \end{minipage}%
    }%
  \vskip -0.1in
  \caption{The impact of leapfrog steps $K$ on FA-HMC applied on the Fashion-MNIST dataset.}
  \label{figure:Fashion_HMC}
\end{figure*}

To study the impact of local step $T$ on the performance of FA-HMC, we fix leapfrog step $K=10$, run FA-HMC with $T=1$, 10, 20, 50, and 100, and plot the curves of the calculated test statistics (accuracy, BS, ECE, and NLL) against communication rounds in Figure \ref{figure:Fashion_local}. 
According to the figure, FA-HMC with $T=1$ performs the worst in terms of all four statistics, which shows the necessity of multiple local updates in this setting. 
Besides, the optimal local step $T$ differs with testing evaluation metrics; e.g., the optimal $T$ is 50 in terms of BS, while the optimal $T$ is 20 in terms of NLL.

\begin{figure*}[htbp]
    \centering
    \subfigure[Accuracy]{
    \begin{minipage}[t]{0.24\linewidth}
    \centering
    \label{fig:local-accu}
    \includegraphics[width=\linewidth]{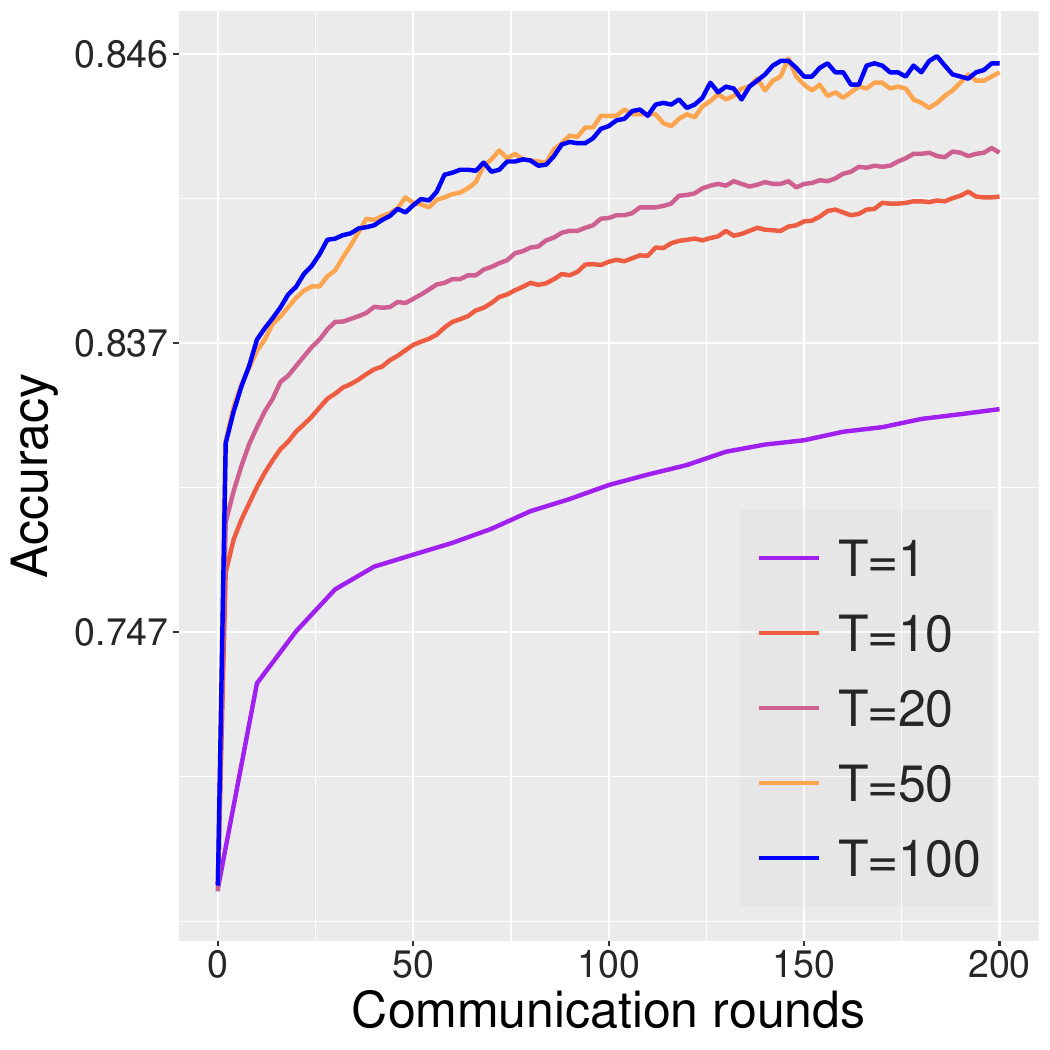}
    \end{minipage}%
    }%
    \subfigure[BS]{
    \begin{minipage}[t]{0.24\linewidth}
    \centering
    \label{fig:local-brier}
    \includegraphics[width=\linewidth]{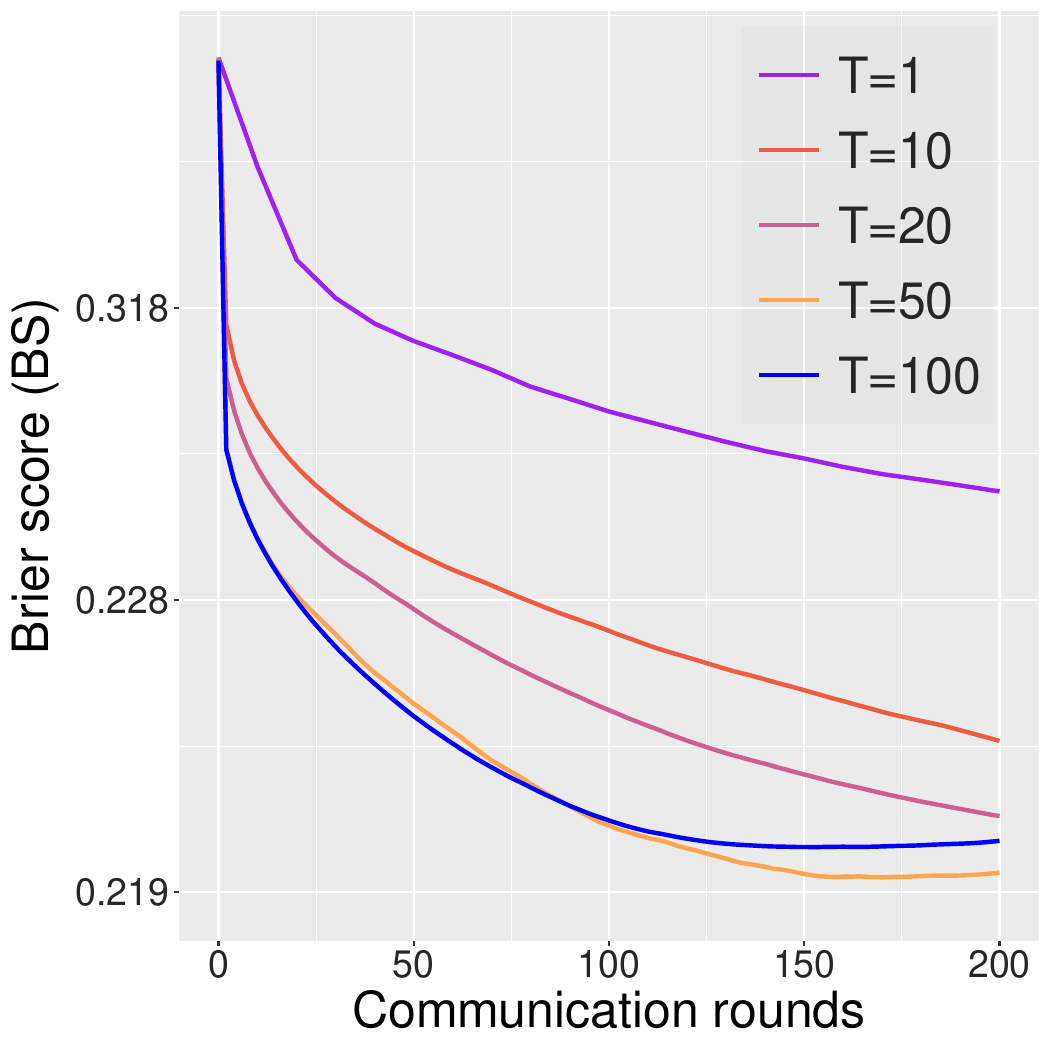}
    \end{minipage}%
    }%
    \subfigure[ECE]{
    \begin{minipage}[t]{0.24\linewidth}
    \centering
    \label{fig:local-ECE}
    \includegraphics[width=\linewidth]{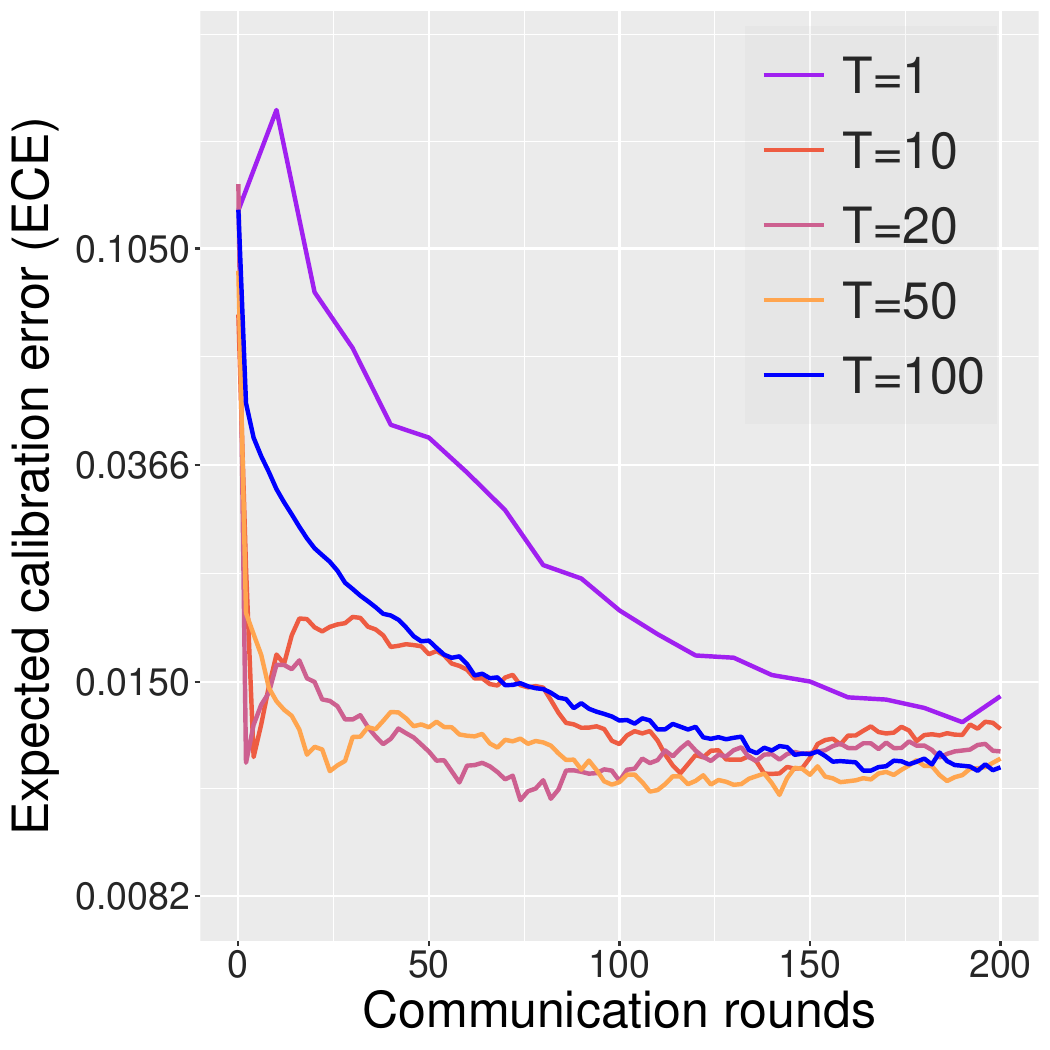}
    \end{minipage}%
    }%
    \subfigure[NLL]{
    \begin{minipage}[t]{0.24\linewidth}
    \centering
    \label{fig:local-NLL}
    \includegraphics[width=\linewidth]{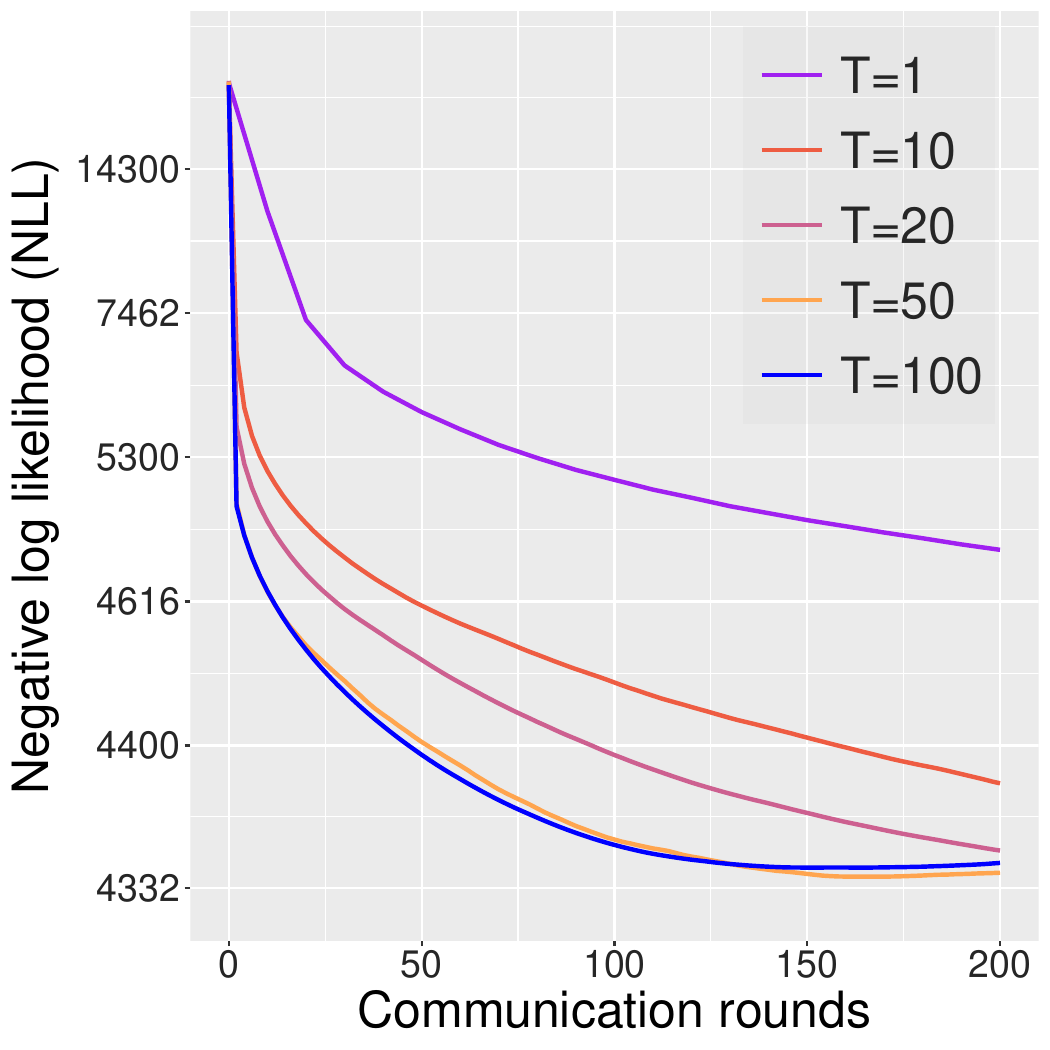}
    \end{minipage}%
    }%
  \vskip -0.1in
  \caption{The impact of local steps $T$ on FA-HMC applied on the Fashion-MNIST dataset.}
  \label{figure:Fashion_local}
\end{figure*}

\subsection{Application: Neural Network Model for FMNIST} \label{sec:nn_exp}
To further assess the performance of FA-HMC on non-convex problems, we apply FA-HMC and FA-LD to train a fully connected neural network with two hidden layers\footnote{The widths of two layers are 512 times input dimension and 512 times the number of classification labels respectively.} and the ReLU activation function on the Fashion-MNIST dataset. 
Other settings of the experiments are the same as the logistic regression experiments in Section \ref{sec:real_data} except that only one chain is simulated in each case. We also calculate prediction accuracy, Brier Score (BS), and Expected Calibration Error (ECE) on the test dataset. The step size $\eta$ is tuned in each setting for the best test statistic. Fixing local step $T=50$ and choosing leapfrog step $K=1$, 10, 50, and 100, the curves of the calculated test statistics against communication rounds are plotted in Figure \ref{fig:fashion-accu}, \ref{fig:fashion-brier}, and \ref{fig:fashion-ece}. As is shown in the figures, the optimal leapfrog step differs among different test statistics. For accuracy and ECE, the optimal $K=10$ (i.e., FA-HMC notably outperforms FA-LD), while the optimal $K=1$ for BS (i.e., FA-LD sightly outperforms FA-HMC). 
Fixing $K=10$ and choosing $T=1$, 10, and 50, the curves of the calculated test statistics against communication rounds are plotted in  Figure \ref{fig:fashion-accu-local}, \ref{fig:fashion-brier-local}, and \ref{fig:fashion-ece-local}. We can see that the best local step is $T=50$ and the worst local step is $T=1$, indicating that the communication cost can be greatly reduced.
\begin{figure*}[htbp]
    \centering
    \subfigure[Accuracy]{
    \begin{minipage}[t]{0.16\linewidth}
    \centering
    \label{fig:fashion-accu}
    \includegraphics[width=\linewidth]{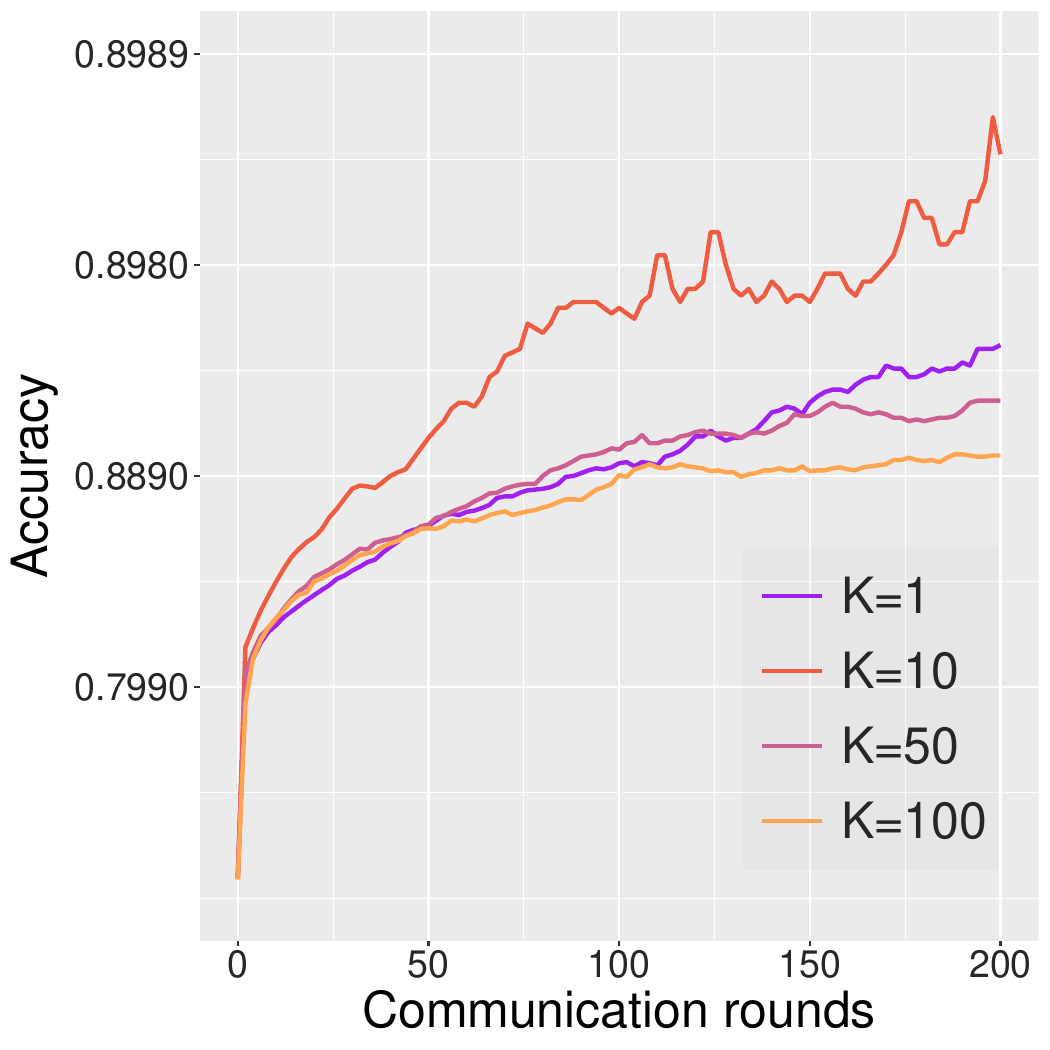}
    \end{minipage}%
    }%
    \subfigure[BS]{
    \begin{minipage}[t]{0.16\linewidth}
    \centering
    \label{fig:fashion-brier}
    \includegraphics[width=\linewidth]{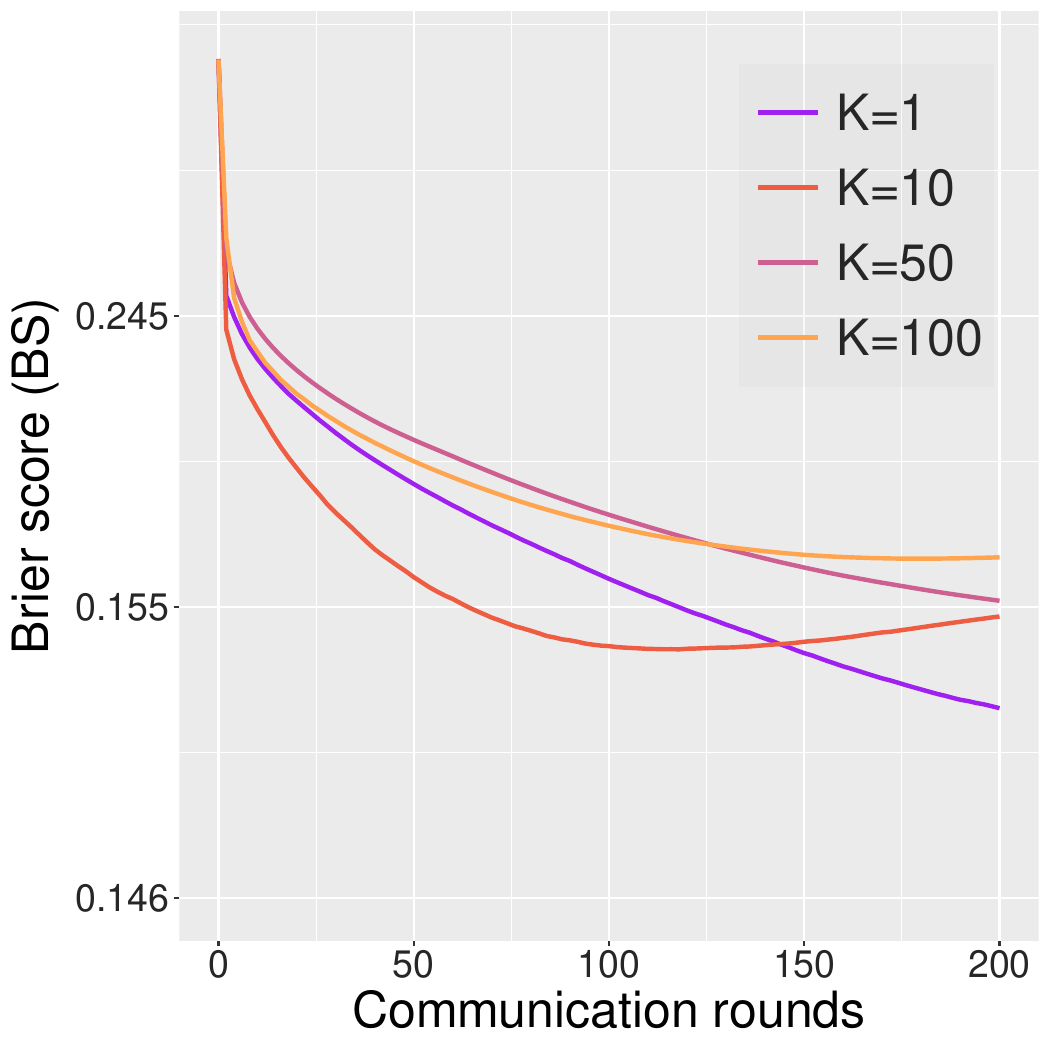}
    \end{minipage}%
    }%
    \subfigure[ECE]{
    \begin{minipage}[t]{0.16\linewidth}
    \centering
    \label{fig:fashion-ece}
    \includegraphics[width=\linewidth]{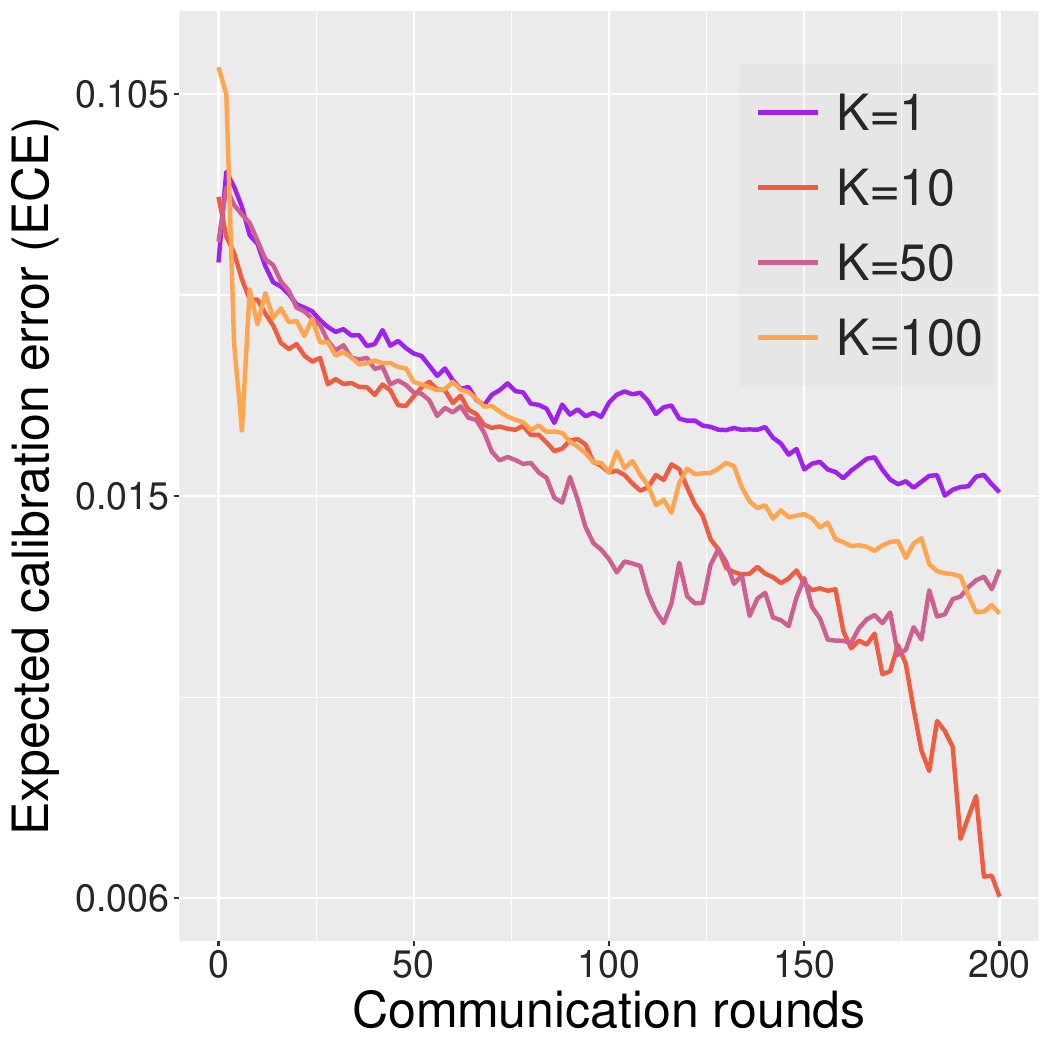}
    \end{minipage}%
    }%
    \subfigure[Accuracy]{
    \begin{minipage}[t]{0.16\linewidth}
    \centering
    \label{fig:fashion-accu-local}
    \includegraphics[width=\linewidth]{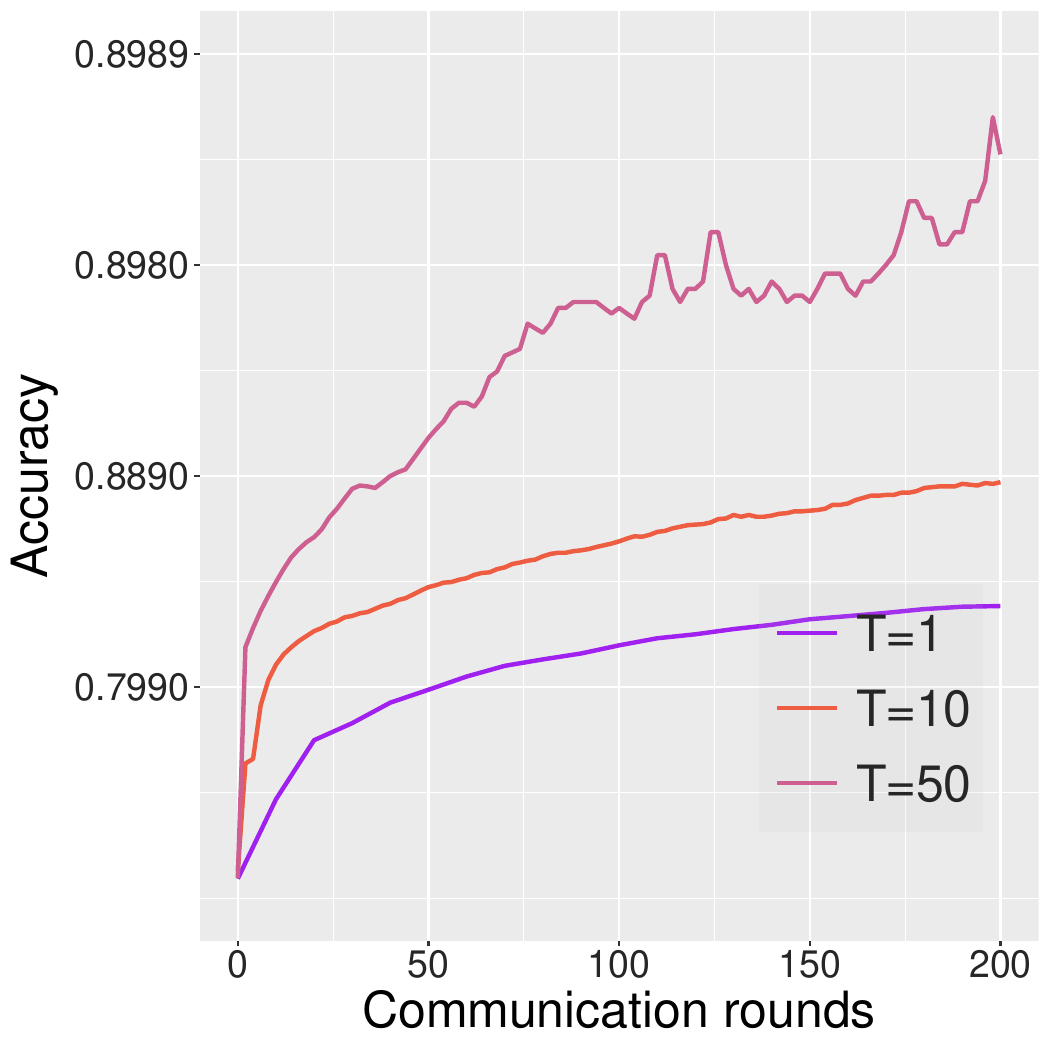}
    \end{minipage}%
    }%
    \subfigure[BS]{
    \begin{minipage}[t]{0.16\linewidth}
    \centering
    \label{fig:fashion-brier-local}
    \includegraphics[width=\linewidth]{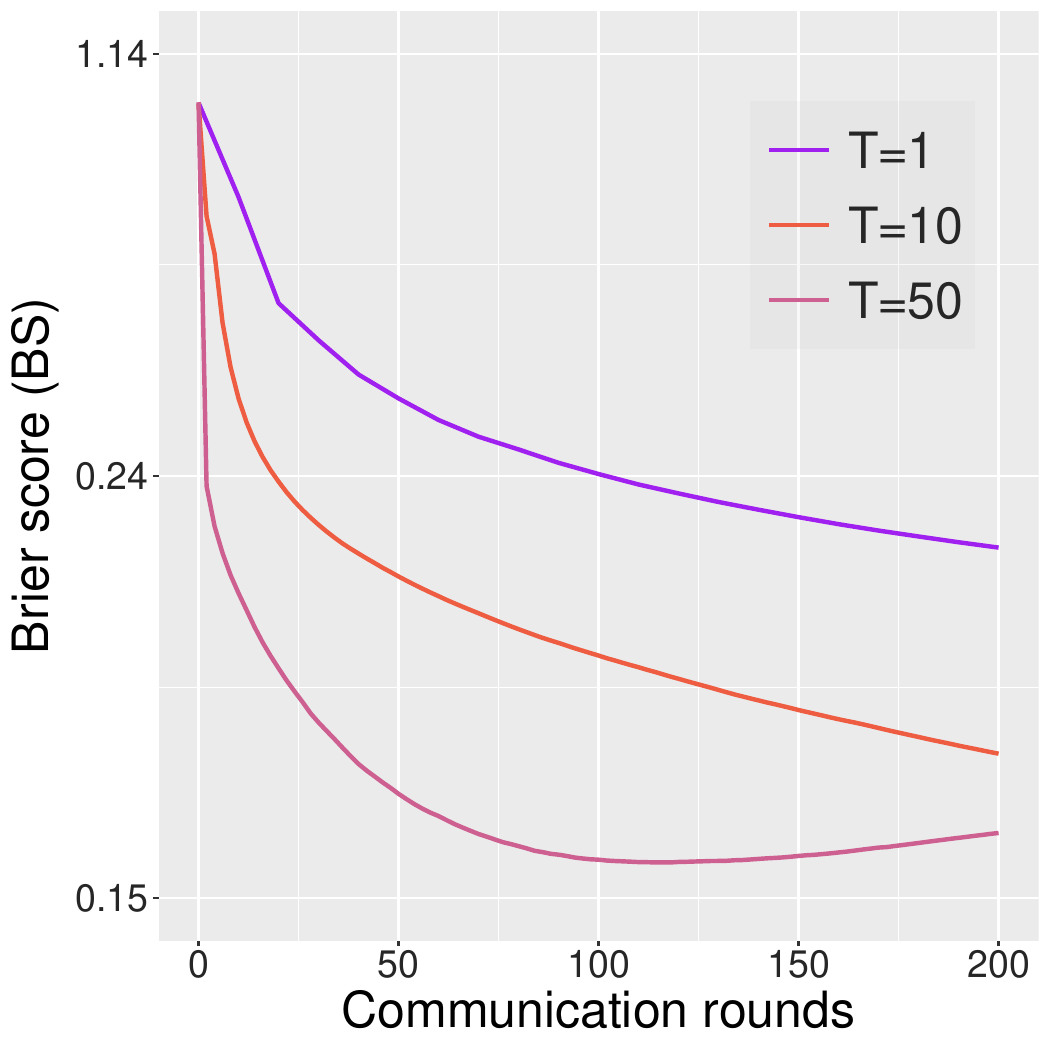}
    \end{minipage}%
    }%
    \subfigure[ECE]{
    \begin{minipage}[t]{0.16\linewidth}
    \centering
    \label{fig:fashion-ece-local}
    \includegraphics[width=\linewidth]{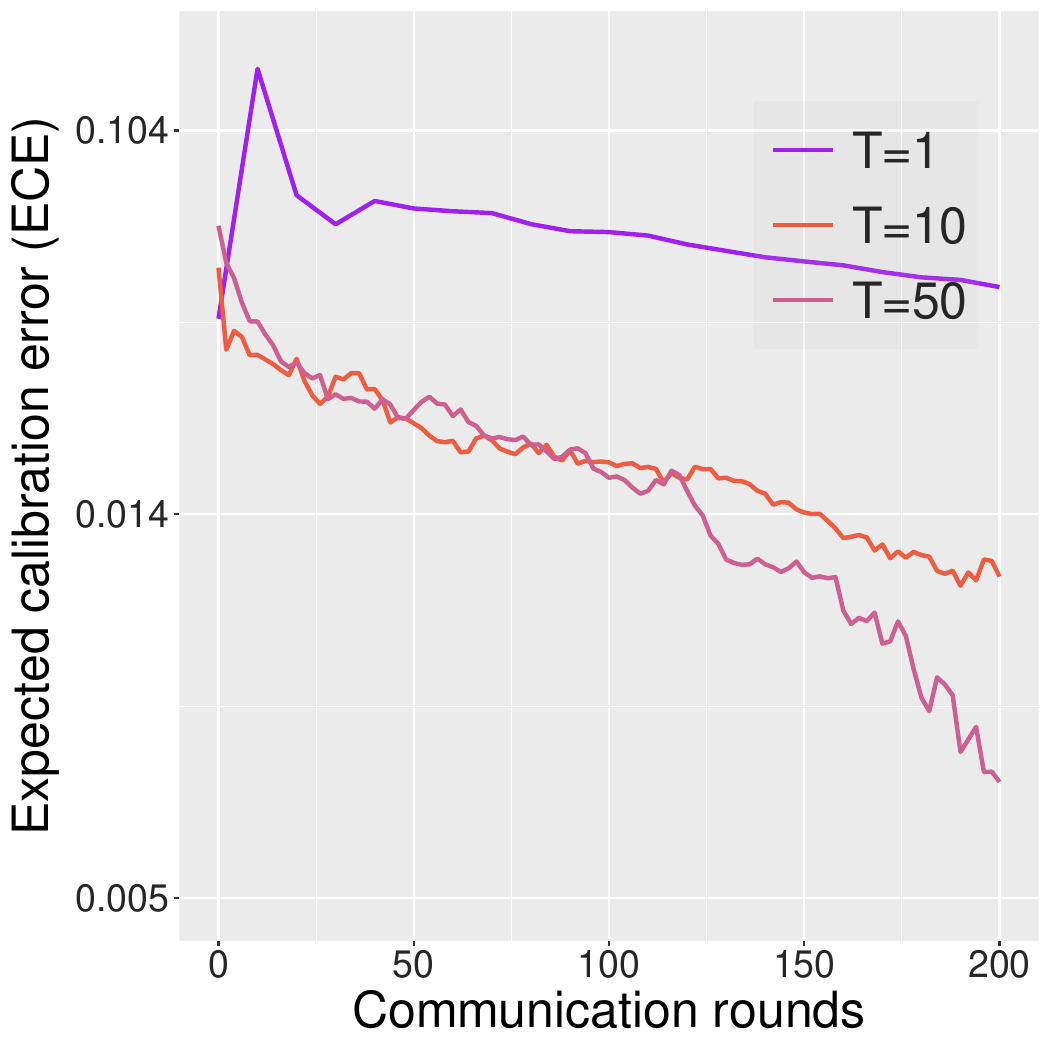}
    \end{minipage}%
    }%
  \vskip -0.1in
  \caption{The impact of leapfrog step $K$ and local step $T$ on FA-HMC applied to train a two-hidden-layer neural network on the Fashion-MNIST datasets.}
  \label{figure:Fashion2}
\end{figure*}

\subsection{Application: Logistic Regression Model on KMNIST/CIFAR2}
We also apply FA-HMC to train logistic regression on the Kuzushiji-MNIST (KM) \citep{clanuwat2018deep} and CIFAR10 dataset \citep{krizhevsky2009learning}. Specifically, we only use the first two classes (airplane and automobile) of the CIFAR10 dataset in the experiments to simplify the problem and denote it by CF2. 
The data points in each dataset are randomly split into 10 subsets of equal size for $N=10$ clients. We run FA-HMC under different settings of local step $T$ and leapfrog step $K$ with stochastic gradients that are calculated using a batch size of 1000 in each local device. 
As usual, we tune the step size $\eta$ in each setting and report the best statistics: prediction accuracy (AC), Brier Score (BS), and Expected Calibration Error (ECE) on the test dataset.
The choices of local step $T$ and leapfrog step $K$ are the same as those in Section \ref{sec:nn_exp}. 
\begin{figure*}[htbp]
    \centering
    \subfigure[CF2: AC]{
    \begin{minipage}[t]{0.16\linewidth}
    \centering
    \label{fig:cifar-accu}
    \includegraphics[width=\linewidth]{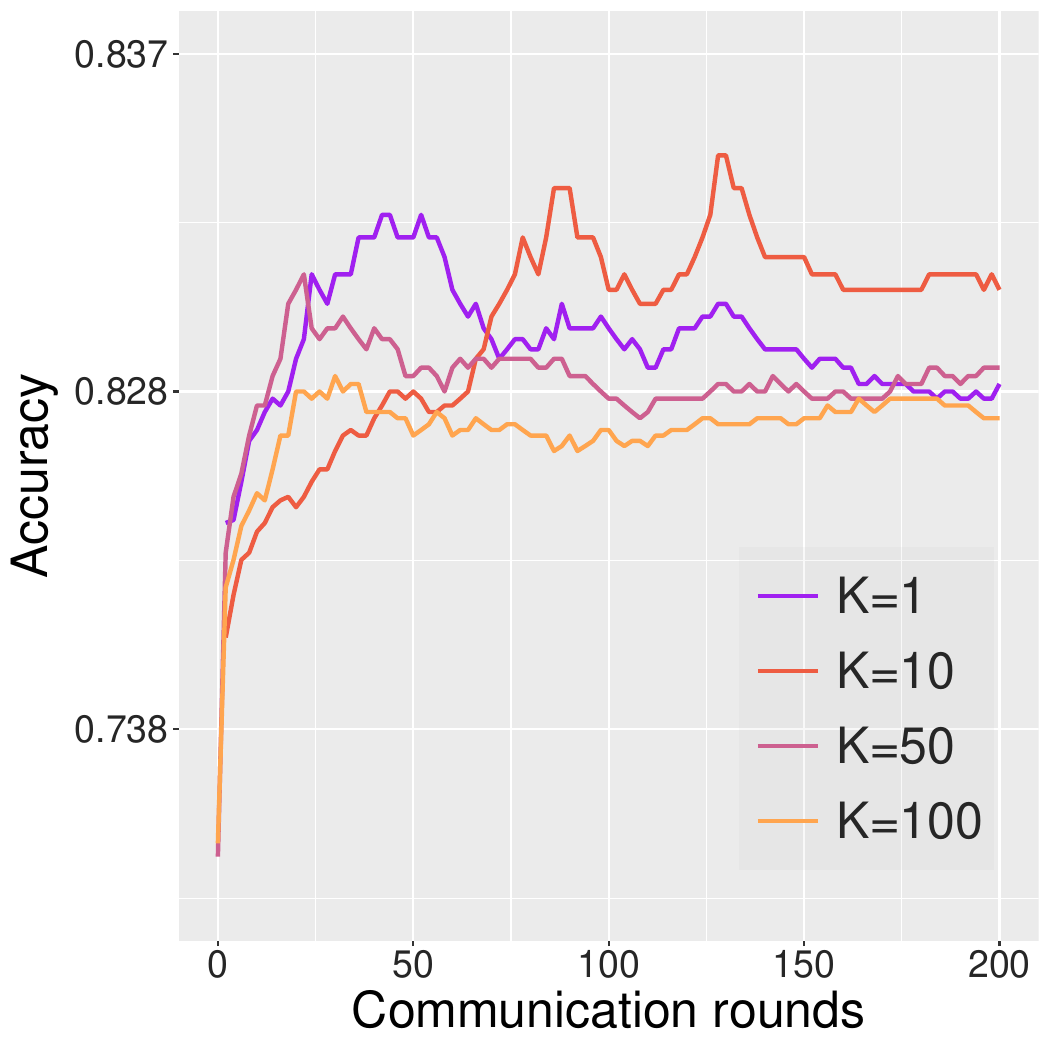}
    \end{minipage}%
    }%
    \subfigure[CF2: BS]{
    \begin{minipage}[t]{0.16\linewidth}
    \centering
    \label{fig:cifar-brier}
    \includegraphics[width=\linewidth]{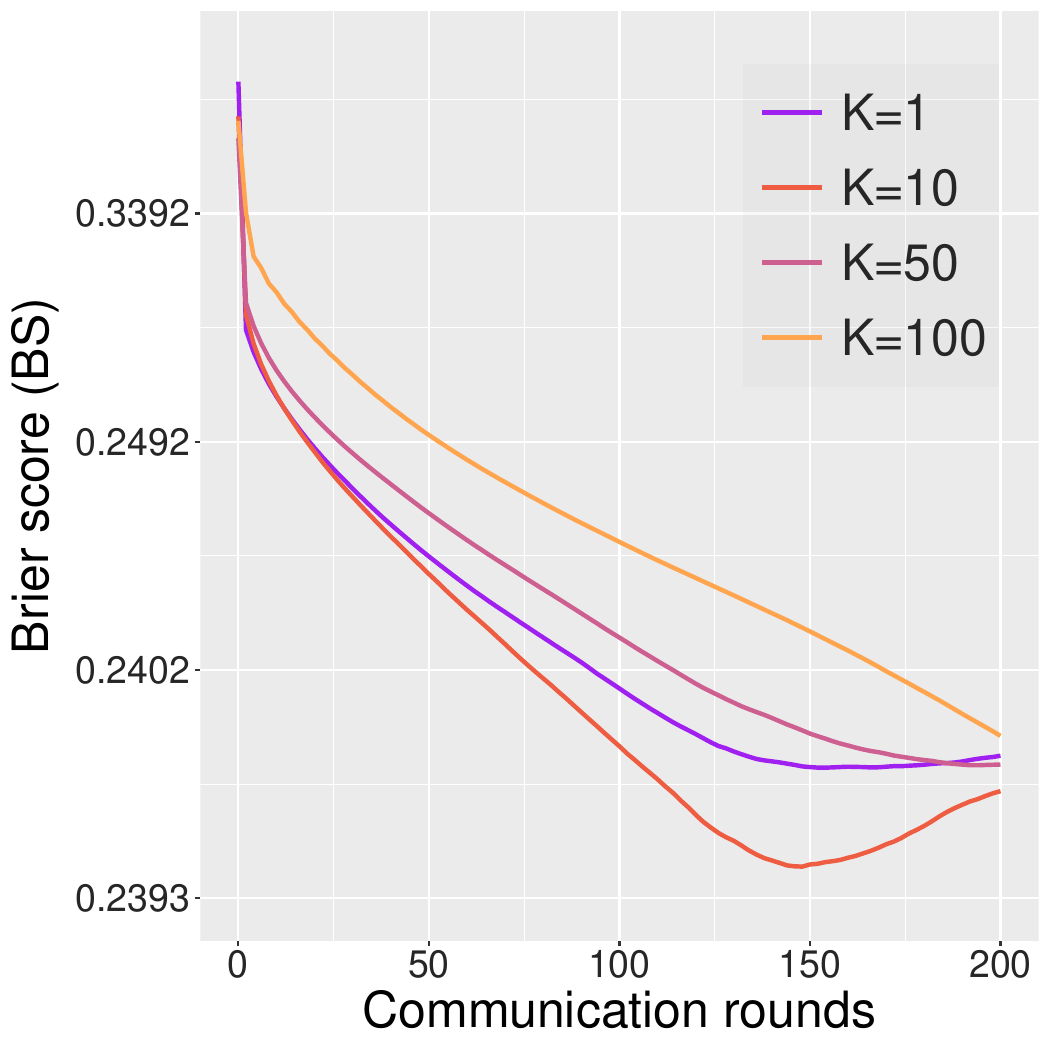}
    \end{minipage}%
    }%
    \subfigure[CF2: ECE]{
    \begin{minipage}[t]{0.16\linewidth}
    \centering
    \label{fig:cifar-ece}
    \includegraphics[width=\linewidth]{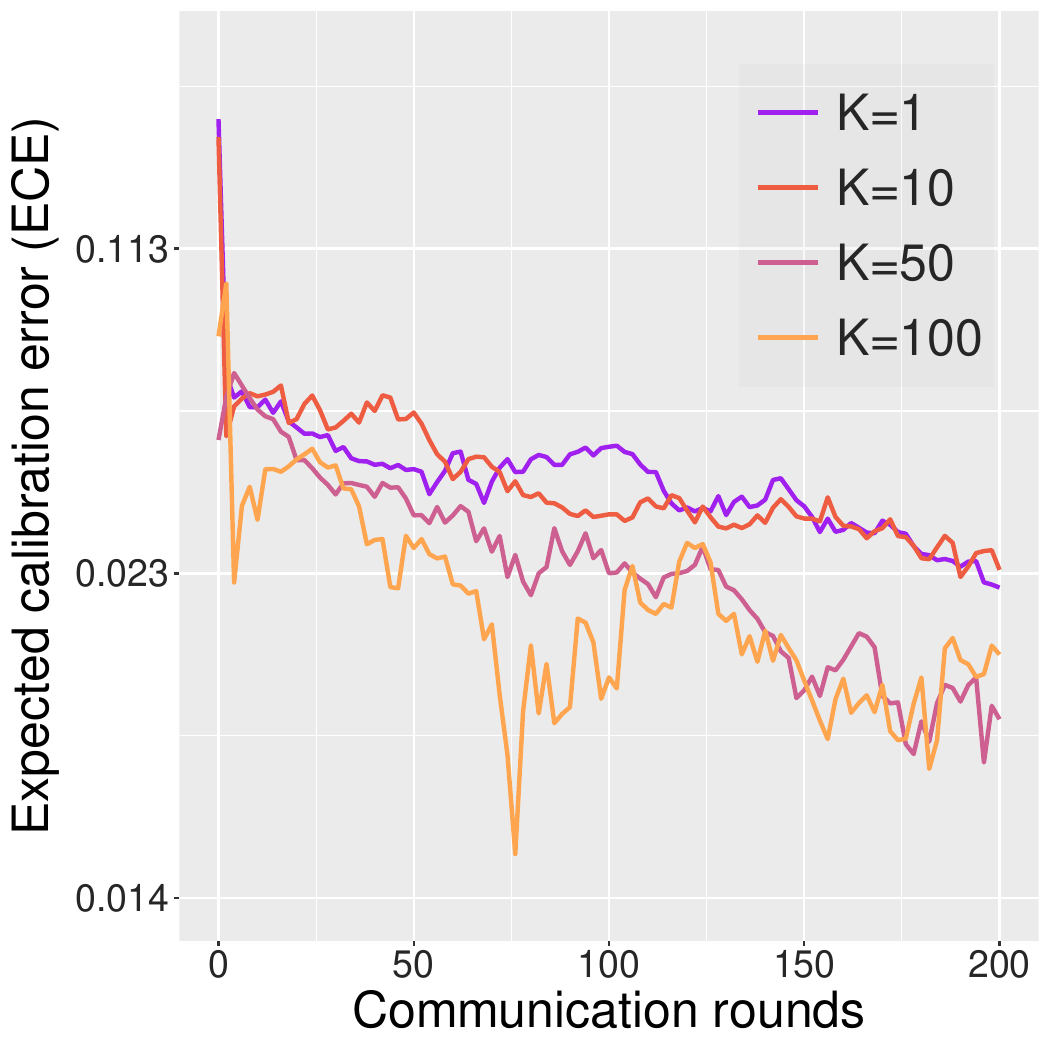}
    \end{minipage}%
    }%
    \subfigure[KM: AC]{
    \begin{minipage}[t]{0.16\linewidth}
    \centering
    \label{fig:kmnist-accu}
    \includegraphics[width=\linewidth]{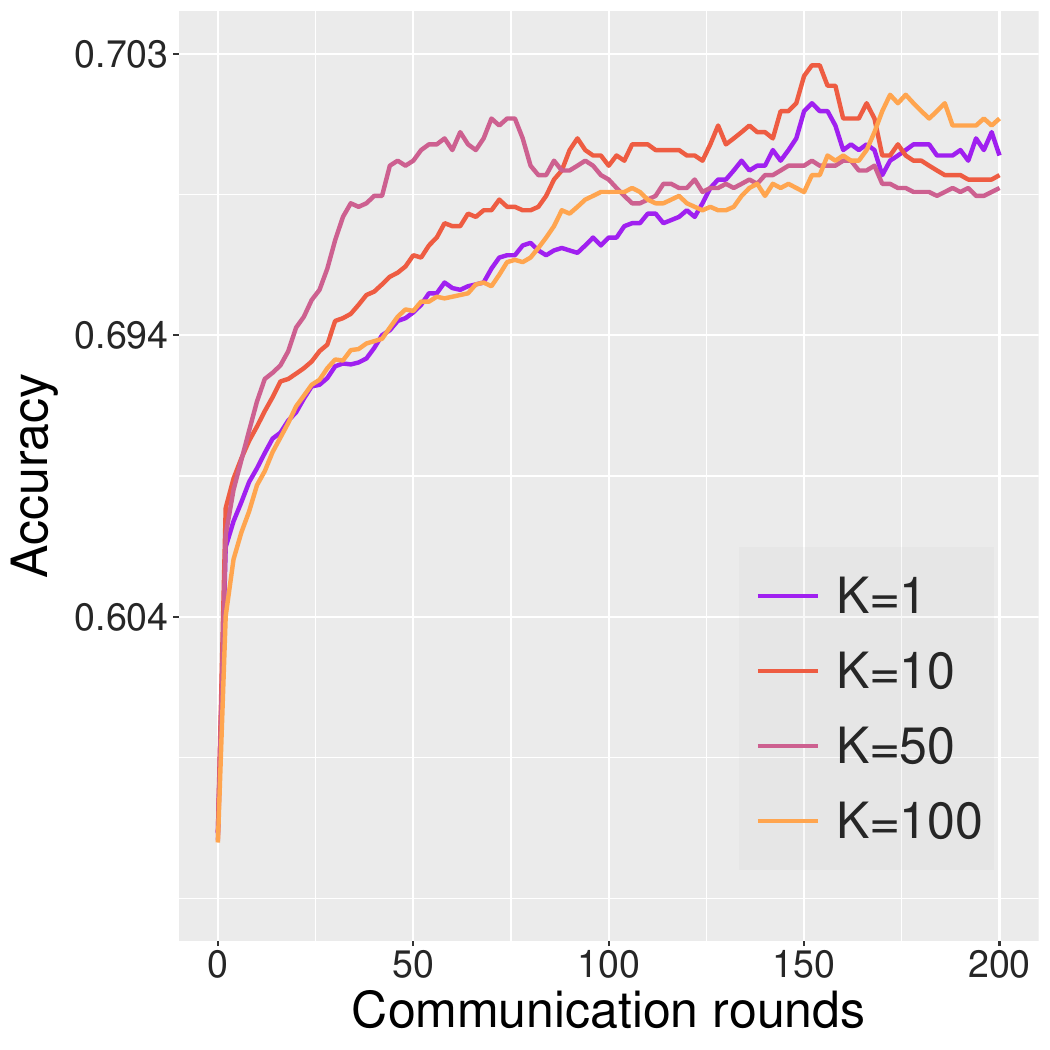}
    \end{minipage}%
    }%
    \subfigure[KM: BS]{
    \begin{minipage}[t]{0.16\linewidth}
    \centering
    \label{fig:kmnist-brier}
    \includegraphics[width=\linewidth]{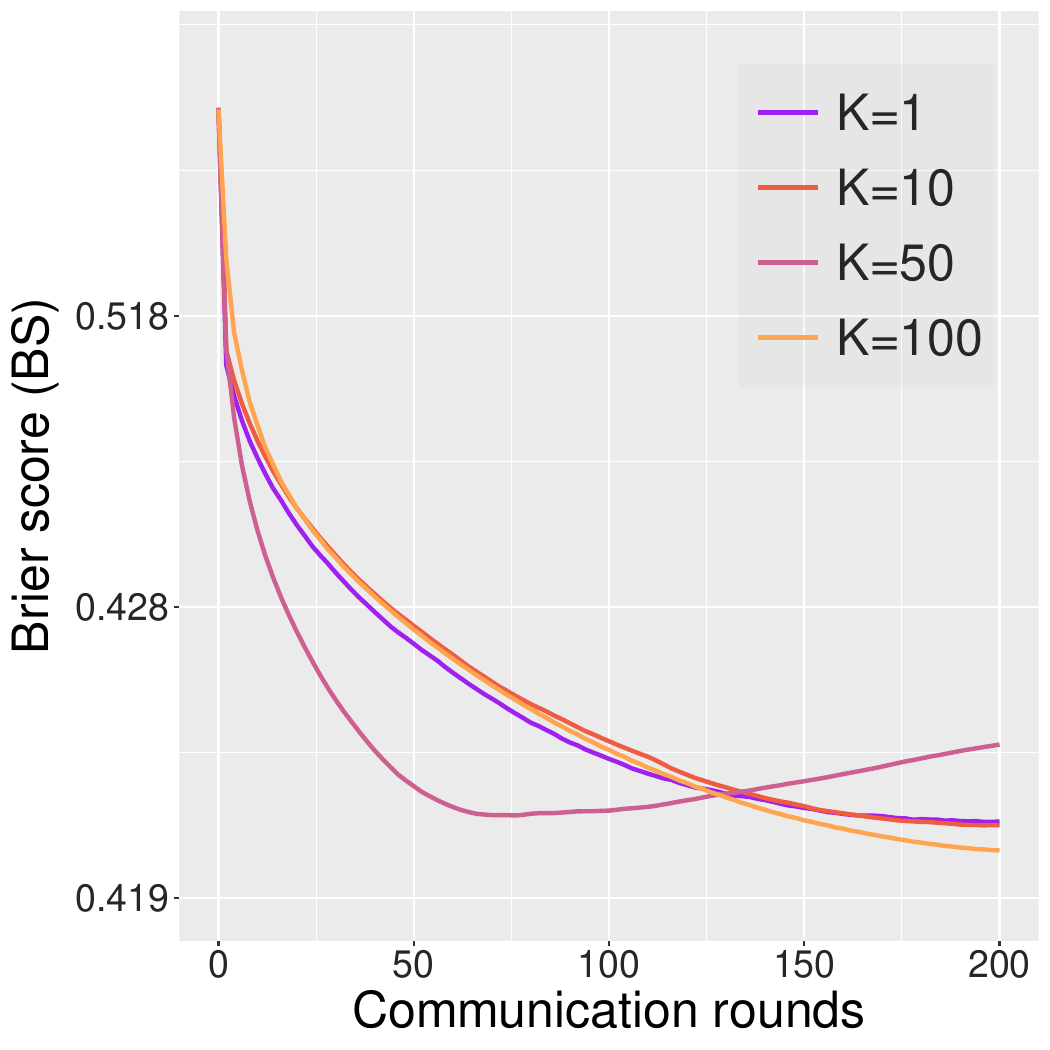}
    \end{minipage}%
    }%
    \subfigure[KM: ECE]{
    \begin{minipage}[t]{0.16\linewidth}
    \centering
    \label{fig:kmnist-ece}
    \includegraphics[width=\linewidth]{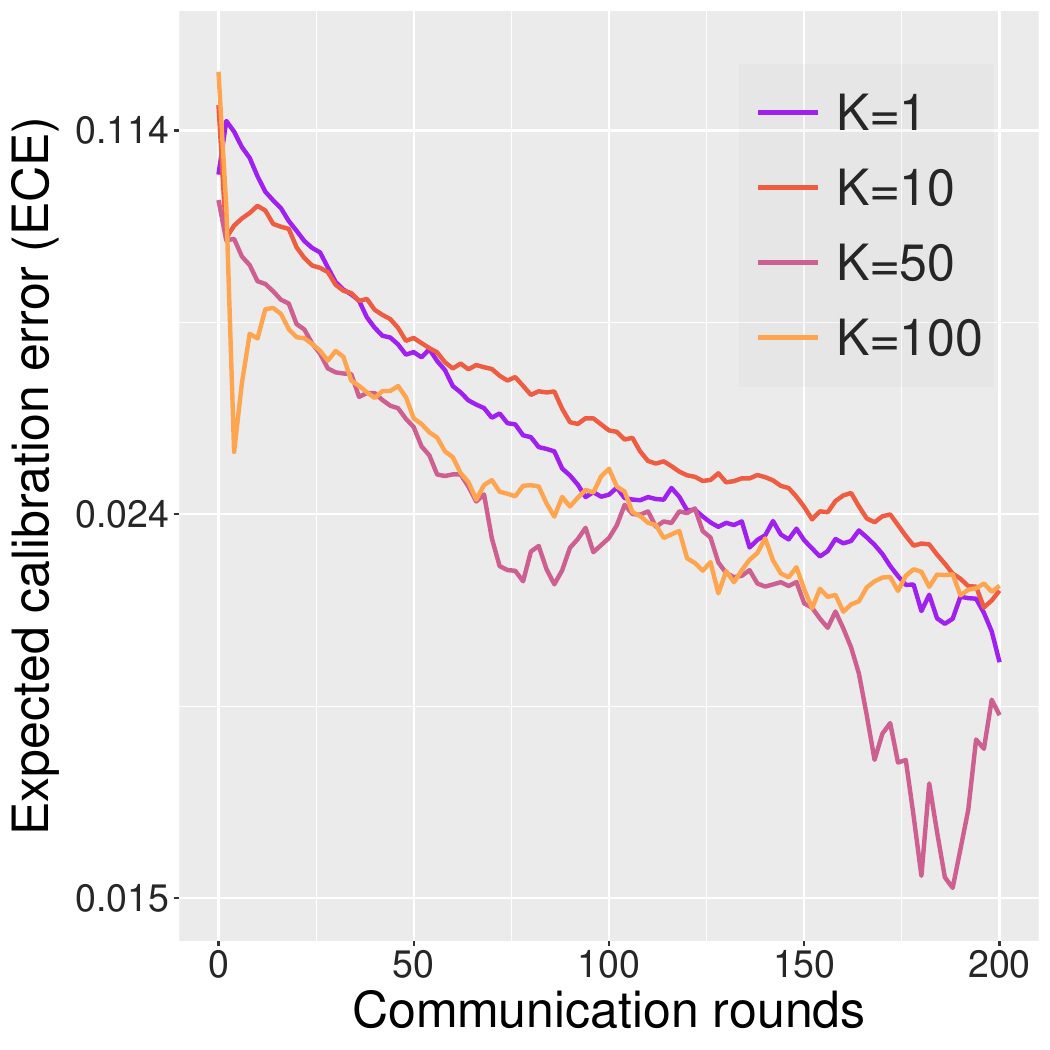}
    \end{minipage}%
    }%
  \vskip -0.1in
  \caption{The impact of leapfrog step $K$ on FA-HMC applied on the CIFAR2 and KMNIST datasets.}
  \label{figure:cifar_kmnist_leap}
\end{figure*}

\begin{figure*}[htbp]
    \centering
    \subfigure[CF2: AC]{
    \begin{minipage}[t]{0.16\linewidth}
    \centering
    \label{fig:cifar-accu-local}
    \includegraphics[width=\linewidth]{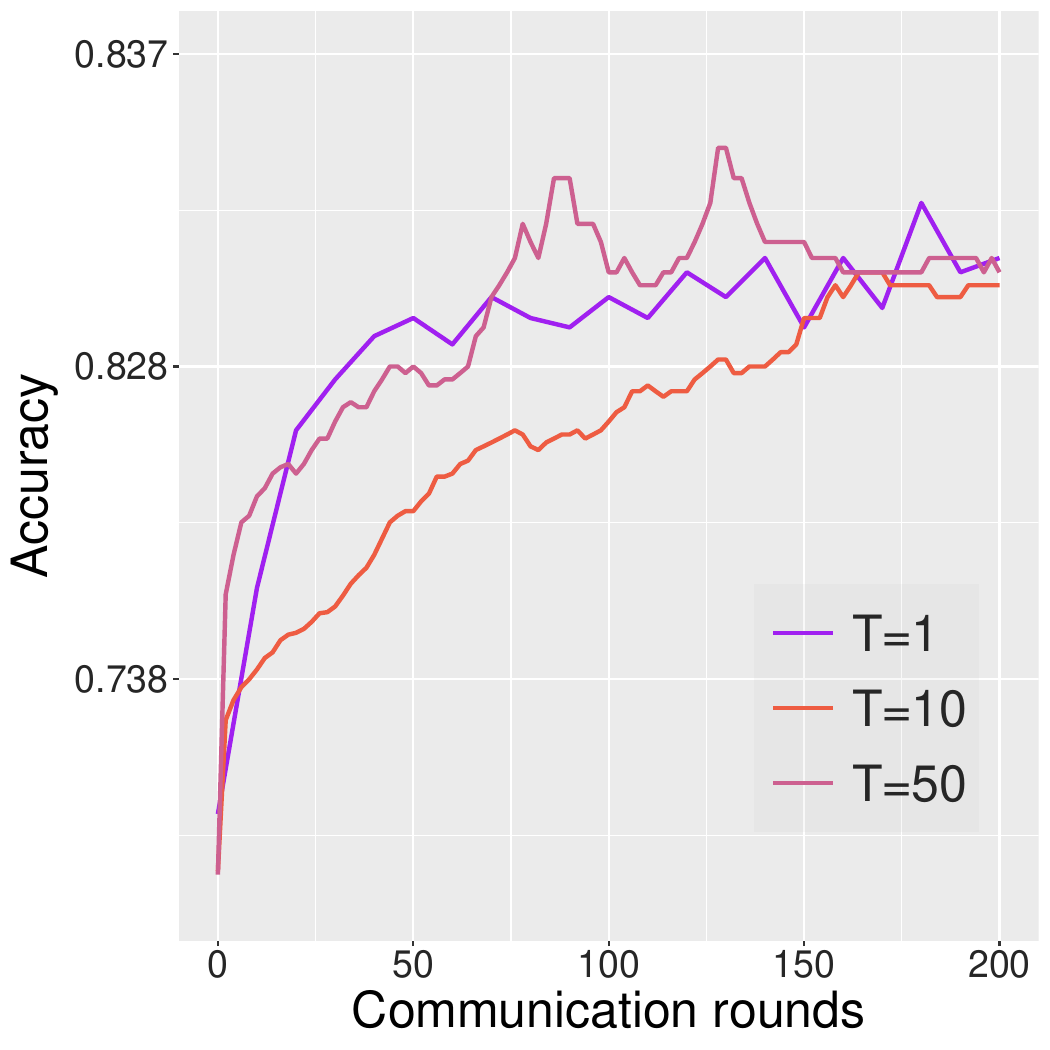}
    \end{minipage}%
    }%
    \subfigure[CF2: BS]{
    \begin{minipage}[t]{0.16\linewidth}
    \centering
    \label{fig:cifar-brier-local}
    \includegraphics[width=\linewidth]{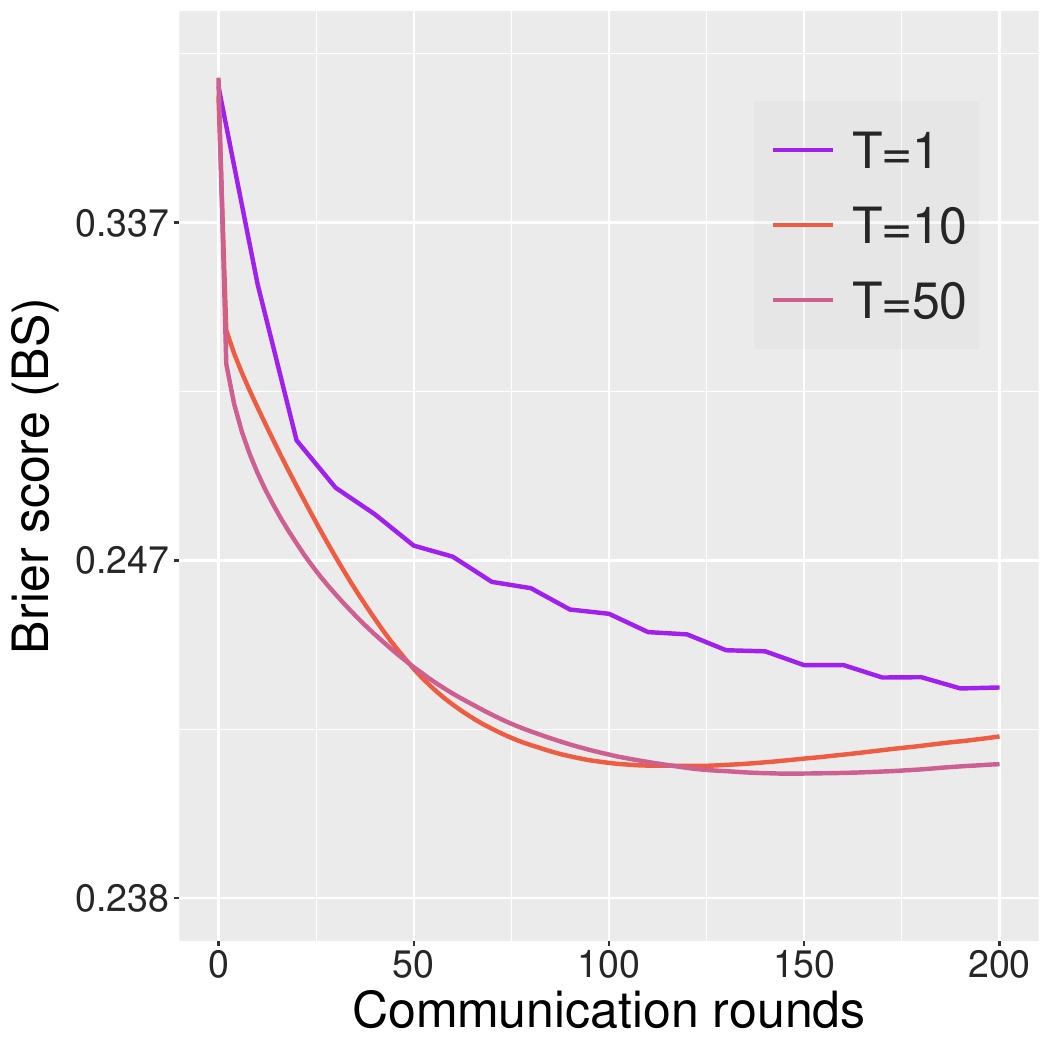}
    \end{minipage}%
    }%
    \subfigure[CF2: ECE]{
    \begin{minipage}[t]{0.16\linewidth}
    \centering
    \label{fig:cifar-ece-local}
    \includegraphics[width=\linewidth]{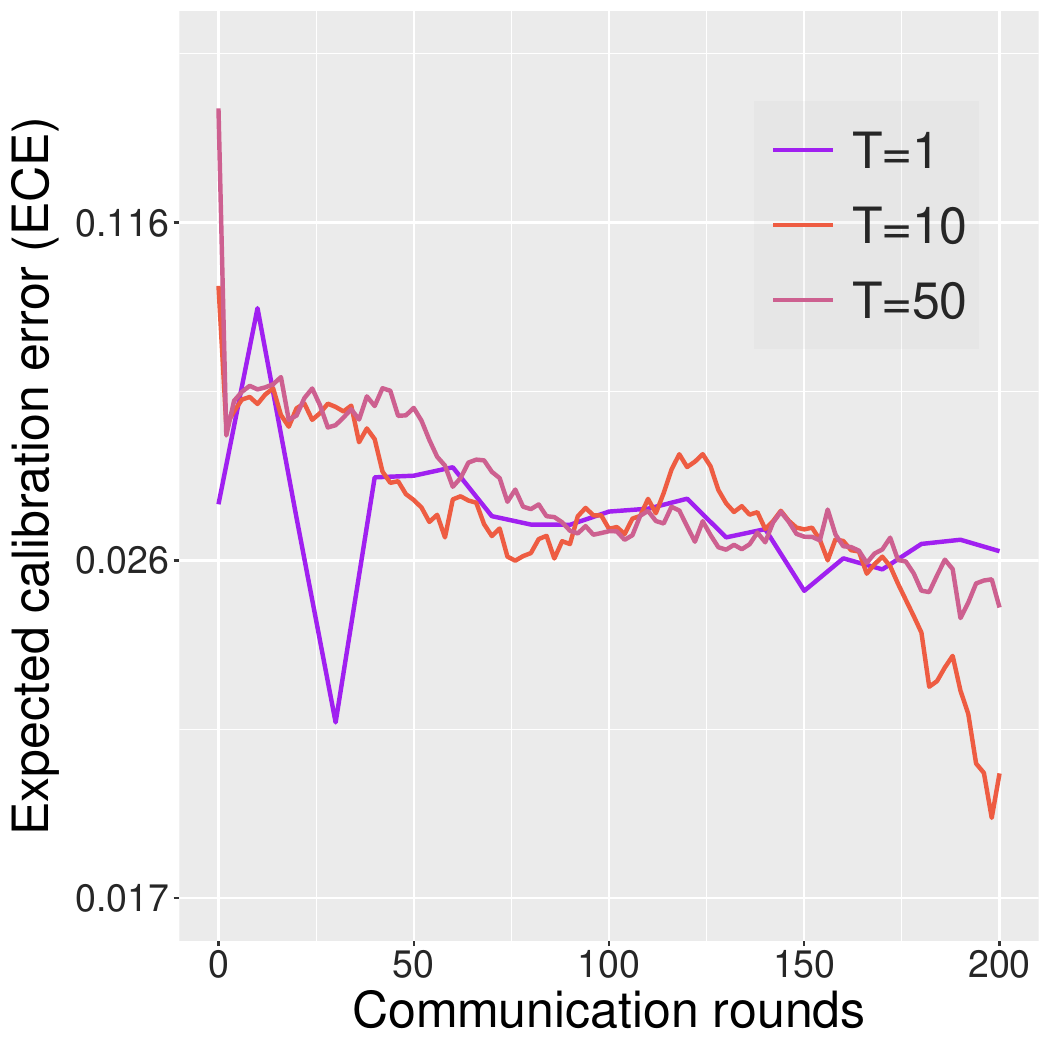}
    \end{minipage}%
    }%
    \subfigure[KM: AC]{
    \begin{minipage}[t]{0.16\linewidth}
    \centering
    \label{fig:kmnist-accu-local}
    \includegraphics[width=\linewidth]{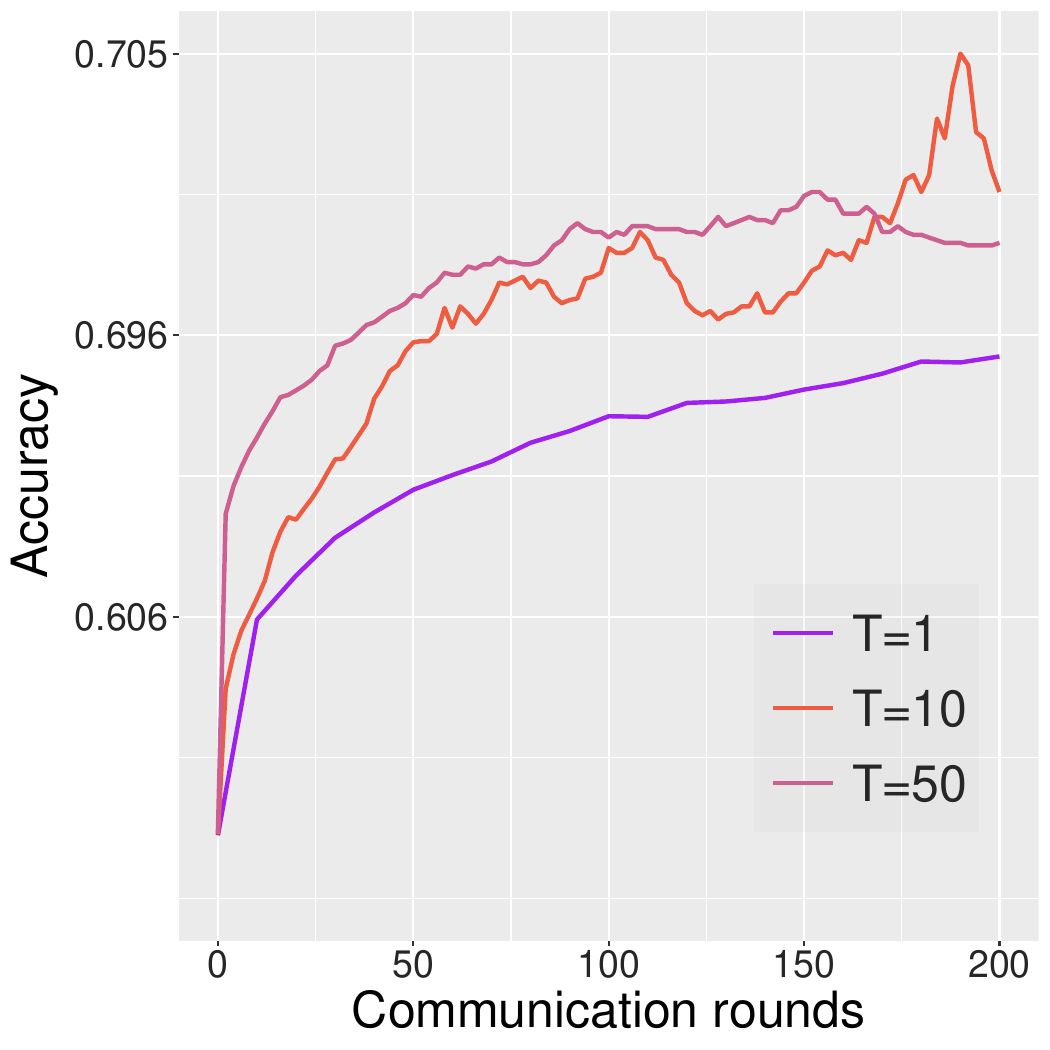}
    \end{minipage}%
    }%
    \subfigure[KM: BS]{
    \begin{minipage}[t]{0.16\linewidth}
    \centering
    \label{fig:kmnist-brier-local}
    \includegraphics[width=\linewidth]{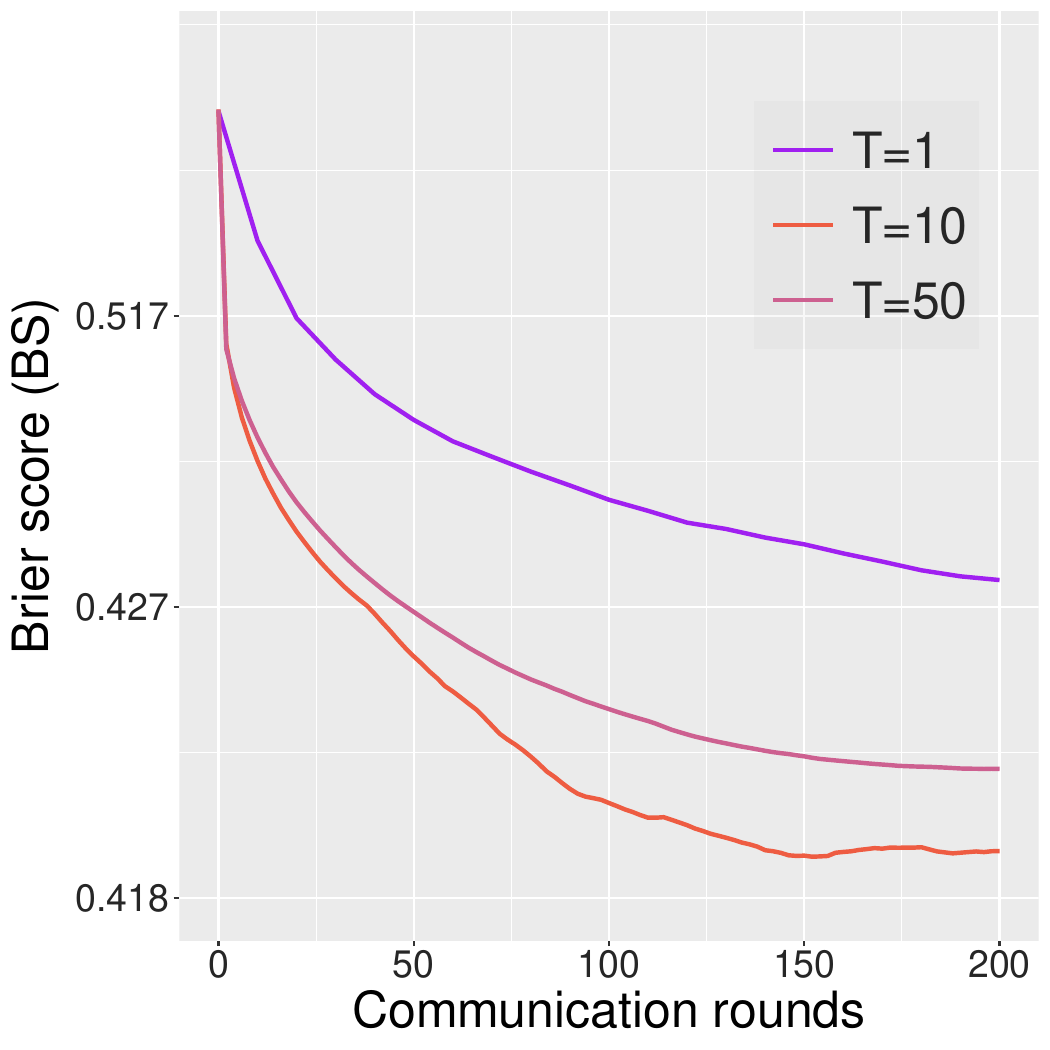}
    \end{minipage}%
    }%
    \subfigure[KM: ECE]{
    \begin{minipage}[t]{0.16\linewidth}
    \centering
    \label{fig:kmnist-ece-local}
    \includegraphics[width=\linewidth]{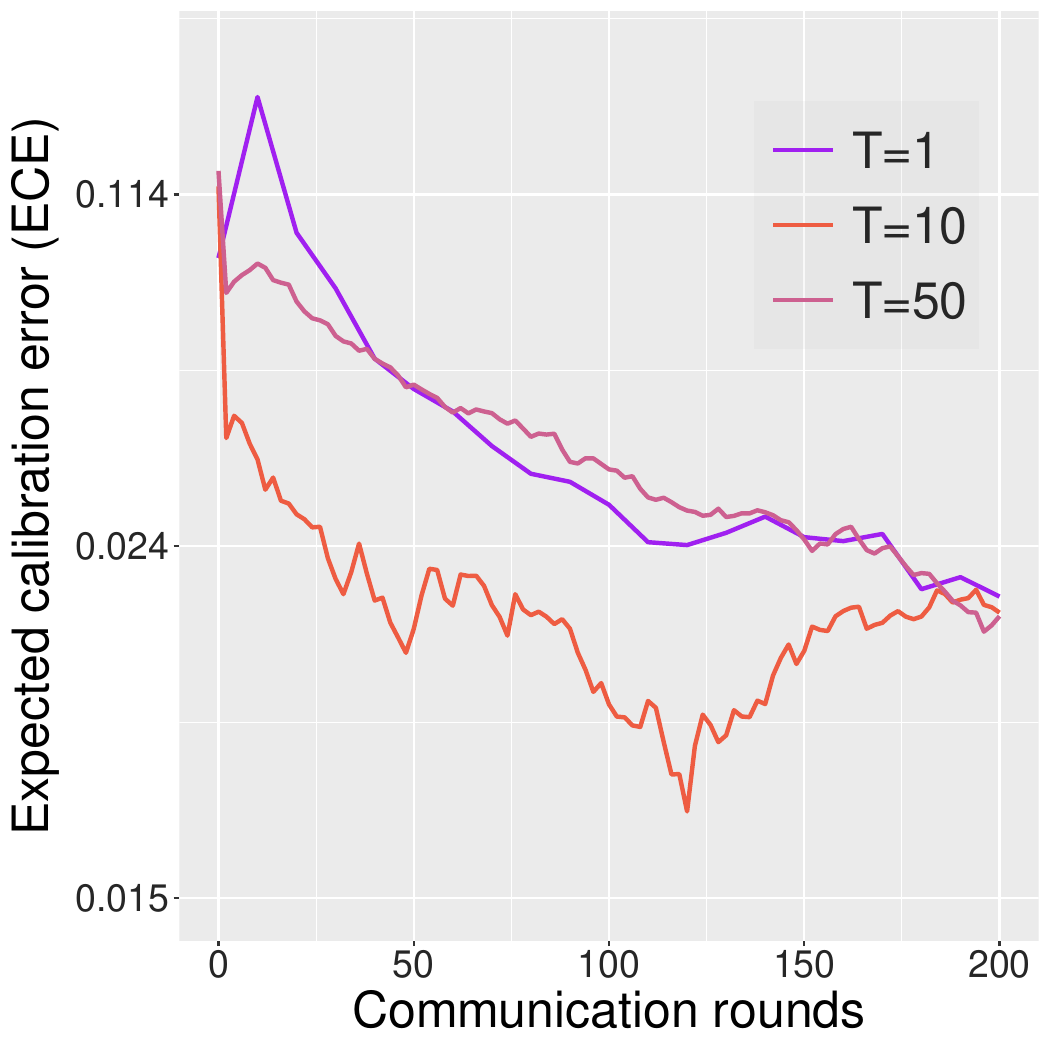}
    \end{minipage}%
    }%
  \vskip -0.1in
  \caption{The impact of local step $T$ on FA-HMC applied on the CIFAR2 and KMNIST datasets.}
  \label{figure:cifar_kmnist_local}
\end{figure*}

The performance of FA-HMC using different leapfrog steps $K$ is shown in Figure \ref{figure:cifar_kmnist_leap}. We see that the optimal leapfrog step $K$ varies with different test statistics and datasets, for example, the best $K$ is 10 for AC and BS and the best $K$ is larger than 10 for ECE on the CF2 dataset, and none of the experiments support $K=1$ (i.e., FA-LD) as the optimal leapfrog step, showcasing the advantage of FA-HMC (i.e., $K>1$) over FA-LD. We also study the impact of different local steps $T$ as shown in Figure \ref{figure:cifar_kmnist_local}. We observe that except for the AC metric on CF2, federated learning with $T>1$ outperforms the standard baseline $T=1$ on the rest of the metrics on both the CF2 and KM datasets.

\section{Conclusions and Future Work}
\label{section_conclude}
In this paper, we develop a tight theoretical guarantee for FA-HMC and provide suggestions to speed up FA-HMC.  Through experimentation, we demonstrate that FA-HMC outperforms FA-LD. We believe that FA-HMC potentially captures the similarities between local nodes, giving it an advantage over FA-LD. For future directions, it would be interesting to explore if further improvements can be achieved by addressing heterogeneity in local leapfrog steps. Note that for second-order methods, one would need to tackle heterogeneity both on local positions and local momentum parameters. For example, motivated by \cite{kkm+20}, suppose at $t_0$-th iteration (communication round), each local device obtain $\theta_{t_0+T}$ and $\nabla f(\theta_{t_0}):=\sum_{c=1}^Nw_c\nabla f^{(c)}(\theta_{t_0})$. Then each local device with local loss function $f^{(c)}$ is going to perform the following update for $k=0,1,\ldots,K-1$, $t=t_0+T+1,\ldots,t_0+2T$
\begin{align*}
\theta^{(c)}_{t,k+1}=&\theta^{(c)}_{t,k}+\eta_t p^{(c)}_{t,k}-\frac{\eta_t^2}{2}\big(\nabla f^{(c)}(\theta^{(c)}_{t,k})
    -\nabla f^{(c)}( \theta_{t_0})+\nabla f( \theta_{t_0})\big),\\
    p^{(c)}_{t,k+1}=&p^{(c)}_{t,k}-\frac{\eta_t}{2} \big(\nabla f^{(c)}(\theta^{(c)}_{t,k})+\nabla f^{(c)}(\theta^{(c)}_{t,k+1})
    -2\nabla f^{(c)}( \theta_{t_0})+2\nabla f( \theta_{t_0})\big).
\end{align*}
The complete version is deferred to Algorithm B.2 
in the Supplementary Material.

%
%
%
Another direction we are working on is to consider the privacy guarantee of the sampling algorithms and compare them with optimization algorithms.

It would also be interesting to examine  the above directions and the application of underdamped Langevin Monte Carlo algorithm to federated learning as future research.

\bibliographystyle{jasa3}
\bibliography{reference}





\end{document}


\maketitle
\appendix
In this supplementary file, we first review related literature on federated sampling and HMC in Section \ref{related-work}. Then, we include the algorithms omitted in the main text in Section~\ref{sec:omitalg}., and the low bound result of $t_\epsilon$ in Section \ref{sec:LB}. Subsequently, we first organize the settings, notations and preliminary lemmas in Section~\ref{sec:pre_sett}. Following that, we provide the proof for the main result in Section~\ref{sec:pfpropconver} and offer a sketch proof for the convergence of vanilla FA-HMC in Section~\ref{sec:sketchproof}. This is done to enlighten the understanding of more complex cases. We then present the formal proofs for the main convergence result and their associated preliminary lemmas in Sections~\ref{sec:pfmainthm}-\ref{sec:pfpropdynamic}. We also discuss how to extend the proofs to more general stochastic gradient assumptions in Section~\ref{sec:relaxation}.
Finally, in Section \ref{secton:figures}, we show additional plots with error bars that are omitted from the main text.



\section{Related Work}
\label{related-work}

\paragraph{Federated optimization} Federated optimization  is a collaborative learning  that trains a model without a direct share of user data. In addition to encryption techniques to ensure secure communications, a major focus is to minimize communications in distributed computing \citep{dcm+12,ss15,mmra16,mmr+17,yyz19,kkm+20, ljz+21, hlsy21}. In particular, federated averaging  \citep[FedAvg,][]{mmr+17}  proposed a scalable paradigm by conducting more local steps to achieve this target, which further motivates the study   based on non-iid data \citep{zhao18, Sattler20, lhy+20, LS20} and asynchronous computing \citep{xie20}.

\paragraph{Federated Langevin sampling:} 

Vanilla-distributed Monte Carlo methods require frequent communications of user data, which has data privacy concerns. To tackle this issue while maintaining model uncertainty, posterior averaging is empirically studied in federated learning to reduce data leakage risks \cite{agxr21, F-SGLD, hw+21}. To guarantee theoretical properties in privacy and communication efficiency, further analysis has been established in \cite{deng2021convergence, plassier2022federated,Maxime2021} based on multiple local steps and compressed operators, respectively.


\paragraph{Distributed Monte Carlo methods} 
Despite the advances of Monte Carlo methods in big data problems \citep[e.g., ][]{Welling11, Ahn12, chen2014stochastic, Chen15, ccbj18, Yian_higher, deng2020, icsgld}, the computation is still inefficient given limited devices. To mitigate this issue, \cite{Neiswanger13, wang13, Minsker14} proposed sub-posterior aggregation to speed-up the computations with distributed devices. In addition, other types of parallel paradigms, such as synchronous, asynchronous, decentralized computing, are also conducted with Monte Carlo computations \cite{Nishihara14, Ahn14_icml, chen16_distributed, Chowdhury18, Li19_v2}.

\paragraph{Hamiltonian Monte Carlo:}

 HMC is widely known as a state-of-the-art sampling algorithm \citep{Neal12, no_U_turn_sampler, chen2014stochastic}. However, a theoretical underpinning for the convergence rate study has been lacking until recently. The irreducibility and geometrical ergodicity have been studied in \cite{Irreducibility_HMC}. \cite{Mangoubi18_leapfrog} showed that the
number of gradient evaluations for unadjusted HMC only depends on the dimension with an order as low as $d^{1/4}$ given proper smoothness. On another direction, \cite{mangoubi2017rapid} presented the convergence rate with an upper bound of   $O(\kappa^2)$ for the unadjusted HMC; \cite{Yin_Tat_18} improved that result to $O(\kappa^{1.5})$ by considering an ODE solver in the HMC algorithm and \cite{chen2019optimal} further refined the mixing time to $O(\kappa)$. 

Without the Metropolis-Hastings correction, the corresponding Markov chains lead to a consistent bias and require a mixing time that is polynomially dependent on $O(1/\epsilon)$ to control such errors; by contrast, the mixing time for Metropolized HMC can be further improved to  $O(\log(1/\epsilon))$. To analyze the Metropolized HMC algorithms, \cite{coupling_HMC} proposed 
the coupling methods and showed that the mixing-time bound scales in the order of $O(d^{1.5})$; \cite{chen2020fast} presented the non-asymptotic convergence and confirmed that Metropolished HMC is strictly faster than MALA and other basic Metropolished algorithms.

In big data settings, querying the entire dataset becomes quite expensive and increases the challenge to obtain the desired performance \citep{Betancourt_15, talldata17, HMC_preserve}; to tackle this issue, \cite{zou2021convergence} studied the convergence of HMC based on stochastic gradients achieves the target distribution with an error up to $O(\sqrt{\eta})$ (learning rate); and variance reduction techniques were further proposed to reduce that error; in decentralized settings, \cite{gurbuzbalaban2021decentralized} studied non-asymptotic guarantees of SGHMC by constructing a proper Lyapunov function and appropriate parameters; they further view decentralized SGHMC as a noisy heavy ball algorithm \citep{Flammarion_bach_15, Can_19, Xin_2020}.

\section{Algorithms Omitted in the Main Text}\label{sec:omitalg}


\begin{algorithm}[]
	\caption{de-bias leapfrog $h_{\text{de-LF}}$ }
	\begin{algorithmic}
		\STATE {\bfseries Input:} Energy function $f^{(c)}(\cdot)$; Shared gradient $\nabla  f( \theta_{t_0})$ and  parameter $\theta_{t_0}$; Initial parameters $ \theta_{t_0}$, momentum $p_0$; learning rate $\eta$; leapfrog step $K$; $k=0$
\WHILE{$k\leq K$:}
\STATE  $\theta_{k+1}=\theta_k+\eta p_k-\frac{\eta^2}{2}\big(\nabla f^{(c)}(\theta_k)-\nabla f^{(c)}( \theta_{t_0})+\nabla f( \theta_{t_0})$\big)
\STATE $p_{k+1}=p_k-\frac{\eta}{2} (\nabla f^{(c)}(\theta_k)+\nabla f^{(c)}(\theta_{k+1})-2\nabla f^{(c)}( \theta_{t_0})+2\nabla f( \theta_{t_0})$)
\STATE $k=k+1;$
\ENDWHILE
\STATE {\bfseries Output:} $\theta_K$
	\end{algorithmic}
\end{algorithm}

\begin{algorithm}[]
	\caption{debias FL-HMC algorithm}\label{alg:disdeFLHMC}
	\begin{algorithmic}
		\STATE {\bfseries Input:} Initial parameters $\theta^{(c)}_0=\theta_0$, $\theta_{-T}=\nabla f(\theta_{-T})={\bf 0}$, $t=0$; Stepsize function $\eta_t=\eta(t)$; Local update step $T$; Leapfrog update step $K$
		\WHILE{stopping conditions are not satisfied}
  \STATE {\bf sample }momentum $p_t^{(c)}\sim N(0,\sigma_t^{(c)2}\mathbb{I}_d)$
  \IF{$t\equiv 0 (\mathrm{mod }\ T)$}
  \STATE Broadcast  $\theta_t:=\sum_{c=1}^Nw_c\theta^{(c)}_t$, $\nabla f(\theta_{t-T}):=\sum_{c=1}^Nw_c\nabla f^{(c)}(\theta_{t-T})$ and set $\theta^{(c)}_{t+1,0}=\theta_t$
  \ELSE
  \STATE $\theta^{(c)}_{t+1,0}=\theta^{(c)}_t$
  \ENDIF
  \STATE {\bf update} $\theta^{(c)}_{t+1}=h_{\text{de-LF}}(f^{(c)},\theta^{(c)}_{t+1,0},\nabla f(\theta_{t-T}),\theta_{t-T},p^{(c)}_t,\eta_t,K) $ in parallel for all devices, $t=t+1$
  \ENDWHILE
	\end{algorithmic}
\end{algorithm}
\section{Lower Bound for  FA-HMC Algorithm}\label{sec:LB}
In the main text, we claim a trade-off between communication and divergence based on the upper bounds results in Theorem \ref{thm:convergrate:sg}.  In this section, we show that for certain sampling problems, 
the upper bound results on vanilla FA-HMC above also match the lower bound below. 
\begin{restatable}[Dimensional tight lower bound for ideal FA-HMC process]{theorem}{thmLBconFLHMC}\label{thm:LBconFLHMC}
Under Assumptions \ref{assum:convex}-\ref{assum:Hsmooth},  if we initialize $\theta_0\in N(0,\sigma^2\mathbb{I}_d)$ for any $\sigma^2> 0$, then there is a sampling task where the least running time $t$ for continuous FA-HMC to achieve ${\cal W}_2(\theta_t,\theta^\pi)^2\leq  \epsilon^2$  is taken when $t=\Omega(\sqrt{d}T\log(d/\epsilon)/\epsilon)$.
\end{restatable}
The result provides a lower bound on the continuous FA-HMC algorithm, where the discretization error does not exist. Therefore, even if FA-HMC runs with an infinitely small learning rate, the convergence rate cannot be improved in general.

\section{Preliminary Settings}\label{sec:pre_sett}

\subsection{Continuous HMC Process of Single Device}

The dynamic of HMC regrading the loss(energy) function $f(\cdot)$ is characterized by the following two ODEs:
\begin{align}
\begin{split}\label{eq:PDE}
    \frac{d\theta(t)}{dt}=&p(t),\\
    \frac{dp(t)}{dt}=&-\nabla_\theta f(\theta(t)).
\end{split}
\end{align}

 Noting that the solution of 
 \begin{equation}\label{eq:contin}
     \theta(t)=\theta(0)+\int_0^t p(s)ds=\theta(0)+\int_0^t\Bigl(p(0)+ \int_{0}^s \nabla_\theta f(\theta(u))du\Bigr)ds 
 \end{equation}
involves double integration with respect to $\nabla_\theta f(\theta(u))$. Denote ${\cal H}^{\pi}_s(\theta(0))=\theta(s)$, the trajectory starting from $\theta(0)$ with an independently sampled momentum $p(0)$. Further, we define
\[
\theta_t^{\pi}=({\cal H}^{\pi}_{K\eta_t})^{\frac{t}{K\eta_t}}(\theta^{\pi}),\qquad t\in \mathbb{N} \cdot K\eta_t.
\]
Note that each mapping ${\cal H}_{K\eta_t}$ uses an independently sampled momentum.
 

\subsection{Notations for FA-HMC}

For iteration $t=0,1,\ldots$, $k=0,1,\ldots$, define $\theta^{(c)}_{t,k}$ and $p^{(c)}_{t,k}$ to be the position and momentum parameters of the $c$-th local device at iteration $t$ and $k$-th leapfrog step, respectively and 
\[
\theta_t=\sum_{c=1}^Nw_c\theta^{(c)}_t,
\quad p_t=\sum_{c=1}^Nw_cp^{(c)}_t,\quad  p^{(c)}_{t,0}=p^{(c)}_t,\quad \theta^{(c)}_{t,0}=\Bigl\{\begin{array}{ll}
   \theta_t,  &  \text{ if $\frac{t}{T}\in \mathbb{Z}$}\\
   \theta^{(c)}_t,  &  \text{if $\frac{t}{T}\notin \mathbb{Z}$}
\end{array}
\]
We use the following short hand
\[
\sigma_t^{(c)}:=(\mathbb{E}\|p_t^{(c)}\|^2)^{1/2}=((\rho+\frac{1-\rho}{w_c})d)^{1/2},\qquad \sigma_t:=(\mathbb{E}\|p_t\|^2)^{1/2}=d^{1/2}.
\]

For iteration $t=0,1,\ldots$ and leapfrog step  $k=1,2\ldots,K$, from the Algorithm 3,
by recursive calculations, if the true gradient is used, then
\begin{align}
\begin{split}\label{eq:iteraform:nonran}
        p^{(c)}_{t,k}=&p^{(c)}_{t,0}-\frac{\eta_t}{2}\nabla f^{(c)}(\theta^{(c)}_{t,0})-\sum_{j=1}^{k-1}\eta_t\nabla f^{(c)}(\theta^{(c)}_{t,j})-\frac{\eta_t}{2}\nabla f^{(c)}(\theta^{(c)}_{t,k})\\
    \theta^{(c)}_{t,k}=&\theta^{(c)}_{t,0}+k\eta_t p^{(c)}_{t,0}-\frac{k\eta_t^2}{2}\nabla f^{(c)}(\theta^{(c)}_{t,0})-\eta_t^2\sum_{j=1}^{k-1}(k-j)\nabla f^{(c)}(\theta^{(c)}_{t,j});
\end{split}
\end{align}


and if stochastic gradient is used, then
\begin{align}
\begin{split}\label{eq:iteraform}
        p^{(c)}_{t,k}=&p^{(c)}_{t,0}-\frac{\eta_t}{2}\widetilde{g}^{(c)}_{t,0}(\xi_0)-\frac{\eta_t}{2}\sum_{j=1}^{k-1}\Bigl(\widetilde{g}^{(c)}_{t,j}(\xi_{j-\frac{1}{2}})+\widetilde{g}^{(c)}_{t,j}(\xi_j)\Bigr)-\frac{\eta_t}{2}\widetilde{g}^{(c)}_{t,k}(\xi_{k-\frac{1}{2}})\\
    \theta^{(c)}_{t,k}=&\theta^{(c)}_{t,0}+k\eta_t p^{(c)}_{t,0}-\frac{k\eta_t^2}{2}\widetilde{g}^{(c)}_{t,0}(\xi_0)-\frac{\eta_t^2}{2}\sum_{j=1}^{k-1}(k-j)\Bigl(\widetilde{g}^{(c)}_{t,j}(\xi_{j-\frac{1}{2}})+\widetilde{g}^{(c)}_{t,j}(\xi_j)\Bigr)
\end{split}
\end{align}
where we denote
\[
    \widetilde{g}^{(c)}_{t,k}(\xi_{k-\frac{1}{2}})=\nabla \widetilde{f}^{(c)}(\theta^{(c)}_{t,k},\xi^{(c)}_{t,k-\frac{1}{2}}),\qquad  \widetilde{g}^{(c)}_{t,k}(\xi^{(c)}_k)=\nabla \widetilde{f}^{(c)}(\theta^{(c)}_{t,k},\xi^{(c)}_{t,k}),\qquad \widetilde{g}_{t,k}(\cdot)=\sum_{c=1}^Nw_c\widetilde{g}^{(c)}_{t,k}(\cdot).
\]
\subsection{Assumptions}

\begin{customassu}{3.1}[$\mu$-Strongly Convex]\label{assum:convex}
For each $c=1,2,\ldots, N$, $f^{(c)}$ is $\mu$-strongly convex for some $\mu>0$, i.e., $\forall x, y \in \mathbb{R}^{d}$, 
$f^{(c)}(y)\geq f^{(c)}(x)+\langle\nabla f^{(c)}(x), y-x\rangle+\frac{\mu}{2}\|y-x\|_{2}^{2}.$ 
\end{customassu}

\begin{customassu}{3.2}[$L$-Smoothness]\label{assum:smooth}
 For each $c=1,2,\ldots, N$, $f^{(c)}$ is $L$-smooth for some $L>0$, i.e., $\forall x, y \in \mathbb{R}^{d}$, $\|\nabla f^{(c)}(y)-\nabla f^{(c)}(x)\|\leq L\|x-y\|.$
\end{customassu}

\begin{customassu}{3.3}[$L_H$-Hessian Smoothness]\label{assum:Hsmooth} For each $c=1,2,\ldots, N$, $f^{(c)}$ is $L_H$ Hessian smoothness, i.e., for any  $\theta_1,\theta_2, p\in\mathbb{R}^d$,
$
\|\big(\nabla^2 f^{(c)}(\theta_1)-\nabla^2 f^{(c)}(\theta_2)\big)p\|^2\leq L_H^2\|\theta_1-\theta_2\|^2\|p\|_{\infty}^2.
$
\end{customassu}

\begin{customassu}{3.4}[$\sigma_g$-Bounded Variance]\label{assum:boundvar}For local device $c=1,2,\ldots,N$, and leapfrog step $k=1,2,\ldots,K$, $t=1,2,\ldots$, we have $\max_{x=k-1/2,k}\mathrm{tr(Var}(\nabla \widetilde{f}^{(c)}(\theta^{(c)}_{t,k},\xi^{(c)}_{t,x})|\theta^{(c)}_{t,k}))\leq \sigma_g^2Ld$,
 for some $\sigma_g>0$.
\end{customassu}
\subsection{Preliminary Lemmas}
\begin{restatable}[Jensen's inequality]{lemmma}{lemconvex}\label{lem:convex}
For any vectors $\{x_i\}_{i=1}^n$ and positive constants $\{\lambda_i\}_{i=1}^n$ with $\sum\lambda_i=1$, due to the convexity of $\|\cdot\|^2$, we have
\[
\begin{split}
&\|\sum_{i}\lambda_ix_i\|^2\leq  \sum_i\lambda_i\|x_i\|^2;\\
&\|\sum_{i}x_i\|^2\leq \sum_i\lambda_i\|x_i/\lambda_i\|^2=\sum_i \frac{1}{\lambda_i}\|x_i\|^2.
\end{split}
\]
For any vector-value function $x(t)$ and non-negative function $\lambda(t)$, if $\int\lambda(t)dt=1$, then
\[
\|\int \lambda(t)x(t)dt\|^2\leq \int\lambda(t)\|x(t)\|^2dt.
\]
\end{restatable}

The next lemma is Proposition 1 in \cite{durmus2019high}:

\begin{restatable}{lemmma}{lemOracle}\label{lem:Oracle} Let $x^*\in \mathbb{R}^d$ be the global minimizer of a loss function $f(\cdot)$ satisfying Assumption \ref{assum:convex} and $x^{\pi}$ be a random variable following distribution $\propto e^{-f(x)}$, then 
\[
\mathbb{E}\|x^*-x^{\pi}\|_2^2\leq \frac{d}{\mu}.
\]
\end{restatable}

\begin{restatable}{lemmma}{lemGuD}\label{lem:GuD3} 
Under Assumption \ref{assum:smooth}  consider two HMC systems $\theta(t)$ and $\bar{\theta}(t)$ starting with the same initial positions and different initial momentum $p(0)$ and $\widetilde{p}(0)$ respectively. Suppose $\sigma^2$ is the variance of the momentum, we have
\[
\|\theta(t)-\bar{\theta}(t)\|_{2} \leq \frac{1}{\sqrt{L/\sigma^2}}\sinh(\sqrt{L/\sigma^2}t)\|p(0)-\widetilde{p}(0)\|_{2},
\]
Further if we have Assumption  \ref{assum:convex}, consider two HMC systems $\theta(t)$ and $\widetilde{\theta}(t)$ starting with initial positions  $\theta(0)$ and  $\widetilde{\theta}(0)$, and initial momentum $p(0)$ and $\widetilde{p}(0)$ respectively, then
for $0\leq t\leq \frac{1}{2\sqrt{L/\sigma^2}}$, we have
\[
\|\theta(t)-\widetilde{\theta}(t)\|_{2} \leq (1-\frac{\mu}{4\sigma^2}t^2)\|\theta(0)-\widetilde{\theta}(0)\|_{2}+Ct\|p(0)-\widetilde{p}(0)\|_{2},
\]
where $C$ is any constant $\geq\sinh(0.5)$.

\end{restatable}
The proof is postponed to Section \ref{pflem:GuD3}.

\begin{restatable}[One-step update]{lemmma}{lemonestepupdate}\label{lem:onestepupdate}Under the same conditions of Theorem \ref{thm:disFLHMC}. Denote
\begin{align*}
    \text{(I)}=&\sum_{c=1}^Nw_c\Bigl(\theta^{(c)}_{t,0}-\frac{(K\eta_t)^2}{2}\nabla f^{(c)}(\theta^{(c)}_{t,0})-\frac{(K^3-K)\eta_t^3}{6}\nabla^2 f^{(c)}(\theta^{(c)}_{t,0})p^{(c)}_t\Bigr)\\
    &\qquad-\Bigl(\theta^\pi_t-\frac{(K\eta_t)^2}{2}\nabla f(\theta^\pi_t)-\frac{(K\eta_t)^3}{6}\nabla^2 f(\theta^\pi_t)p_t\Bigr);
\end{align*}
we have
\[
 \mathbb{E}\|(\text{I})\|^2\leq (1-\frac{\mu(K\eta_t)^2}{2})\mathbb{E}\|\theta_t-\theta^\pi_t\|^2+2(K\eta_t)^6L^2\big(T\sum_{i=t_0}^{t-1} \sum_{c=1}^Nw_c(1-w_c)\Delta^{(c)}_i+\frac{\sum_{c=1}^Nw_c\sigma_t^{(c)2}d}{9}\big)
\]
where $\Delta^{(c)}_i$ is defined in (\ref{eqlem:dislocb:1}).
\end{restatable}
The proof is postponed to Section \ref{pflem:onestepupdate}.

\begin{restatable}[$0$-th order Approximation Bound]{lemmma}{lemdisthetab}\label{lem:disthetab} Let $q(0),p(0)\in \mathbb{R}^d$ be two vectors that independent of $\{\xi_k,\xi_{k-\frac{1}{2}}\}_k$. For any $F(\cdot):\mathbb{R}^d \mapsto\mathbb{R}$ satisfying smoothness Assumption \ref{assum:smooth} and variance Assumption \ref{assum:boundvar} with parameter $\sigma^2$, under step size assumption $K\eta\leq 1/\sqrt{L}$, the output $q(k):=h(F(\cdot),q(0),p(0),\eta,k)$ yield by Algorithm 3 
with gradients randomized by $\{\xi_k,\xi_{k-\frac{1}{2}}\}_k$ satisfies
\begin{align}
       & \mathbb{E}_{\xi}\| q(k)-q(0)\|^2\leq 2(k\eta)^2\|p(0)\|^2+(k\eta)^4\|\nabla F(q(0))\|^2+(k\eta)^4\sigma^2d \label{eqn:disthetab:a1}\\
      & \mathbb{E}_{\xi}\| \nabla F (q(k))-\nabla F(q(0))\|^2\leq L^2(2(k\eta)^2\|p(0)\|^2+(k\eta)^4\|\nabla F(q(0))\|^2+(k\eta)^4\sigma^2d)\label{eqn:disthetab:a2}
\end{align}
where $\mathbb{E}_{\xi}[\cdot]$ denotes the expectation over random variables $\{\xi_k,\xi_{k-\frac{1}{2}}\}_k$.
\end{restatable}
The proof is postponed to Section \ref{pflem:disthetab}.
\begin{restatable}[$1$-th order Approximation Bound]{lemmma}{lemdisappb}\label{lem:disappb}Under Assumptions \ref{assum:smooth}-\ref{assum:boundvar}, for any iteration $t$ and any local device $c$ in Algorithm 3, 
we have
\begin{align}
   & \mathbb{E}\| \theta^{(c)}_{t,k}-\theta^{(c)}_t-\eta_tp^{(c)}_tk\|^2\leq \frac{1}{3}(k\eta_t)^4L\big(\frac{B^{(c)}_{\nabla}}{L}+\sigma_t^{(c)2}d+ \frac{\sigma_g^2}{Lk}d\big) \label{eqn:disappb:a}\\
   & \mathbb{E}\| \nabla f^{(c)}(\theta^{(c)}_{t,k})-\nabla f^{(c)}(\theta^{(c)}_t)-\nabla^2 f^{(c)}(\theta^{(c)}_t)\eta_tp^{(c)}_tk\|^2\leq (k\eta_t)^4L^3\delta^{(c)}_k \label{eqn:disappb:b}
\end{align}
where $B^{(c)}_{\nabla}:=\sup_t\mathbb{E}\|\nabla f^{(c)}(\theta^{(c)}_{t,0})\|^2$ and for $1\leq k\leq K$
\begin{equation}\label{eqlem:disappb:1}
    \delta^{(c)}_k:=(1+\frac{c_dL_H^2}{8L^3}\sigma_t^{(c)2})\frac{B^{(c)}_{\nabla}}{L}+(\sigma_t^{(c)2}+\frac{c_dL_H^2}{L^3}\sigma_t^{(c)4})d+ \frac{\sigma_g^2d}{kL}(1+\frac{c_dL_H^2}{L^3}\sigma_t^{(c)2}).
\end{equation}
\end{restatable}
The proof is postponed to Section \ref{pflem:disappb}.

\begin{restatable}[Uniform bound]{lemmma}{lemdisunib}\label{lem:disunib}Define $\theta^*=\mathrm{argmin}_{\theta}f(\theta)$. Under Assumptions \ref{assum:convex}-\ref{assum:boundvar} and \ref{assum:boundvar}, for  any local device $1\leq c\leq N$ in Algorithm 3 
and any iteration $t$ such that $t\not\equiv 0 (\mathrm{mod }\ T)$, we have
\begin{align}
 \mathbb{E}\|\theta^{(c)}_{t+1}-\theta^*\|^2
   \leq(1-\frac{\mu(K\eta_t)^2}{4}+\frac{2(K\eta_t)^6L^4}{\mu}\mathbb{I}_{\{K\geq 2\}})\mathbb{E}\|\theta^{(c)}_{t,0}-\theta^*\|^2+\frac{(K\eta_t)^2L}{\mu}B^{(c)}_a, \label{eqn:disunib:a}
\end{align}
where
\begin{align*}
 B^{(c)}_a=&\underbrace{\frac{2L}{\mu}d}_{\text{sampling cost}}+\underbrace{\frac{4\|\nabla f^{(c)}(\theta^*)\|^2}{3\mu}}_{\text{local heterogeneity}}+\underbrace{\frac{\sigma_g^2}{7K\mu}d}_{\text{stochastic gradient effect}}.
\end{align*}
Further if we choose $(K\eta_t)^2\leq \frac{\mu}{4L^2}$ , and for any t such that $t\equiv 0 (\mathrm{mod }\ T)$, we set $\eta_{t}=\eta_{t+1}=\cdots=\eta_{t+T-1}$ 
\begin{align}
&\sup_{t\geq 0} \mathbb{E}\|\theta^{(c)}_{t,0}-\theta^*\|^2\leq D+\frac{8}{\mu}(B^{(c)}_a+\sum_{c=1}^Nw_cB^{(c)}_a) \label{eqn:disunib:a1}\\
&\sup_{t\geq 0}\mathbb{E}\|\nabla f^{(c)}(\theta^{(c)}_{t,0})\|^2\leq 2L^2(D+\frac{8}{\mu}(B^{(c)}_a+\sum_{c=1}^Nw_cB^{(c)}_a))+2\|\nabla f^{(c)}(\theta^*)\|^2, \label{eqn:disunib:a2}
\end{align}
where $D=\|\theta_0-\theta^*\|^2$, representing the initialization effect. 
\end{restatable}

\begin{restatable}[Local bound]{lemmma}{lemdislocb}\label{lem:dislocb}Under Assumptions \ref{assum:smooth}-\ref{assum:boundvar}, for any iteration $k$ and any local device $c$ in Algorithm 3, 
we have 
\[
\sum_{c=1}^Nw_c\mathbb{E}\|\theta^{(c)}_t-\theta_t\|^2\leq  T\sum_{i=t_0}^{t-1} (K\eta_i)^4\sum_{c=1}^Nw_c(1-w_c)L\Delta^{(c)}_i
\]
where $t_0$ is the latest communication step before $t$, and
\begin{equation}\label{eqlem:dislocb:1}
    \Delta^{(c)}_i:=\frac{5}{9}\frac{B^{(c)}_{\nabla}}{L}+\frac{1}{81}\sigma_t^{(c)2}d+\frac{1}{6KL}\sigma_g^2d+\frac{\sum_{i=1}^Nw_c\sigma_t^{(c)2} -1}{L(K\eta_i)^2} d 
\end{equation}
 with $B^{(c)}_{\nabla}=\sup_{t\geq 0}\mathbb{E}\|\nabla f^{(c)}(\theta^{(c)}_{t,0})\|^2$.
\end{restatable}
The proof is postponed to Section \ref{pflem:dislocb}.
\begin{restatable}[Oracle bounds]{lemmma}{lemFLorab}\label{lem:FLorab}If $f$ satisfies Assumptions \ref{assum:convex}-\ref{assum:boundvar}, then for any iteration $t$ and any $u\leq K\eta_t$, we have
\begin{align}
    &\mathbb{E}\|\theta^\pi_t(u)-\theta^\pi_t-p_tu\|^2\leq \frac{L^2d}{4\mu}u^4, \label{eqn:FLorab:a}\\
    &\mathbb{E}\|\theta^\pi_t(u)-\theta^\pi_t\|^2\|p_t\|^2\leq \frac{1}{3}u^2c_dd,\label{eqn:FLorab:b}\\
    &\mathbb{E}\|\nabla f(\theta^\pi_t(u))-\nabla f(\theta^\pi_t)-\nabla^2 f(\theta^\pi_t)p_tu\|^2\leq \delta^\pi u^4, \label{eqn:FLorab:c}
\end{align}
where $c_d=128+32\log^2(2d)$ and 
\begin{equation}\label{eqlem:FLorab:1}
    \delta^\pi:=\frac{5}{4}\big(\frac{L}{\mu}+L_H^2c_d/L^3\big)d.
\end{equation}
\end{restatable}
The proof is postponed to Section \ref{pflem:FLorab}.
\begin{restatable}[Maximal Gaussian bounds]{lemmma}{lemMaxGau}\label{lem:MaxGau} Suppose $p_1,p_2\sim N(0,\sigma^2\mathbb{I}_d)$, then we have
\[
\mathbb{E}\|p_1\|_{\infty}^2\leq c_d\sigma^2,\qquad\mathbb{E}\|p_1\|^2\|p_2\|_{\infty}^2\leq c_d d\sigma^4,\qquad\mathbb{E}\|p_2\|_{\infty}^4\leq c_d \sigma^4,
\]
with $c_d=128+32\log^2(2d)$.
\end{restatable}
The proof is postponed to Section \ref{pflem:MaxGau}.

\section{Proof of Theorem~\ref{thm:convergrate:sg}}\label{sec:pfpropconver}

\begin{customthm}{4.1}\label{thm:convergrate:sg}
Assume \ref{assum:convex}-\ref{assum:boundvar}, and ${\cal W}_2(\pi_0,\pi)^2=O\footnote{\label{fn:bigO} As $d\rightarrow \infty$, we say $f=O(g)$ if $f\leq C g$ for some constant $C$, and say $f=\widetilde{O}(g)$ for $C$ being a  polynomial of $\log(d)$. } (d)$ and $\sum_{c=1}^Nw_c\|\nabla f^{(c)}(\theta^*)\|^2=O(d)$. For a given local iteration step $T$, there exists some constant $C$ depending on $L,L/\mu,L^2_H/L^3$ such that if we choose $\eta(t)\equiv \eta$  and (denote $\gamma=(K\eta)^2$)
\[
\eta^2 = \frac{\gamma}{K^2} =C \min\Bigl\{\frac{1}{K^2L},\frac{\epsilon}{K^2\sqrt{d}T},\frac{\epsilon^2}{K^2dT^2(1-\rho)N},\frac{\epsilon^2}{Kd\sum_{c=1}^Nw_c^2\sigma^2_g}\Bigr\}
\]
then  ${\cal W}_2(\pi_{t_\epsilon},\pi)\leq \epsilon$ for any $\epsilon>0$, with iteration number
\[
t_\epsilon=\frac{d\log(d/\epsilon^2)}{\epsilon^2}\widetilde{O}^{\ref{fn:bigO}}\Bigl(T^2\big(\gamma+(1-\rho)N\big)+\frac{\sum_{c=1}^Nw_c^2\sigma^2_g}{K}\Bigr)
\]
and corresponding communication times
\[
\frac{t_{\epsilon}}{T}=\frac{d\log(d/\epsilon^2)}{\epsilon^2}\widetilde{O}\Bigl(T\big(\gamma+(1-\rho)N\big)+\frac{\sum_{c=1}^Nw_c^2\sigma^2_g}{KT}\Bigr).
\]
\end{customthm}

\begin{proof}
By the assumptions and Theorem \ref{thm:disFLHMC} and Lemma \ref{lem:disunib}, we have
\[
\Delta_t\leq  C_1T^2\Bigl(\frac{1}{L}+\frac{(1-\rho)N}{L\gamma_t}\Bigr)d+C_2\sum_{c=1}^N\frac{w_c^2\sigma_g^2d}{K\gamma_t}
\]
define $D:=\mathbb{E}\|\theta_0-\theta^{\pi}\|^2$, then to have ${\cal W}(\theta_t,\theta^\pi)^2\leq\mathbb{E}\|\theta_t-\theta_t^\pi\|^2\leq \epsilon^2$, it is feasible to let $(\eta,t)$ to satisfy for some constant $C$
\begin{equation}\label{pfprop:convergrate:1}
    C(K\eta)^2\Delta_t\leq \frac{\epsilon^2}{2},\qquad C(1-\frac{\mu(K\eta)^2}{4})^tD=\frac{\epsilon^2}{2}.
\end{equation}
By the fact that $\frac{-1}{\log(1-x)}\geq \frac{1}{x/(1-x)}$ for any $1>x>0$, it follows that
\[
t_{\epsilon}=\frac{\log(\epsilon/(2D))}{\log(1-\mu(K\eta_{\epsilon})^2/4)}\geq \frac{(1-\mu(K\eta_{\epsilon})^2/4)\log(\epsilon/(2D))}{\mu(K\eta_{\epsilon})^2/4} 
\]
when $\mu(K\eta_{\epsilon})^2/4<1$.
Further,  by the stepsize assumption (i.e., $K\eta\leq 1/\sqrt{L}$ up to a constant) and $L\geq \mu$, we have $1-\mu(K\eta_{\epsilon})^2/4\geq 15/16$; and by (\ref{pfprop:convergrate:1}) we have $(K\eta)^2\leq C \min\Bigl\{\frac{\epsilon}{\sqrt{d}T},\frac{\epsilon^2}{dT^2(1-\rho)N},\frac{\epsilon^2K}{d\sum_{c=1}^Nw_c^2\sigma^2_g}\Bigr\}$. Eventually, we get
\[
t_{\epsilon}\geq C\frac{\sqrt{d}\log(\epsilon/(2D))(\kappa^{2.5}T+\kappa^3(1+(\frac{c_dL_H^2}{L^3})^{0.5}))}{\epsilon \sqrt{L}}.
\]
We can pick one $t_\epsilon$ that satisfies the above inequality. This finishes the proof.
\end{proof}

\section{Sketch Proof of Convergence of Vanilla FA-HMC}\label{sec:sketchproof}
In this section, we present a sketch proof for Theorem \ref{thm:disFLHMC} without considering SG and correlation between momentum. A formal proof is postponed in Section \ref{sec:pfmainthm}.

Let $\theta^\pi_0$ be the a random variable following the target distribution $\pi$, and $\theta^{\pi}_{t}(\cdot)$ be the solution of the Hamilton's dynamic (1) 
initialized from $\theta^{\pi}_t$ and momentum $p_t$. Note that this $p_t$ is the same momentum used in the vanilla FA-HMC (i.e., Algorithm 3 
with $\rho=1$ and $\sigma_g=0$). Define $\theta^{\pi}_{t+1}=\theta^{\pi}_{t}(K\eta)$.
Define $\theta_{t,j}^{(c)}$ be the intermediate value that is yielded after $j$ steps of leapfrog approximation within the $t$-th iteration of Algorithm 3. 

The main idea to bound the difference between $\sum w_c\theta_t^{(c)}$ and $\theta^{\pi}_t$.

Recursively applying the definitions of leapfrog approximation and Hamilton's dynamic respectively, it is easy to derive that 
\begin{align*}
 \theta_{t+1}^{(c)}=&\theta^{(c)}_{t,0}+K\eta p_t-\frac{K\eta^2}{2}\nabla f^{(c)}(\theta^{(c)}_{t,0})-\eta^2\sum_{j=1}^{K-1}(K-j)\nabla f^{(c)}(\theta^{(c)}_{t,j}),\\    
 \theta^\pi_{t+1}=&\theta^\pi_t+K\eta p_t-\int_0^{K\eta}\int_0^s\nabla f(\theta^\pi_t(u))duds.
\end{align*}
Note that intermediate positions $ \theta_{t,k}^{(c)}$, $k=1,2,\ldots K$ and $\theta_t^{\pi}(u)$ are correlated to momentum $p_t$.
For simplicity of the representation, we consider the first-order approximation
{$\theta_{t,j}^{(c)}=\theta^{(c)}_{t,0}+j\eta  p_t +O((j\eta)^2)$ and $\theta^\pi_t(u)=\theta^\pi_t+u  p_t +O(u^2)$, and apply Taylor's expansion to the  gradients $\nabla f^{(c)}(\theta^{(c)}_{t,j})$ and $\nabla f(\theta^\pi_t(u))$ respectively. This yield the following approximation:}

\begin{align}
    \theta_{t+1}^{(c)}\approx &\;\theta^{(c)}_{t,0}+K\eta  p_t-\frac{(K\eta)^2}{2} \nabla f^{(c)}(\theta^{(c)}_{t,0})-\frac{(K^3-K)\eta^3}{6}\nabla^2 f^{(c)}(\theta^{(c)}_{t,0})p_t,\label{eq:appro1}\\
    \theta^\pi_{t+1}\approx &\;\theta^\pi_t+K\eta  p_t-\frac{(K\eta)^2}{2}\nabla f(\theta^\pi_t)-\frac{(K\eta)^3}{6}\nabla^2 f(\theta^\pi_t)p_t. \label{eq:appro2}
\end{align}
where the RHS of the equalities has a more tractable form in terms of $p_t$.

It follows that
\begin{align*}
    \sum_{c=1}^Nw_c\theta_{t+1}^{(c)}-\theta^\pi_{t+1}=\underbrace{\text{(I)}}_{\text{second-order approximation}}-\underbrace{\eta^2\text{(II)}}_{\text{discrete approx. error}}-\underbrace{\text{(III)}}_{\text{continuous approx. error}},
\end{align*}
where
\begin{align*}
    \text{(I)}=&\sum_{c=1}^Nw_c\Bigl(\theta^{(c)}_{t,0}-\frac{(K\eta)^2}{2}\nabla f^{(c)}(\theta^{(c)}_{t,0})-\frac{(K^3-K)\eta^3}{6}\nabla^2 f^{(c)}(\theta^{(c)}_{t,0})p_t\Bigr)\\
    &\qquad-\Bigl(\theta^\pi_t-\frac{(K\eta)^2}{2}\nabla f(\theta^\pi_t)-\frac{(K\eta)^3}{6}\nabla^2 f(\theta^\pi_t)p_t\Bigr);\\
    \text{(II)}=&\sum_{c=1}^Nw_c\sum_{k=1}^{K-1}(K-k)\Bigl(\nabla f^{(c)}(\theta_k)-\nabla f^{(c)}(\theta^{(c)}_{t,0})-\nabla^2 f^{(c)}(\theta^{(c)}_{t,0})\eta p_tk\Bigr);\\
    \text{(III)}=&\int_0^{K\eta}\int_0^s\big(\nabla f(\theta^\pi_t(u))-\nabla f(\theta^\pi_t)-\nabla^2 f(\theta^\pi_t)p_tu\big)duds.
\end{align*}
To interpret the above decomposition, we comment that $\text{(I)}$ is the second-order approximation of $ \sum_{c=1}^Nw_c\theta_{t+1}^{(c)}-\theta^\pi_{t+1}$;  $\text{(II)}$ is the error yielded by the second-order approximation on the gradients; $\text{(III)}$ can be viewed as the difference between the continuous process and the discrete process of the global HMC system.

Now, by Lemma \ref{lem:convex}, for any $c_1,c_2,c_3>0$ with $\sum_ic_i=1$, we can separate the errors 
\begin{align}
\begin{split}\label{pfthm:vandisFLHMC:1}
     \|\sum_{c=1}^Nw_c\theta_{t+1}^{(c)}-\theta^\pi_{t+1}\|^2\leq& \frac{1}{c_1}\|\text{(I)}\|^2+\frac{1}{c_2}\eta^4\|\text{(II)}\|^2+\frac{1}{c_3}\|\text{(III)}\|^2.
\end{split}
\end{align}
Term $\text{(I)}$ is the difference of the second-order approximation (i.e., \eqref{eq:appro1}) of the discrete FA-HMC system and that of the continuous process (i.e., \eqref{eq:appro2}). Intuitively, both of them converge in roughly the same direction. Therefore, this term is expected to have a contraction relationship with respect to the initial deviation, i.e., $\sum_{c=1}^Nw_c\theta_t^{(c)}-\theta^\pi_t$, and contains additional errors due to the stochastic momentum and slightly different form of recursive functions \eqref{eq:appro1} and \eqref{eq:appro2}. Indeed, by Lemma \ref{lem:onestepupdate}  in the appendix,
\begin{align}
    \begin{split}\label{pfthm:vandisFLHMC:4}
        \mathbb{E} \|\text{(I)}\|^2
         \leq (1-\frac{\mu(K\eta)^2}{2})\mathbb{E}\|\sum_{c=1}^Nw_c\theta_t^{(c)}-\theta^\pi_t\|^2+2T^2(K\eta)^6L^2\Delta+\frac{(K\eta)^6L^2d}{9},
    \end{split}
\end{align}
where $t_0$ is the latest communication step before $t$-th iteration and  $\Delta:=\sum_{c=1}^Nw_c(1-w_c)\big (\frac{5}{9} B^{(c)}_{\nabla}/L+\frac{d}{81}\big)$ with $B^{(c)}_{\nabla}:=\sup_t\mathbb{E}\|\nabla f^{(c)}(\theta^{(c)}_{t,0})\|^2$.


Terms $\text{(II)}$ and $\text{(III)}$ study the approximation errors of \eqref{eq:appro1} and \eqref{eq:appro2} respectively. The errors, intuitively are small given the small step size $\eta$.

For  $\text{(II)}$, due to Lemma \ref{lem:convex} and the fact that $\sum_{c=1}^Nw_c\sum_{k=1}^{K-1}\frac{(K-k)k}{K(K^2-1)/6} =1$, we have
\begin{align}
 \begin{split} \label{pfthm:vandisFLHMC:6}
       \mathbb{E}\|\text{(II)}\|^2
  \leq&  \sum_{c=1}^Nw_c\sum_{k=1}^{K-1}\frac{(K-k)K^3}{6k}\mathbb{E}\| \nabla f^{(c)}(\theta_k)-\nabla f^{(c)}(\theta^{(c)}_t)-\nabla^2 f^{(c)}(\theta^{(c)}_t)\eta p_tk\|^2\\
  \leq &\sum_{c=1}^Nw_c\sum_{k=1}^{K-1}\frac{(K-k)K^3}{6k}(k\eta)^4\delta^{(c)}\leq \frac{K^4(K\eta)^4}{120}\sum_{c=1}^Nw_c\delta^{(c)}
 \end{split}
\end{align}
where the second inequality is due to Lemma \ref{lem:disappb} with $\delta^{(c)}=(1+\frac{c_dL_H^2}{L^3})(d+B^{(c)}_{\nabla}/L)$.


Last, for  $\text{(III)}$, note that $\int_0^{K\eta}\int_0^s\frac{12u^2}{(K\eta)^4}duds=1$, by Lemma \ref{lem:convex},
\begin{align*}
\begin{split}\label{pfthm:vandisFLHMC:7}
       \mathbb{E}\|\text{(III)}\|^2\leq \int_0^{K\eta}\int_0^s\frac{(K\eta)^4}{12u^2}\mathbb{E}\|\nabla f(\theta^\pi_t(u))-\nabla f(\theta^\pi_t)-\nabla^2 f(\theta^\pi_t)p_tu\|^2duds\leq \frac{(K\eta)^8L^3\delta^\pi}{144}.
\end{split}
\end{align*}
where $\delta^\pi:=\frac{5}{4}\big(\frac{L}{\mu}+\frac{c_dL_H^2}{L^3}\sinh^2(\frac{1}{4})\big)d$ and the last inequality is by \eqref{eqn:FLorab:c} of Lemma \ref{lem:FLorab}.

Combining this with (\ref{pfthm:vandisFLHMC:1})-(\ref{pfthm:vandisFLHMC:6}) and set
\[
(c_1,c_2,c_3)=(1-\frac{\mu(K\eta)^2/4}{1-\mu(K\eta)^2/4},\frac{6}{11}\frac{\mu(K\eta)^2/4}{1-\mu(K\eta)^2/4},\frac{5}{11}\frac{\mu(K\eta)^2/4}{1-\mu(K\eta)^2/4}),
\]
we have
\[
\frac{1}{c_1}(1-\frac{\mu(K\eta)^2}{2})=1-\frac{\mu(K\eta)^2}{4},\qquad \frac{1}{c_1}\leq\frac{15}{14},\qquad \frac{1}{c_2}\leq \frac{22}{3\mu(K\eta)^2},\qquad\frac{1}{c_3}\leq \frac{44}{5\mu(K\eta)^2},
\]
and
\begin{align}
    \mathbb{E} \|\sum_{c=1}^Nw_c\theta_{t+1}^{(c)}-\theta^\pi_{t+1}\|^2
     \leq&(1-\frac{\mu(K\eta)^2}{4})\mathbb{E}\|\sum_{c=1}^Nw_c\theta_t^{(c)}-\theta^\pi_t\|^2+3T(K\eta)^6L^2\sum_{i=t_0}^{t-1}\Delta_i\nonumber\\
     &+\frac{(K\eta)^6L^2d}{8}+\frac{(K\eta)^6L^3}{15\mu}\frac{(K\eta)^6}{\mu}\Bigl(\sum_{c=1}^Nw_c\delta^{(c)}+\delta^\pi\Bigr),\label{eq:contraction}
\end{align}
where we recall that $\Delta_i=\sum_{c=1}^Nw_c(1-w_c)\big (\frac{5}{9} B^{(c)}_{\nabla}/L+\frac{d}{81}\big)$,  $\delta^{(c)}=(1+\frac{c_dL_H^2}{L^3})(d+B^{(c)}_{\nabla}/L)$ and  $\delta^\pi=\frac{5}{4}\big(\frac{L}{\mu}+\frac{c_dL_H^2}{L^3}\sinh^2(\frac{1}{4})\big)d$. Rearranging the coefficients of $B^{(c)}_{\nabla}/L$ and $d$, and recursively applying \eqref{eq:contraction} concludes the proof.

\section{Proof of Theorem~\ref{thm:disFLHMC}}\label{sec:pfmainthm}
\subsection{Proof of  Theorem \ref{thm:disFLHMC}}

\begin{customthm}{4.4}[Convergence]\label{thm:disFLHMC}
Under Assumptions \ref{assum:convex}-\ref{assum:boundvar}, if we  set $\eta_{t''}\leq \eta_{t'}\leq 1/(K\sqrt{L})$ for any $t' \leq t''$ in Algorithm 3, 
then  $\{\theta_t\}_t$  satisfies 
\[
    \mathbb{E} \|\theta_{t+1}-\theta^\pi_{t+1}\|^2
     \leq(1-\frac{\mu(K\eta_t)^2}{4})^t\mathbb{E}\|\theta_0-\theta^\pi_0\|^2+\eta_t^2\Delta_t
\]
where there exist constants $C_1,C_2>0$ depending on $L,L/\mu,L^2_H/L^3$  and $c_d=\log^2(d)$, such that
\begin{align*}
 \Delta_t=& C_1T^2 K^2\sum_{c=1}^N\Bigl(\underbrace{\frac{w_cB^{(c)}_{\nabla}}{L}}_{\mathrm{B
 ias}}(K\eta_t)^2+\underbrace{\frac{1-\rho}{L}d}_{\mathrm{Correlation}}\Bigr)+C_2K\underbrace{\sum_{c=1}^Nw_c^2\sigma_g^2d}_{\mathrm{Stoc.\; Grad.}}
\end{align*}
with $B^{(c)}_{\nabla}:=\sup_t\mathbb{E}\|\nabla f^{(c)}(\theta^{(c)}_{t,0})\|^2$.
\end{customthm}

\begin{proof}
To start with, we rewrite the expression of $ \theta_{t+1}^{(c)}$ and $ \theta^\pi_{t+1}$. By iterative formula (\ref{eq:iteraform}) and (\ref{eq:contin}), 
\begin{align*}
 \theta_{t+1}^{(c)}=&\theta^{(c)}_{t,0}+K\eta_t p^{(c)}_t-\frac{K\eta_t^2}{2}\widetilde{g}^{(c)}_{t,0}(\xi_0)-\frac{\eta_t^2}{2}\sum_{j=1}^{K-1}(K-j)\Bigl(\widetilde{g}^{(c)}_{t,j}(\xi_{j-\frac{1}{2}})+\widetilde{g}^{(c)}_{t,j}(\xi_j)\Bigr),\\    
 \theta^\pi_{t+1}=&\theta^\pi_t+K\eta_tp_t-\int_0^{K\eta_t}\int_0^s\nabla f(\theta^\pi_t(u))duds,
\end{align*}
where we recall the stochastic gradients are defined by
\[
 \widetilde{g}^{(c)}_{t,k}(\xi_{k-\frac{1}{2}})=\nabla \widetilde{f}^{(c)}(\theta^{(c)}_{t,k},\xi^{(c)}_{t,k-\frac{1}{2}}),\qquad  \widetilde{g}^{(c)}_{t,k}(\xi_k)=\nabla \widetilde{f}^{(c)}(\theta^{(c)}_{t,k},\xi^{(c)}_{t,k}),\qquad k=1,2,\ldots, K.
\]
Note that $ \theta_{t+1}^{(c)}$ contains stochastic gradients and intermediate positions $ \theta_{t,k}^{(c)}$, $k=1,2,\ldots K$ that correlate to momentum $p_t$. These cause analytical troubles. To overcome these, we use the second-order approximation with non-stochastic gradients to approximate $\theta_{t+1}^{(c)}$ as follow. 
\begin{align*}
           \theta_{t+1}^{(c)}=& \theta^{(c)}_{t,0}+K\eta_t  p^{(c)}_t-\frac{K\eta_t^2}{2} (\widetilde{g}^{(c)}_{t,0}(\xi_0)-\nabla f^{(c)}(\theta^{(c)}_{t,0})+\nabla f^{(c)}(\theta^{(c)}_{t,0}))\\
     &-\frac{\eta_t^2}{2}\sum_{j=1}^{K-1}(K-j) \Bigl(\big(\widetilde{g}^{(c)}_{t,j}(\xi_{j-\frac{1}{2}})+\widetilde{g}^{(c)}_{t,j}(\xi_j)-2\nabla f^{(c)}(\theta^{(c)}_{t,j})\big)+2\big(\nabla f^{(c)}(\theta^{(c)}_{t,j})\\
     &-\nabla f^{(c)}(\theta^{(c)}_{t,0})-j\eta_t\nabla^2 f^{(c)}(\theta^{(c)}_{t,0})p^{(c)}_t+\nabla f^{(c)}(\theta^{(c)}_{t,0})+j\eta_t\nabla^2 f^{(c)}(\theta^{(c)}_{t,0})p^{(c)}_t\big)\Bigr)\\
     =& \theta^{(c)}_{t,0}+K\eta_t  p^{(c)}_t-\frac{K\eta_t^2}{2} \nabla f^{(c)}(\theta^{(c)}_{t,0})-\frac{\eta_t^2}{2}\sum_{j=1}^{K-1}(K-j)2\big(\nabla f^{(c)}(\theta^{(c)}_{t,0})+j\eta_t\nabla^2 f^{(c)}(\theta^{(c)}_{t,0})p^{(c)}_t\big)\\
     &-\frac{K\eta_t^2}{2}(\widetilde{g}^{(c)}_{t,0}(\xi_0)-\nabla f^{(c)}(\theta^{(c)}_{t,0}))-\frac{\eta_t^2}{2}\sum_{j=1}^{K-1}(K-j) \Bigl(\widetilde{g}^{(c)}_{t,j}(\xi_{j-\frac{1}{2}})+\widetilde{g}^{(c)}_{t,j}(\xi_j)-2\nabla f^{(c)}(\theta^{(c)}_{t,j})\Bigr)\\
     &-\frac{\eta_t^2}{2}\sum_{j=1}^{K-1}(K-j)2\big(\nabla f^{(c)}(\theta^{(c)}_{t,j})-\nabla f^{(c)}(\theta^{(c)}_{t,0})-j\eta_t\nabla^2 f^{(c)}(\theta^{(c)}_{t,0})p^{(c)}_t\big)\\
     =& \text{(I)}^{(c)}-\frac{\eta_t^2}{2}\text{(II)}^{(c)}-\eta_t^2\text{(III)}^{(c)},
\end{align*}
where by $\sum_{j=1}^{K-1}(K-j)=\frac{K(K-1)}{2}$ and $\sum_{j=1}^{K-1}(K-j)j=\frac{K^3-K}{6}$,
\begin{align*}
    \text{(I)}^{(c)}=&\theta^{(c)}_{t,0}+ K\eta_tp^{(c)}_t-\frac{(K\eta_t)^2}{2}\nabla f^{(c)}(\theta^{(c)}_{t,0})-\frac{(K^3-K)\eta_t^3}{6}\nabla^2 f^{(c)}(\theta^{(c)}_{t,0})p^{(c)}_t\\
    \text{(II)}^{(c)}=&K(\widetilde{g}^{(c)}_{t,0}(\xi_0)-\nabla f^{(c)}(\theta^{(c)}_{t,0}))-\sum_{j=1}^{K-1}\frac{K-j}{2} \Bigl(\widetilde{g}^{(c)}_{t,j}(\xi_{j-\frac{1}{2}})+\widetilde{g}^{(c)}_{t,j}(\xi_j)-2\nabla f^{(c)}(\theta^{(c)}_{t,j})\Bigr)\\
    \text{(III)}^{(c)}=&    \sum_{j=1}^{K-1}(K-j)\big(\nabla f^{(c)}(\theta^{(c)}_{t,j})-\nabla f^{(c)}(\theta^{(c)}_{t,0})-j\eta_t\nabla^2 f^{(c)}(\theta^{(c)}_{t,0})p^{(c)}_t\big)
\end{align*}
Here we can interpret $\text{(I)}^{(c)}$ as the second-order approximation of $\theta^{(c)}_{t+1}$, $\text{(II)}^{(c)}$ as the random error induced by stochastic gradient and $\text{(III)}^{(c)}$ as the numerical error brought by second-order approximation.

Combining this with FL setting $\sum_{c=1}^Nw_c p^{(c)}_t=p_t$ , we can rewrite
\begin{align*}
    \sum_{c=1}^Nw_c\theta_{t+1}^{(c)}-\theta^\pi_{t+1}=\text{(I)}-\frac{\eta_t^2}{2}\text{(II)}-\eta_t^2\text{(III)}+\text{\text{(IV)}},
\end{align*}
where
\begin{align*}
    \text{(I)}=&\sum_{c=1}^Nw_c\Bigl(\theta^{(c)}_{t,0}-\frac{(K\eta_t)^2}{2}\nabla f^{(c)}(\theta^{(c)}_{t,0})-\frac{(K^3-K)\eta_t^3}{6}\nabla^2 f^{(c)}(\theta^{(c)}_{t,0})p^{(c)}_t\Bigr)\\
    &\qquad-\Bigl(\theta^\pi_t-\frac{(K\eta_t)^2}{2}\nabla f(\theta^\pi_t)-\frac{(K\eta_t)^3}{6}\nabla^2 f(\theta^\pi_t)p_t\Bigr);\\
    \text{(II)}=&\sum_{c=1}^Nw_c\Bigl(K\big(\widetilde{g}^{(c)}_{t,0}(\xi_0)-\nabla f^{(c)}(\theta^{(c)}_{t,0})\big)+\sum_{k=1}^{K-1}(K-k)\big(\widetilde{g}^{(c)}_{t,k}(\xi_{k-\frac{1}{2}})+\widetilde{g}^{(c)}_{t,k}(\xi_k)-2\nabla f^{(c)}(\theta^{(c)}_{t,k})\big)\Bigr);\\
    \text{(III)}=&\sum_{c=1}^Nw_c\sum_{k=1}^{K-1}(K-k)\Bigl(\nabla f^{(c)}(\theta^{(c)}_{t,k})-\nabla f^{(c)}(\theta^{(c)}_{t,0})-\nabla^2 f^{(c)}(\theta^{(c)}_{t,0})\eta_tp^{(c)}_tk\Bigr);\\
    \text{(IV)}=&\int_0^{K\eta_t}\int_0^s\nabla f(\theta^\pi_t(u))duds-\frac{(K\eta_t)^2}{2}\nabla f(\theta^\pi_t)-\frac{(K\eta_t)^3}{6}\nabla^2 f(\theta^\pi_t)p_t.
\end{align*}
To interpret the above decomposition, we comment that $\text{(I)}$ is the non-random second-order approximation of $ \sum_{c=1}^Nw_c\theta_{t+1}^{(c)}-\theta^\pi_{t+1}$; $\text{(II)}$ is the random noise induced by the stochastic gradients; $\text{(III)}$ is the error yielded by the second-order approximation on the gradients; $\text{(IV)}$ can be viewed as the difference between the continuous process and the discrete process of the global HMC system.

By Lemma \ref{lem:convex}, for any $c_1,c_2,c_3>0$ with $\sum_ic_i=1$,
\begin{equation}\label{pfthm:disFLHMC:1}
     \|\sum_{c=1}^Nw_c\theta_{t+1}^{(c)}-\theta^\pi_{t+1}\|^2\leq \frac{1}{c_1}\|\text{(I)}-\frac{\eta_t^2}{2}\text{(II)}\|^2+\frac{1}{c_2}\eta_t^4\|\text{(III)}\|^2+\frac{1}{c_3}\|\text{(IV)}\|^2.
\end{equation}
In what follows, to show the goal, we will bound $\text{(I)}-\frac{\eta_t^2}{2}\text{(II)}$, $\text{(III)}$ and $\text{(IV)}$ separately. Specifically, we claim that $\text{(I)}-\frac{\eta_t^2}{2}\text{(II)}$ has a contraction relationship with $\sum_{c=1}^Nw_c\theta_t^{(c)}-\theta^\pi_t$ plus some errors, and
\[
\mathbb{E}\|\text{(III)}\|^2=O(K^4(K\eta_t)^C),\qquad \mathbb{E}\|\text{(IV)}\|^2=O((K\eta_t)^C),\qquad \text{for some constant } C>0.
\]
Let's consider the term $\text{(I)}-\frac{\eta_t^2}{2}\text{(II)}$ first. Denote $\mathbb{E}_{\xi_t}$ be the expectation over $\{(\xi^{(c)}_{t,k-\frac{1}{2}},\xi^{(c)}_{t,k}):c\in [N], k=0,1,\ldots,K-1\}$. Note that $\text{(II)}$ is a mean-zero martingale w.r.t., $\{(\xi^{(c)}_{t,k-\frac{1}{2}},\xi^{(c)}_{t,k}):c\in [N],k=0,1,\ldots,K-1\}$ and $\mathbb{E}_{\xi_t}\text{(I)}=\text{(I)}$. By definitions and taking conditional expectation on $(\xi^{(c)}_{t,k-\frac{1}{2}},\xi^{(c)}_{t,k})$ following order $k=K-1,K-2,\ldots,0$, sequentially, we have
\begin{align}
\begin{split}\label{pfthm:disFLHMC:2}
  \mathbb{E}_{\xi_t}\|\text{(I)}-\frac{\eta_t^2}{2}\text{(II)}\|^2
  =&\|\text{(I)}\|^2+\frac{\eta_t^4}{4}\sum_{c=1}^Nw_c^2\Bigl(K^2\text{Var}_{\xi_t}(\widetilde{g}^{(c)}_{t,0}(\xi_0))\\
  &\quad+\sum_{k=1}^{K-1}(K-k)^2\big(\text{Var}_{\xi_t}(\widetilde{g}^{(c)}_{t,k}(\xi_{k-\frac{1}{2}}))+\text{Var}_{\xi_t}(\widetilde{g}^{(c)}_{t,k}(\xi_k))\big)\Bigr)\\
  \leq &\|\text{(I)}\|^2+\frac{\eta_t^4}{4}\sum_{c=1}^Nw_c^2 \Bigl(\frac{K(2K^2+1)}{3}\sigma_g^2d\Bigr)\\
   \leq& \|\text{(I)}\|^2+\frac{(K\eta_t)^4}{4K}\sum_{c=1}^Nw_c^2\sigma_g^2d,
\end{split}
\end{align}
where  the first inequality is by variance Assumption \ref{assum:boundvar} and $\sum_{k=1}^K(K-k)^2=(2K^3-3K^2+K)/6$.

At the same time, by Lemma \ref{lem:onestepupdate}, we have
\[
 \mathbb{E}\|(\text{I})\|^2\leq (1-\frac{\mu(K\eta_t)^2}{2})\mathbb{E}\|\theta_t-\theta^\pi_t\|^2+2(K\eta_t)^6L^2\big(T\sum_{i=t_0}^{t-1} \sum_{c=1}^Nw_c(1-w_c)\Delta^{(c)}_i+\frac{\sum_{c=1}^Nw_c\sigma_t^{(c)2}d}{18}\big)
\]
where $\Delta^{(c)}_i$ is defined in (\ref{eqlem:dislocb:1}).

Next we consider $\text{(III)}$. Due to Lemma \ref{lem:convex} and the fact that $\sum_{c=1}^Nw_c\sum_{k=1}^{K-1}\frac{(K-k)k}{(K^3-K)/6} =1$, thus
\begin{align*}
 \|\text{(III)}\|^2
  \leq&  \sum_{c=1}^Nw_c\sum_{k=1}^{K-1}\frac{6(K-k)k}{K^3-K}\Bigl\| \frac{K^3-K}{6k}\big(\nabla f^{(c)}(\theta^{(c)}_{t,k})-\nabla f^{(c)}(\theta^{(c)}_{t,0})-\nabla^2 f^{(c)}(\theta^{(c)}_{t,0})\eta_tp^{(c)}_tk\big)\Bigr\|^2 \\
  \leq&  \sum_{c=1}^Nw_c\sum_{k=1}^{K-1}\frac{(K-k)K^3}{6k}\| \nabla f^{(c)}(\theta^{(c)}_{t,k})-\nabla f^{(c)}(\theta^{(c)}_{t,0})-\nabla^2 f^{(c)}(\theta^{(c)}_{t,0})\eta_tp^{(c)}_tk\|^2.
\end{align*}
Combining the above inequality with  Lemma \ref{lem:disappb} with $\delta_k^{(c)}$ defined in (\ref{eqlem:disappb:1}), we have 
\begin{equation}\label{pfthm:disFLHMC:6}
    \mathbb{E}\|\text{(III)}\|^2\leq\sum_{c=1}^Nw_c\sum_{k=1}^{K-1}\frac{(K-k)K^3}{6k}(k\eta_t)^4L^3\delta^{(c)}_k\leq \frac{K^4(K\eta_t)^4}{72}L^3\sum_{c=1}^Nw_c\delta^{(c)}_K,
\end{equation}
where the last inequality is due to $\sum_{k=1}^{K-1}(K-k)k^2\leq K^4/12 $ and $k\delta^{(c)}_k\leq K\delta^{(c)}_K$.

Last, consider  $\text{(IV)}$. Note that $\frac{(K\eta_t)^2}{2}\nabla f(\theta^\pi_t)+\frac{(K\eta_t)^3}{6}\nabla^2 f(\theta^\pi_t)p_t=\int_0^{K\eta_t}\int_0^s\big(\nabla f(\theta^\pi_t)+\nabla^2 f(\theta^\pi_t)p_tu\big)duds$, we have
\[
\text{(IV)}=\int_0^{K\eta_t}\int_0^s\Bigl(\nabla f(\theta^\pi_t(u))-\nabla f(\theta^\pi_t)-\nabla^2 f(\theta^\pi_t)p_tu\Bigr)duds.
\]
By direct calculations, 
\begin{align}
\begin{split}\label{pfthm:disFLHMC:7}
       \mathbb{E}\|\text{(IV)}\|^2\leq&\int_0^{K\eta_t}\int_0^s\frac{12u^2}{(K\eta_t)^4}\mathbb{E}\|\frac{(K\eta_t)^4}{12u^2}(\nabla f(\theta^\pi_t(u))-\nabla f(\theta^\pi_t)-\nabla^2 f(\theta^\pi_t)p_tu)\|^2duds\\
   \leq & \frac{(K\eta_t)^4}{12}L^3\delta^\pi\int_0^{K\eta_t}\int_0^su^2duds\\
  \leq&\frac{(K\eta_t)^8}{144}L^3\delta^\pi.
\end{split}
\end{align}
where the first inequality is by Lemma \ref{lem:convex} and note that $\int_0^{K\eta_t}\int_0^s\frac{12u^2}{(K\eta_t)^4}duds=1$; the second inequality is by result \eqref{eqn:FLorab:c} of Lemma \ref{lem:FLorab} with $\delta^\pi$ defined in (\ref{eqlem:FLorab:1}).

Now combining (\ref{pfthm:disFLHMC:1})-(\ref{pfthm:disFLHMC:7}) we get
\begin{align*}
    \mathbb{E} \|\theta_{t+1}&-\theta^\pi_{t+1}\|^2
    \leq\frac{1}{c_1}\Bigl((1-\frac{\mu(K\eta_t)^2}{2})\mathbb{E}\|\theta_t-\theta^\pi_t\|^2+2(K\eta_t)^6L^2T\sum_{i=t_0}^{t-1}\sum_{c=1}^Nw_c(1-w_c)\Delta^{(c)}_i\\
    &+ \frac{(K\eta_t)^6L^2\sum_{c=1}^Nw_c\sigma^{(c)2}d}{9}+\frac{(K\eta_t)^4}{4K}\sum_{c=1}^Nw_c^2\sigma_g^2d\Bigr)+\frac{(K\eta_t)^8L^3}{72c_2}\sum_{c=1}^Nw_c\delta^{(c)}_K+\frac{(K\eta_t)^8L^3}{144c_3}\delta^\pi.
\end{align*}
Setting
\[
(c_1,c_2,c_3)=\Bigl(1-\frac{\mu(K\eta_t)^2/4}{1-\mu(K\eta_t)^2/4},\frac{6}{11}\frac{\mu(K\eta_t)^2/4}{1-\mu(K\eta_t)^2/4},\frac{5}{11}\frac{\mu(K\eta_t)^2/4}{1-\mu(K\eta_t)^2/4}\Bigr),
\]
we have
\[
\frac{1}{c_1}(1-\frac{\mu(K\eta_t)^2}{2})=1-\frac{\mu(K\eta_t)^2}{4},\qquad\frac{1}{c_1}\leq\frac{15}{14},\qquad \frac{1}{c_2}\leq \frac{22}{3\mu(K\eta_t)^2},\qquad\frac{1}{c_3}\leq \frac{44}{5\mu(K\eta_t)^2},
\]
and
\begin{align*}
    \mathbb{E} \|\theta_{t+1}&-\theta^\pi_{t+1}\|^2
     \leq(1-\frac{\mu(K\eta_t)^2}{4})\mathbb{E}\|\theta_t-\theta^\pi_t\|^2+3(K\eta_t)^6L^2T\sum_{i=t_0}^{t-1}\sum_{c=1}^Nw_c(1-w_c)\Delta^{(c)}_i\\
    &+ \frac{(K\eta_t)^6L^2\sum_{c=1}^Nw_c\sigma^{(c)2}d}{8}+\frac{(K\eta_t)^4}{3K}\sum_{c=1}^Nw_c^2\sigma_g^2d+\frac{(K\eta_t)^6L^3}{9\mu}(\sum_{c=1}^Nw_c\delta_K^{(c)}+\delta^\pi).
\end{align*}
Inserting  the definitions of $\Delta^{(c)}_i$ defined in (\ref{eqlem:dislocb:1}), $\delta_k^{(c)}$ defined in (\ref{eqlem:disappb:1}), and $\delta^\pi$ defined in (\ref{eqlem:FLorab:1}), and rearranging the coefficients yield the claim of this theorem. 
\end{proof}

\section{Proof of Lemmas~\ref{lem:GuD3}-\ref{lem:MaxGau}}\label{sec:pflems}

\subsection{Proof of Lemma \ref{lem:GuD3}}\label{pflem:GuD3}
\lemGuD*

\begin{proof}
This lemma is a generalization of Lemma D.3 in \cite{chen2019optimal}.
Consider the first claim. By definition \eqref{eq:PDE} and Lipschitz assumption \ref{assum:smooth},
\[
\bigg|\frac{d\|p(t)-\bar{p}(t)\|_2}{dt}\bigg|=\bigg|\frac{(p(t)-\bar{p}(t))^T}{\|p(t)-\bar{p}(t)\|_2}\frac{d(p(t)-\bar{p}(t))}{dt}\bigg|
\leq \frac{L}{\sigma}\|\theta(t)-\bar{\theta}(t)\|_2.
\]
This implies that
\[
\|p(t)-\bar{p}(t)\|_2\leq \|p(0)-\bar{p}(0)\|_2+\int_{0}^t \frac{L}{\sigma}\|\theta(u)-\bar{\theta}(u)\|_2du.
\]
Combining the above inequality with (\ref{eq:PDE}), it is easy to obtain that
\[
\|\theta(t)-\bar{\theta}(t)\|_2\leq \|\theta(0)-\bar{\theta}(0)\|_2+\frac{t}{\sigma}\|p(0)-\bar{p}(0)\|_2+\int_0^t\int_{0}^s \frac{L}{\sigma^2}\|\theta(u)-\bar{\theta}(u)\|_2duds.
\]
Therefore, 
\[
\|\theta(t)-\bar{\theta}(t)\|_2\leq x(t),
\]
where $x(t)$ satisfies $x(0)=\|\theta(0)-\bar{\theta}(0)\|_2$, $x'(0)=\|p(0)-\bar{p}(0)\|_2/\sigma$ and $x''(t)= \frac{L}{\sigma^2}x(t) $

By elementary ODE, the solution of the $x$ process is $x(t)=\|\theta(0)-\bar{\theta}(0)\|_2\cosh(\sqrt{L/\sigma^2}t)+\frac{\|p(0)-\bar{p}(0)\|_2}{\sqrt{L/\sigma^2}}\sinh (\sqrt{L/\sigma^2}t)$ and thus
\[
\|\theta(t)-\bar{\theta}(t)\|_2\leq\|\theta(0)-\bar{\theta}(0)\|_2\cosh(\sqrt{L/\sigma^2}t)+\frac{\|p(0)-\bar{p}(0)\|_2}{\sqrt{L/\sigma^2}}\sinh (\sqrt{L/\sigma^2}t),
\]
where the first term in the RHS is $0$ since $\bar{\theta}(0)=\theta(0)$. This finishes the proof of the first claim

Consider the second claim. Note that by Lemma 6 in \cite{chen2019optimal}, we immediately obtain $\|\bar{\theta}(t)-\widetilde{\theta}(t)\|_2\leq (1-\frac{\mu t^2}{4\sigma^2})\|\bar{\theta}(0)-\widetilde{\theta}_0\|_2=(1-\frac{\mu t^2}{4 \sigma^2})\|\theta(0)-\widetilde{\theta}_0\|_2$. Therefore it suffices to show that
\[
   \|\theta(t)-\bar{\theta}(t)\|_2\leq Ct\|p(0)-\widetilde{p}(0)\|_2. 
\]
which is implies by the first claim and elementary algebra that $\sinh (\sqrt{L/\sigma^2}t)\leq C\sqrt{L/\sigma^2}t$ for all $C>\sinh(0.5)$ and $0\leq \sqrt{Lt/\sigma^2}\leq 0.5$. This concludes the proof.
\end{proof}
\subsection{Proof of Lemma~\ref{lem:onestepupdate}}\label{pflem:onestepupdate}
\lemonestepupdate*
\begin{proof}
Noting that $\{p^{(c)}_t\}_{c=1}^N$ are mean-zero Gaussian and independent of $\{\theta^{(c)}_{t,0}\}_{c=1}^N$, 
\begin{equation}\label{pfthm:disFLHMC:3}
    \mathbb{E}\|\text{(I)}\|^2=\mathbb{E}\|(\text{I}_1)\|^2+0+\mathbb{E}\|(\text{I}_2)\|^2,
\end{equation}
where we denote
\begin{align*}
    (\text{I}_1)=&\sum_{c=1}^Nw_c\Bigl(\theta^{(c)}_{t,0}-\frac{(K\eta_t)^2}{2}\nabla f^{(c)}(\theta^{(c)}_{t,0})\Bigr)-\theta^\pi_t+\frac{(K\eta_t)^2}{2}\nabla f(\theta^\pi_t),\\
    (\text{I}_2)=&\frac{(K^3-K)\eta_t^3}{6}\sum_{c=1}^Nw_c\Bigl(\nabla^2 f^{(c)}(\theta^{(c)}_{t,0})-\nabla^2 f^{(c)}(\theta^\pi_t)\Bigr)p^{(c)}_t.
\end{align*}
Note that for $(\text{I}_1)$, by Lemma B.1 in  \cite{deng2021convergence},
\[
     \|(\text{I}_1)\|^2\leq (1-\frac{\mu(K\eta_t)^2}{2})\|\sum_{c=1}^Nw_c\theta^{(c)}_{t,0}-\theta^\pi_t\|^2+2(K\eta_t)^2L\sum_{c=1}^Nw_c\|\theta^{(c)}_{t,0}-\theta_t\|^2.
\]
Further, combining this with Lemma \ref{lem:dislocb} and $\theta_t=\sum_{c=1}^Nw_c\theta^{(c)}_{t,0}$, it follows that
\begin{equation}\label{pfthm:disFLHMC:4}
    \mathbb{E}\|(\text{I}_1)\|^2\leq (1-\frac{\mu(K\eta_t)^2}{2})\mathbb{E}\|\theta_t-\theta^\pi_t\|^2+2(K\eta_t)^2L^2T\sum_{i=t_0}^{t-1} (K\eta_i)^4\sum_{c=1}^Nw_c(1-w_c)\Delta^{(c)}_i.
\end{equation}
As for $(\text{I}_2)$, by direct calculations 
\begin{align}
\begin{split}\label{pfthm:disFLHMC:5}
        \mathbb{E}\|(\text{I}_2)\|^2
        \leq&\sum_{c=1}^N\frac{w_c}{2}\Bigl(\mathbb{E}\|2\frac{(K^3-K)\eta_t^3}{6}\nabla^2 f^{(c)}(\theta^{(c)}_t)p^{(c)}_t\|^2+\mathbb{E}\|2\frac{(K\eta_t)^3}{6}\nabla^2 f^{(c)}(\theta^\pi_t)p^{(c)}_t\|^2\Bigr)\\
    \leq&\frac{(K\eta_t)^6}{18}\sum_{c=1}^Nw_c\Bigl(\mathbb{E}\|\nabla^2 f^{(c)}(\theta^{(c)}_t)\|_F^2+\mathbb{E}\|\nabla^2 f^{(c)}(\theta^\pi_t)\|_F^2\Bigr)\sigma_t^{(c)2}\\
    \leq& \frac{(K\eta_t)^6L^2d}{9}\sigma_t^{(c)2},
\end{split}
\end{align}
where the first inequality is due to Lemma \ref{lem:convex}; the second inequality is by numerical bound and that $\mathbb{E}\|Ap\|^2=\mathbb{E}\|A\|_F^2$ for any matrix $A$ independent of Gaussian vector $p$; the last inequality is due to the facts that $\|\nabla^2 f(\cdot)\|_F^2\leq d\|\nabla^2 f(\cdot)\|^2$ and  $\|\nabla^2 f(\cdot)\|^2\leq L^2$ (ensured by smoothness Assumption \ref{assum:smooth}).
\end{proof}

\subsection{Proof of Lemma \ref{lem:disthetab}}\label{pflem:disthetab}
\lemdisthetab*
\begin{proof}
For simplicity of notation, whenever it is clear from the context, for any $k\in \mathbb{R}$, we misuse $q_k:=q(k)$, $p_k:=p(k)$, $\xi_k:=\xi(k)$ and
\begin{align*}
    \widetilde{g}_k(\xi_{k-\frac{1}{2}})=\nabla \widetilde{F}(q_k,\xi_{k-\frac{1}{2}}),\qquad  \widetilde{g}_k(\xi_k)=\nabla \widetilde{F}(q_k,\xi_k).
\end{align*}
Recall that by definitions (\ref{eq:iteraform}), we have
\[
    q_k=q_0+k\eta p_0-\frac{k\eta^2}{2}\nabla \widetilde{F}(q_0,\xi_0)-\frac{\eta^2}{2}\sum_{j=1}^{k-1}(k-j)\Bigl(\nabla \widetilde{F}(q_j,\xi_{j-\frac{1}{2}})+\nabla \widetilde{F}(q_j,\xi_j)\Bigr).
\]
We prove \eqref{eqn:disthetab:a1} by induction. 
For $k=1$, by simple algebra, we have that
\begin{align*}
    &\mathbb{E}_{\xi}\|q_1-q_0\|^2=\mathbb{E}_{\xi}\|\eta p_0-\frac{\eta^2}{2}\nabla \widetilde{F}(q_0,\xi_0)\|^2
    =\|\eta p_0-\frac{\eta^2}{2}\nabla F(q_0)\|^2+0+(\frac{\eta^2}{2})^22\sigma^2d\\
    \leq & 2\|\eta p_0\|^2+2\|\frac{\eta^2}{2}\nabla F(q_0)\|^2+(\frac{\eta^2}{2})^22\sigma^2d,
\end{align*}
where the first two quality follow the expectation calculations conditional on $\xi_0$ first and variance Assumption \ref{assum:boundvar}; the inequality is due to Lemma \ref{lem:convex}.

Next, we assume that the claim \eqref{eqn:disthetab:a1} holds for $j=1,2,\ldots,k-1$. By iterative formula (\ref{eq:iteraform}), we can write
\begin{equation}\label{pflem:disthetab:1}
 q_k-q_0= k\eta p_0-\frac{(k\eta)^2}{2}\nabla F(q_0)-\frac{\eta^2}{2} \widetilde{\text{(I)}}+ \eta^2\sum_{j=1}^{k-1}(k-j)(\nabla F(q_j)-\nabla F(q_0))
\end{equation}
where
\[
\widetilde{\text{(I)}} =k(\nabla \widetilde{F}(q_0,\xi_0)
    -\nabla F(q_0))+\sum_{j=1}^{k-1}(k-j)(\nabla \widetilde{F}(q_j,\xi_{j-\frac{1}{2}})+\nabla \widetilde{F}(q_j,\xi_j)-2\nabla F(q_j))
\]
By direct calculations, for any $c_1,c_2,c_3>0$ such that $\sum_ic_i=1$  we have
\begin{align*}
    \mathbb{E}_{\xi}&\|q_k-q_0\|^2\\
        \leq &   \frac{1}{c_1}\mathbb{E}_{\xi}\|k\eta p_0-\frac{\eta^2}{2} \widetilde{\text{(I)}}\|^2+\frac{(k\eta)^4}{4c_2}\|\nabla F(q_0)\|^2+\sum_{j=1}^{k-1}\frac{(k-j)\eta^4k^3}{6j c_3}\|\nabla F(q_j)-\nabla F(q_0)\|^2\\
    \leq &   \frac{1}{c_1}\Bigl((k\eta)^2\| p_0\|^2+\frac{k^3\eta^4}{4}\sigma^2d\Bigr)+\frac{(k\eta)^4}{4c_2}\|\nabla F(q_0)\|^2+\sum_{j=1}^{k-1}\frac{(k-j)\eta^4k^3}{6j c_3}\|\nabla F(q_j)-\nabla F(q_0)\|^2\\
    \leq &  \frac{1}{c_1}\Bigl((k\eta)^2\| p_0\|^2+\frac{k^3\eta^4}{4}\sigma^2d\Bigr)+\frac{(k\eta)^4}{4c_2}\|\nabla F(q_0)\|^2+\sum_{j=1}^{k-1}\frac{(k-j)\eta^4k^3L^2}{6j c_3}\|q_j-q_0\|^2\\
    \leq &   \frac{1}{c_1}\Bigl((k\eta)^2\| p_0\|^2+\frac{k^3\eta^4}{4}\sigma^2d\Bigr)+\frac{(k\eta)^4}{4c_2}\|\nabla F(q_0)\|^2+\sum_{j=1}^{k-1}\frac{(k-j)\eta^4k^3L^2}{6j c_3}\\
    &\cdot (2(j\eta)^2\|p_0\|^2+(j\eta)^4\|\nabla F(q_0)\|^2+j^3\eta^4\sigma^2d)\\
    \leq &   (\frac{1}{c_1}+\frac{L^2}{6 c_3}\frac{2(k\eta)^4}{6})(k\eta)^2\| p_0\|^2+(\frac{1}{4c_2}+\frac{L^2}{6 c_3}\frac{(k\eta)^4}{20})(k\eta)^4\|\nabla F(q_0)\|^2+(\frac{1}{4c_1}\\
    &+\frac{L^2}{6 c_3}\frac{(k\eta)^4}{12})\frac{(k\eta)^4}{k}\sigma^2d
\end{align*}
where the first inequality is due to Lemma \ref{lem:convex} and noting that  $c_1+c_2+\frac{\sum_{j=1}^{k-1}(k-j)j}{(k^3-k)/6}c_3=1$; the second inequality is by similar argument of (\ref{pfthm:disFLHMC:2}); the third inequality is due to smoothness Assumption \ref{assum:smooth} and numeric bound; the fourth inequality holds by the induction assumption; the fifth inequality is by $\sum_{j=1}^{k-1}(k-j)j=\frac{k^3-k}{6}\leq \frac{k^3}{6}$, $\sum_{j=1}^{k-1}(k-j)j^2\leq \frac{k^4}{12}$ and  $\sum_{j=1}^{k-1}(k-j)j^3=\frac{k^5}{20}-\frac{k^3}{12}+\frac{1}{20}k\leq \frac{k^5}{20}$. The completion of the induction follows by step size assumption $K\eta\leq 1/\sqrt{L}$ and setting $(c_1,c_2,c_3)=(\frac{3}{5},\frac{1}{3},\frac{1}{15},\frac{1}{24})$ in the last inequality. Therefore, we conclude the proof of \eqref{eqn:disthetab:a1}.

The claim \eqref{eqn:disthetab:a2} follows directly from \eqref{eqn:disthetab:a1} and smoothness Assumption \ref{assum:smooth}.
\end{proof}

\subsection{Proof of Lemma \ref{lem:disappb}}\label{pflem:disappb}
\lemdisappb*
\begin{proof}
We first prove \eqref{eqn:disappb:a}. By the decomposition in (\ref{pflem:disthetab:1}), we can write
\[
 \theta^{(c)}_{t,k}-\theta^{(c)}_{t,0}-\eta_tp^{(c)}_tk=-\frac{( k\eta_t)^2}{2}\nabla f^{(c)}(\theta^{(c)}_{t,0}) -\frac{\eta_t^2}{2}\widetilde{\text{(I)}}^{(c)}-\eta_t^2\sum_{j=1}^{ k-1}( k-j)\big(\nabla f^{(c)}(\theta^{(c)}_{t,j})-\nabla f^{(c)}(\theta^{(c)}_{t,0})\big),
\]
where $\widetilde{\text{(I)}}^{(c)}$ is the random noise induced by stochastic gradient
\[
\widetilde{\text{(I)}}^{(c)}=k(\widetilde{g}^{(c)}_{t,0}(\xi_0)-\nabla f^{(c)}(\theta^{(c)}_{t,0}))-\frac{1}{2}\sum_{j=1}^{ k-1}( k-j) \Bigl(\widetilde{g}^{(c)}_{t,j}(\xi_{j-\frac{1}{2}})+\widetilde{g}^{(c)}_{t,j}(\xi_j)-2\nabla f^{(c)}(\theta^{(c)}_{t,j})\Bigr).
\]
Then by direct calculations, for any $c_1,c_2>0$ such that $c_1+c_2=1$, we have
\begin{align*}
\mathbb{E}\| &\theta^{(c)}_{t,k}-\theta^{(c)}_{t,0}-\eta_tp^{(c)}_tk\|^2\\
\leq&\frac{1}{c_1}\mathbb{E}\|\frac{ (k\eta_t)^2}{2}\nabla f^{(c)}(\theta^{(c)}_{t,0})- \frac{\eta_t^2}{2}\widetilde{\text{(I)}}^{(c)}\|^2+\sum_{j=1}^{k-1}\frac{(k-j)(k^3-k)}{6c_2j}\mathbb{E}\|\eta_t^2\big(\nabla f^{(c)}(\theta^{(c)}_{t,j})-\nabla f^{(c)}(\theta^{(c)}_{t,0})\big)\|^2\\
\leq&\frac{(k\eta_t)^4}{4c_1}(\mathbb{E}\|\nabla f^{(c)}(\theta^{(c)}_{t,0})\|^2+\frac{1}{k}\sigma_g^2d )+\sum_{j=1}^{k-1}\frac{(k-j)k^3\eta_t^4}{6c_2j}\mathbb{E}\|\nabla f^{(c)}(\theta^{(c)}_{t,j})-\nabla f^{(c)}(\theta^{(c)}_{t,0})\|^2
\end{align*}
where the first inequality is by Lemma \ref{lem:convex} and noting that $c_1+\sum_{j=1}^{k-1}\frac{(k-j)j}{(k^3-k)/6}c_2=1$; the second inequality is by similar argument as that in (\ref{pfthm:disFLHMC:2}) and $k^3-k\leq k^3$.

Further combining it with result \eqref{eqn:disthetab:a2} in Lemma \ref{lem:disthetab}, we can upper bound the above inequality as follows:
\begin{align*}
 \mathbb{E}\| \theta^{(c)}_{t,k}-\theta^{(c)}_{t,0}-k\eta_tp^{(c)}_t\|^2
 \leq&\frac{(k\eta_t)^4}{4}\big(\frac{\mathbb{E}\|\nabla f^{(c)}(\theta^{(c)}_{t,0})\|^2}{c_1}+\frac{\sigma_g^2d}{c_1k}+\sum_{j=1}^{k-1}\frac{2(k-j)}{3c_2kj}L^2(2(j\eta_t)^2\mathbb{E}\|p^{(c)}_t\|^2\\
 &+(j\eta_t)^4\mathbb{E}\|\nabla f^{(c)}(\theta^{(c)}_{t,0})\|^2+j^3\eta_t^4\sigma_g^2d)\big)\\
 \leq &\frac{(k\eta_t)^4}{4}\Bigl(\frac{1}{c_1}\mathbb{E}\|\nabla f^{(c)}(\theta^{(c)}_{t,0})\|^2+\frac{\sigma_g^2d}{c_1k}+\frac{2L^2}{3c_2}\big(\frac{2(k\eta_t)^2\sigma_t^{(c)2} d}{6}\\
 &+\frac{(k\eta_t)^4}{20}\mathbb{E}\|\nabla f^{(c)}(\theta^{(c)}_{t,0})\|^2+\frac{(k\eta_t)^4}{12k}\sigma_g^2d\big)\Bigr)\\
  \leq&\frac{1}{3}(k\eta_t)^4\big(\mathbb{E}\|\nabla f^{(c)}(\theta^{(c)}_{t,0})\|^2+L\sigma_t^{(c)2}d+ \frac{\sigma_g^2d}{k}\big)
\end{align*}
where the second inequality is due to the facts  $\sum_{j=1}^{k-1}(k-j)j=\frac{k^3-k}{6}\leq \frac{k^3}{6}$, $\sum_{j=1}^{k-1}(k-j)j^2\leq \frac{k^4}{12}$ and  $\sum_{j=1}^{k-1}(k-j)j^3=\frac{k^5}{20}-\frac{k^3}{12}+\frac{1}{30}k\leq \frac{k^5}{20}$; the last inequality is due to the step size assumption $K\eta\leq 1/\sqrt{L}$ and the choice $(c_1,c_2)=(\frac{18}{19},\frac{1}{19})$.

Next consider \eqref{eqn:disappb:b}. By Lemma \ref{lem:convex},
\begin{equation}\label{pflem:disappb:1}
 \| \nabla f^{(c)}(\theta^{(c)}_{t,k})-\nabla f^{(c)}(\theta^{(c)}_{t,0})-\nabla^2 f^{(c)}(\theta^{(c)}_{t,0})k\eta_tp^{(c)}_t\|^2=\|(b_1)+(b_2)\|^2\leq 3\|(b_1)\|^2+\frac{3}{2}\|(b_2)\|^2,   
\end{equation}
where
\begin{align*}
  (b_1)=&\int_0^1\nabla^2 f^{(c)}\Bigl(\theta^{(c)}_{t,0}+s(\theta^{(c)}_{t,k}-\theta^{(c)}_{t,0})\Bigr)(\theta^{(c)}_{t,k}-\theta^{(c)}_{t,0}-k\eta_tp^{(c)}_t)ds,\\
  (b_2)=& \int_0^1\Bigl(\nabla^2 f^{(c)}\Bigl(\theta^{(c)}_{t,0}+s(\theta^{(c)}_{t,k}-\theta^{(c)}_{t,0})\Bigr)-\nabla^2 f^{(c)}(\theta^{(c)}_{t,0})\Bigr)k\eta_tp^{(c)}_tds.  
\end{align*}
Consider $(b_1)$, by simple algebra,
\begin{align*}
    \|(b_1)\|^2\leq & \int_0^1 \Bigl\|\nabla^2 f^{(c)}\big(\theta^{(c)}_{t,0}+s(\theta^{(c)}_{t,k}-\theta^{(c)}_{t,0})\big)(\theta^{(c)}_{t,k}-\theta^{(c)}_{t,0}-k\eta_tp^{(c)}_t)\Bigr\|^2ds\\
    \leq &\int_0^1L^2\|\theta^{(c)}_{t,k}-\theta^{(c)}_{t,0}-k\eta_tp^{(c)}_t\|^2ds\\
    =& L^2\|\theta^{(c)}_{t,k}-\theta^{(c)}_{t,0}-k\eta_tp^{(c)}_t\|^2,
\end{align*}
where the first inequality is due to Jensen's inequality, and the second is implied by Assumption \ref{assum:smooth}.

Combining this with \eqref{eqn:disappb:a} lead to
\begin{align}
\begin{split}\label{pflem:disappb:2_v2}
    \mathbb{E}\|(b_1)\|^2\leq &\frac{1}{3}L^2(k\eta_t)^4\Bigl(\mathbb{E}\|\nabla f^{(c)}(\theta^{(c)}_{t,0})\|^2+L\sigma_t^{(c)2}d+ \frac{\sigma_g^2d}{k}\Bigr).
\end{split}
\end{align}
Similarly for $(b_2)$, we have that
\begin{align}
\begin{split}\label{pflem:disappb:3}
    \mathbb{E}\|(b_2)\|^2\leq&\int_0^1\mathbb{E}\Bigl\|\big(\nabla^2 f^{(c)}(\theta^{(c)}_{t,0}+s(\theta^{(c)}_{t,k}-\theta^{(c)}_{t,0}))-\nabla^2 f^{(c)}(\theta^{(c)}_{t,0})\big)k\eta_tp^{(c)}_t\Bigr\|^2ds\\
\leq &\int_0^1 L_H^2\mathbb{E}\|s(\theta^{(c)}_{t,k}-\theta^{(c)}_{t,0})\|^2\|k\eta_tp^{(c)}_t\|_{\infty}^2ds\\
\leq &\frac{L_H^2}{3} \mathbb{E}\big[\big(2(k\eta_t)^2\|p^{(c)}_t\|^2+(k\eta_t)^4\|\nabla f^{(c)}(\theta^{(c)}_{t,0})\|^2+k^3\eta_t^4\sigma_g^2d\big)\|k\eta_tp^{(c)}_t\|_{\infty}^2\big]\\
\leq &\frac{L_H^2}{3}\mathbb{E}\Bigl[2(k\eta_t)^4c_d\sigma^{(c)4}d+(k\eta_t)^6\mathbb{E}\|\nabla f^{(c)}(\theta^{(c)}_{t,0})\|^2c_d\sigma_t^{(c)2}+k^5\eta_t^6\sigma_g^2c_d\sigma_t^{(c)2}d\Bigr]\\
\leq &\frac{2c_d}{3}L_H^2(k\eta_t)^4\Bigl(\sigma^{(c)4}d+\frac{1}{8L}\mathbb{E}\|\nabla f^{(c)}(\theta^{(c)}_{t,0})\|^2\sigma_t^{(c)2}+\frac{1}{8Lk}\sigma_g^2\sigma_t^{(c)2}d\Bigr)
\end{split}
\end{align}
where the first inequality is due to Jensen's inequality; the second inequality is induced by Hessian Assumption \ref{assum:Hsmooth}; the third inequality is due to Lemma \ref{lem:disthetab}; the  fourth equality is by Lemma \ref{lem:MaxGau}; the last inequality is due to the step size assumption $K\eta\leq 1/\sqrt{L}$ and $d\geq 1$.

Combining (\ref{pflem:disappb:1})-(\ref{pflem:disappb:3}), we get \eqref{eqn:disappb:b}. This finishes the proof.
\end{proof}

\subsection{Proof of Lemma \ref{lem:disunib}}\label{sec:pflem:disunib}
\lemdisunib*
\begin{proof}
We show \eqref{eqn:disunib:a}, \eqref{eqn:disunib:a1} and \eqref{eqn:disunib:a2} sequentially. Recall that, by definitions (\ref{eq:iteraform}),
\[
\theta_{t+1}^{(c)}=\theta^{(c)}_{t,0}+K\eta_t p^{(c)}_t-\frac{K\eta_t^2}{2}\widetilde{g}^{(c)}_{t,0}(\xi_0)-\frac{\eta_t^2}{2}\sum_{j=1}^{K-1}(K-j)\Bigl(\widetilde{g}^{(c)}_{t,j}(\xi_{j-\frac{1}{2}})+\widetilde{g}^{(c)}_{t,j}(\xi_j)\Bigr),
\]
hence, we can decompose
\[
\theta^{(c)}_{t+1}-\theta^*=\text{(I)}-\frac{\eta_t^2}{2}\widetilde{\text{(I)}}+\text{(II)},
\]
where
\begin{align*}
    \text{(I)}=&\theta^{(c)}_{t,0}+K\eta_tp^{(c)}_t-\frac{(K\eta_t)^2}{2}\nabla f^{(c)}(\theta^{(c)}_{t,0})-\theta^*+\frac{(K\eta_t)^2}{2}\nabla f^{(c)}(\theta^*),\\
    \widetilde{\text{(I)}}=&K\big(\widetilde{g}^{(c)}_{t,0}(\xi_0)-\nabla f^{(c)}(\theta^{(c)}_{t,0})\big)+\sum_{k=1}^{K-1}(K-k)\big(\widetilde{g}^{(c)}_{t,k}(\xi_{k-\frac{1}{2}})+\widetilde{g}^{(c)}_{t,k}(\xi_k)-2\nabla f^{(c)}(\theta^{(c)}_{t,k})\big),\\
    \text{(II)}=&\frac{(K\eta_t)^2}{2}\nabla f^{(c)}(\theta^*)-\eta_t^2\sum_{k=1}^{K-1}(K-k)\Bigl(\nabla f^{(c)}(\theta^{(c)}_{t,k})-\nabla f^{(c)}(\theta^{(c)}_{t,0})\Bigr).
\end{align*}
First, we prove inequality \eqref{eqn:disunib:a}. By Lemma \ref{lem:convex}, for any $c_1,c_2>0$ with $c_1+c_2=1$,
\begin{equation}\label{pflem:disunib:1}
  \mathbb{E}\|\theta^{(c)}_{t+1}-\theta^*\|^2\leq \frac{1}{c_1}\mathbb{E}\|\text{(I)}-\frac{\eta_t^2}{2}\widetilde{\text{(I)}}\|^2+\frac{1}{c_2}\mathbb{E}\|\text{(II)}\|^2.   
\end{equation}
We first derive the bound on $\mathbb{E}\|\text{(I)}-\frac{\eta_t^2}{2}\widetilde{\text{(I)}}\|^2$.

Similar to the argument as that in (\ref{pfthm:disFLHMC:2}), we have
\[
\mathbb{E}_{\xi}\|\text{(I)}-\frac{\eta_t^2}{2}\widetilde{\text{(I)}}\|^2    \leq \|\text{(I)}\|^2+\frac{K^3\eta_t^4}{4} \sigma_g^2d.
\]
For $\text{(I)}$, noting that $p^{(c)}_t$ is mean-zero Gaussian and independent of $\theta^{(c)}_{t,0}$, we have
\begin{align}
\begin{split}\label{pflem:disunib:2}
    \mathbb{E}\|\text{(I)}\|^2=&\mathbb{E}\|\theta^{(c)}_{t,0}-\frac{(K\eta_t)^2}{2}\nabla f^{(c)}(\theta^{(c)}_{t,0})-\theta^*+\frac{(K\eta_t)^2}{2}\nabla f^{(c)}(\theta^*)\|^2+2*0+(K\eta_t)^2\sigma_t^{(c)2}d\\
 \leq&(1-\frac{\mu(K\eta_t)^2}{2})\mathbb{E}\|\theta^{(c)}_{t,0}-\theta^*\|^2+(K\eta_t)^2\sigma_t^{(c)2}d,
\end{split}
\end{align}
where the last inequality is due to Lemma B.1 in  \citet{deng2021convergence} with only one local device in the Federated Learning system. 

On the other hand, given some $c_3,c_4,c_5,c_6>0$ with $c_3+c_4=c_5+c_6=1$, we have
\begin{align}
\begin{split}\label{pflem:disunib:3}
       \mathbb{E}\|\text{(II)}\|^2\leq&\frac{1}{c_3}\|\frac{(K\eta_t)^2}{2}\nabla f^{(c)}(\theta^*)\|^2+\frac{1}{c_4}\mathbb{E}\|\eta_t^2\sum_{k=1}^{K-1}(K-k)\big(\nabla f^{(c)}(\theta^{(c)}_{t,k})-\nabla f^{(c)}(\theta^{(c)}_{t,0})\big)\|^2\\
    \leq & \frac{(K\eta_t)^4}{4c_3}\|\nabla f^{(c)}(\theta^*)\|^2+\frac{\eta_t^4}{c_4}\sum_{k=1}^{K-1}\frac{6(K-k)k}{K^3-K}\mathbb{E}\Bigl\|\frac{K^3-K}{6k}\Bigl(\nabla f^{(c)}(\theta^{(c)}_{t,k})-\nabla f^{(c)}(\theta^{(c)}_{t,0})\Bigr)\Bigr\|^2\\
    \leq & \frac{(K\eta_t)^4}{4c_3}\|\nabla f^{(c)}(\theta^*)\|^2+\frac{\eta_t^4K^3}{6c_4}\sum_{k=1}^{K-1}\frac{K-k}{k}L^2\Bigl(2(k\eta_t)^2\mathbb{E}\|p^{(c)}_t\|^2\\
    &+(k\eta_t)^4\mathbb{E}\|\nabla f^{(c)}(\theta^{(c)}_{t,0})\|^2+k^3\eta_t^4\sigma_g^2d\Bigr)\\
    \leq & \frac{(K\eta_t)^4}{4c_3}\|\nabla f^{(c)}(\theta^*)\|^2+\frac{\eta_t^4K^3}{6c_4}\sum_{k=1}^{K-1}\frac{K-k}{k}L^2\Bigl(2(k\eta_t)^2\mathbb{E}\|p^{(c)}_t\|^2\\
    &+(k\eta_t)^4(\frac{1}{c_5}L^2\mathbb{E}\|\theta^{(c)}_{t,0}-\theta^*\|^2+\frac{1}{c_6}\|\nabla f^{(c)}(\theta^*)\|^2)+k^3\eta_t^4\sigma_g^2d\Bigr)\\
           \leq & (\frac{1}{4c_3}+\frac{(K\eta_t)^4L^2}{6*20c_4c_6})(K\eta_t)^4\|\nabla f^{(c)}(\theta^*)\|^2+\frac{2(K\eta_t)^2L^2}{6*6c_4}(K\eta_t)^4\mathbb{E}\|p^{(c)}_t\|^2\\
    &+\frac{(K\eta_t)^8L^4}{6*20c_4c_5}\mathbb{E}\|\theta^{(c)}_{t,0}-\theta^*\|^2+\frac{(K\eta_t)^8L^2}{6*12c_4K}\sigma_g^2d\Bigr)\\
    \leq & \frac{(K\eta_t)^4}{3}\|\nabla f^{(c)}(\theta^*)\|^2+\frac{(K\eta_t)^4L}{3}\sigma_t^{(c)2}d+\frac{(K\eta_t)^8L^4}{2}\|\theta^{(c)}_{t,0}-\theta^*\|^2+\frac{(K\eta_t)^4}{48K}\sigma_g^2d,
\end{split}
\end{align}
where the first two inequalities are due to Lemma \ref{lem:convex} and the fact that $\sum_{k=1}^{K-1}(K-k)k=K(K^2-1)/6$; the third inequality is by Lemma \ref{lem:disthetab}; the fourth inequality implied by that $\|\nabla f^{(c)}(\theta^{(c)}_{t,0})\|^2\leq \frac{1}{c_5}\|\nabla f^{(c)}(\theta^{(c)}_{t,0})-\nabla f^{(c)}(\theta^*)\|^2+\frac{1}{c_6}\|\nabla f^{(c)}(\theta^*)\|^2$ (Lemma \ref{lem:convex}) and smoothness Assumption \ref{assum:smooth}; the fifth inequality follows by the step size assumption  $K\eta\leq 1/\sqrt{L}$, the fact that $\sum_{k=1}^{K-1}(K-k)k^2\leq K^4/12 $ and  $\sum_{k=1}^{K-1}(K-k)k^3\leq K^5/20 $, and the choices that $(c_3,c_4)=(\frac{18}{19},\frac{1}{19})$, $(c_5,c_6)=(\frac{19}{60},\frac{41}{60})$.

Combining (\ref{pflem:disunib:1})-(\ref{pflem:disunib:3}) with $(c_1,c_2)=(1-\frac{\mu(K\eta_t)^2}{4}/(1-\frac{\mu(K\eta_t)^2}{4}),\frac{\mu(K\eta_t)^2}{4}/(1-\frac{\mu(K\eta_t)^2}{4}))$ and step size assumption $K\eta\leq 1/\sqrt{L}$, we have
\[
\frac{1}{c_1}(1-\frac{\mu(K\eta_t)^2}{2})=(1-\frac{\mu(K\eta_t)^2}{4}),\qquad \frac{1}{c_1}\leq\frac{15}{14},\qquad \frac{1}{c_2}\leq \frac{4}{\mu(K\eta_t)^2}
\]
and
\begin{align*}
    \mathbb{E}\|\theta^{(c)}_{t+1}-\theta^*\|^2\leq&(1-\frac{\mu(K\eta_t)^2}{4}+\frac{2(K\eta_t)^6L^4}{\mu})\mathbb{E}\|\theta^{(c)}_{t,0}-\theta^*\|^2+(K\eta_t)^2(\frac{15}{14}\sigma_t^{(c)2}d+\frac{15}{14}\frac{K\eta_t^2}{4} \sigma_g^2d\\
    &+\frac{4\|\nabla f^{(c)}(\theta^*)\|^2}{3\mu}+\frac{2L}{3\mu}\sigma_t^{(c)2}d+\frac{2}{3}\frac{\sigma_g^2}{K\mu}d).
\end{align*}
This holds for any iteration $t$ such that $t\not\equiv 0 (\mathrm{mod }\ T)$ and implies \eqref{eqn:disunib:a}.

Next we  bound $\sup_{t\geq 0}\mathbb{E}\|\theta^{(c)}_{t,0}-\theta^*\|^2$ (i.e., \eqref{eqn:disunib:a1}),
by the condition $(K\eta_t)^2\leq \frac{\mu}{4L^2}$ we have
\[
\frac{\mu(K\eta_t)^2}{4}-\frac{2(K\eta_t)^6L^4}{\mu}\geq \frac{\mu(K\eta_t)^2}{8}.
\]
It follows that
\begin{equation}\label{eqn:temp}
 \mathbb{E}\|\theta^{(c)}_{t+1}-\theta^*\|^2
   \leq (1-\frac{\mu(K\eta_t)^2}{8})\mathbb{E}\|\theta^{(c)}_{t,0}-\theta^*\|^2+(K\eta_t)^2B^{(c)}_a.
\end{equation}
with
\[
B^{(c)}_a=\frac{4\|\nabla f^{(c)}(\theta^*)\|^2}{3\mu}+\frac{2L}{\mu}\sigma_t^{(c)2}d+\frac{\sigma_g^2}{K\mu}d.
\]
Note that for any $\frac{t}{T}\notin \mathbb{Z}$, $\theta^{(c)}_{t,0}=\theta^{(c)}_t$. Then apply inductions on \eqref{eqn:temp}, it follows that for $s=0,1,\ldots$
\begin{align*}
\mathbb{E}\|\theta^{(c)}_{(s+1)T,0}-\theta^*\|^2
    \leq &(1-\frac{\mu(K\eta_{sT})^2}{8})\mathbb{E}\|\theta^{(c)}_{sT+T-1,0}-\theta^*\|^2+(K\eta_{sT})^2B^{(c)}_a\\
        \leq & \cdots\\
    \leq& (1-\frac{\mu(K\eta_{sT})^2}{8})^{T-1}\mathbb{E}\|\theta^{(c)}_{sT+1,0}-\theta^*\|^2+\sum_{t'=1}^{T-1} (1-\frac{\mu(K\eta_{sT})^2}{8})^{t'}(K\eta_{sT})^2B^{(c)}_a\\
    \leq& (1-\frac{\mu(K\eta_{sT})^2}{8})^T\mathbb{E}\|\theta_{sT}-\theta^*\|^2+\frac{8}{\mu}\Bigl(1- (1-\frac{\mu(K\eta_{sT})^2}{8})^T\Bigr)B^{(c)}_a
\end{align*}
where we use the step size setting $\eta_{t}=\eta_{t+1}=\cdots=\eta_{t+T-1}$ for any $t\in T\mathbb{Z}$ and $\theta^{(c)}_{sT,0}=\theta_{sT}$.

By Lemma \ref{lem:convex}, it follows that
\begin{align*}
  &\mathbb{E}\|\theta_{(s+1)T}-\theta^*\|^2=\mathbb{E}\|\sum_{c=1}^Nw_c\theta^{(c)}_{(s+1)T}-\sum_{c=1}^Nw_c\theta^*\|^2
  \leq  \sum_{c=1}^Nw_c\mathbb{E}\|\theta^{(c)}_{(s+1)T}-\theta^*\|^2\\
  \leq & (1-\frac{\mu(K\eta_{sT})^2}{8})^T\mathbb{E}\|\theta_{sT}-\theta^*\|^2+\frac{8}{\mu}\Bigl(1- (1-\frac{\mu(K\eta_{sT})^2}{8})^T\Bigr)\sum_{c=1}^Nw_cB^{(c)}_a
\end{align*}
Further by induction on $s=0,1,\ldots$ and recall that we define $D=\mathbb{E}\|\theta_0-\theta^*\|^2$
\begin{align*}
\mathbb{E}\|\theta_{(s+1)T}-\theta^*\|^2
    \leq& \prod_{s'=0}^s(1-\frac{\mu(K\eta_{s'T})^2}{8})^TD+\sum_{s'=0}^s\prod_{s''=s'}^s(1-\frac{\mu(K\eta_{s''T})^2}{8})^T\\
    &\cdot\Bigl(1- (1-\frac{\mu(K\eta_{s'T})^2}{8})^T\Bigr)\frac{8\sum_{c=1}^Nw_cB^{(c)}_a}{\mu}\\
    =& \prod_{s'=0}^s(1-\frac{\mu(K\eta_{s'T})^2}{8})^TD+\Bigl(1- \prod_{s'=0}^s(1-\frac{\mu(K\eta_{s'T})^2}{8})^T\Bigr)\frac{8\sum_{c=1}^Nw_cB^{(c)}_a}{\mu}\\
     \leq&D+\frac{8}{\mu}\sum_{c=1}^Nw_cB^{(c)}_a.
\end{align*}
Combining the two induction procedures above, it is seen that for any $s=0,1,\ldots$, $0\leq t\leq T-1$
\begin{align*}
\mathbb{E}\|\theta_{sT+t}-\theta^*\|^2
    \leq& (1-\frac{\mu(K\eta_{s'T})^2}{8})^t(D+\frac{8}{\mu}\sum_{c=1}^Nw_cB^{(c)}_a)+\sum_{t'=1}^{t-1} (1-\frac{\mu(K\eta_{sT})^2}{8})^{t'}(K\eta_{sT})^2B^{(c)}_a\\
     \leq&D+\frac{8}{\mu}\sum_{c=1}^Nw_cB^{(c)}_a+\frac{8}{\mu}B^{(c)}_a.
\end{align*}
This finishes the proof of \eqref{eqn:disunib:a1}.

For \eqref{eqn:disunib:a2}, Lemma \ref{lem:convex} and smoothness Assumption \ref{assum:smooth} imply that
\[
\mathbb{E}\|\nabla f^{(c)}(\theta^{(c)}_{t,0})\|^2\leq 2\mathbb{E}\|\nabla f^{(c)}(\theta^{(c)}_{t,0})-f^{(c)}(\theta^*)\|^2+2\|f^{(c)}(\theta^*)\|^2\leq 2L^2\mathbb{E}\|\theta_t-\theta^*\|^2+2\|\nabla f^{(c)}(\theta^*)\|^2.
\]
Taking supremum over $t\geq 0$ on both sides, then \eqref{eqn:disunib:a2} directly follows by the above inequality and \eqref{eqn:disunib:a1}. 
\end{proof}

\subsection{Proof of Lemma \ref{lem:dislocb}}\label{pflem:dislocb}
\lemdislocb*

\begin{proof}
Denote $t_0$ be the largest communication step before $t$. By definitions (\ref{eq:iteraform}), we have that 
\[
   \theta^{(c)}_t=\theta^{(c)}_{t_0,0}+\sum_{i=t_0}^{t-1}\Bigl[K\eta_i p^{(c)}_i-\frac{K\eta_i^2}{2}\widetilde{g}^{(c)}_{i,0}(\xi_0)-\frac{\eta_i^2}{2}\sum_{k=1}^{K-1}(K-k)\Bigl(\widetilde{g}^{(c)}_{i,k}(\xi_{k-\frac{1}{2}})+\widetilde{g}^{(c)}_{i,k}(\xi_k)\Bigr)\Bigr].
\]
To tackle  the stochastic gradient and its dependence on stochastic gradient, we rewrite it into
\begin{align*}
   \theta^{(c)}_t=&\theta^{(c)}_{t_0,0}+\sum_{i=t_0}^{t-1}K\eta_i\Bigl[ p^{(c)}_i-\frac{K\eta_i^2}{2}\big(\widetilde{g}^{(c)}_{i,0}(\xi_0)-\nabla f^{(c)}(\theta^{(c)}_{i,0})+\nabla f^{(c)}(\theta^{(c)}_{i,0})\big)-\frac{\eta_i^2}{2}\sum_{k=1}^{K-1}(K-k)\\
   &\Bigl(\widetilde{g}^{(c)}_{i,k}(\xi_{k-\frac{1}{2}})+\widetilde{g}^{(c)}_{i,k}(\xi_k)-2\nabla f^{(c)}(\theta^{(c)}_{i,k})+2(\nabla f^{(c)}(\theta^{(c)}_{i,k})-\nabla f^{(c)}(\theta^{(c)}_{i,0})+\nabla f^{(c)}(\theta^{(c)}_{i,0}))\Bigr)\Bigr]\\
   =&\theta^{(c)}_{t_0,0}+\sum_{i=t_0}^{t-1}\Bigl(-(K\eta_i)^2\text{(I)}^{(c)}_i+K\eta_i\widetilde{\text{(I)}}^{(c)}_i\Bigr)
\end{align*}
where
\begin{align*}
    \text{(I)}^{(c)}_i=&\frac{1}{2}\nabla f^{(c)}(\theta^{(c)}_{i,0})+\sum_{k=1}^{K-1}\frac{K-k}{K^2}\big(\nabla f^{(c)}(\theta^{(c)}_{i,0})-\nabla f^{(c)}(\theta^{(c)}_{i,k})\big),\\
    \widetilde{\text{(I)}}^{(c)}_i=&p^{(c)}_i-\frac{\eta_i}{2}\big[(\widetilde{g}^{(c)}_{i,0}(\xi_0)-\nabla f^{(c)}(\theta^{(c)}_{i,0}))+\sum_{k=1}^{K-1}\frac{K-k}{K}\big(\widetilde{g}^{(c)}_{i,k}(\xi_{k-\frac{1}{2}})+\widetilde{g}^{(c)}_{i,k}(\xi_k)-2\nabla f^{(c)}(\theta^{(c)}_{i,k})\big)\big].
\end{align*}
Here is $\text{(I)}^{(c)}_i$ as the first-order approximation with non-stochastic gradient and $\widetilde{\text{(I)}}^{(c)}_i$ is the random error brought by stochastic gradient.

Denote
\[
 \text{(I)}_i=\sum_{s=1}^Nw_s(\frac{w_s-1}{w_s})^{\mathbb{I}_{\{s=c\}}}\text{(I)}^{(s)}_i,\qquad 
\widetilde{\text{(I)}}_i=\sum_{s=1}^Nw_s(\frac{w_s-1}{w_s})^{\mathbb{I}_{\{s=c\}}}\widetilde{\text{(I)}}^{(s)}_i. 
\]
Combining this with  FL setting that  $\theta^{(c)}_{t_0,0}=\theta_{t_0}=\sum_{s=1}^Nw_s\theta^{(s)}_{t_0,0}$, and that  $w_c-\sum_{s=1}^Nw_s=\sum_{s=1}^Nw_s(\frac{w_s-1}{w_s})^{\mathbb{I}_{\{s=c\}}}$, we can write  
\[
\theta^{(c)}_{t,0}-\theta_t=\sum_{i=t_0}^{t-1}\big(-(K\eta_i)^2\text{(I)}_i+(K\eta_i)\widetilde{\text{(I)}}_i\big).
 \]
 By Lemma \ref{lem:convex}, we have for any $c_1,c_2>0$ such that $c_1+c_2=1$
 \[
\sum_{c=1}^Nw_c\mathbb{E}\|\theta^{(c)}_{t,0}-\theta_t\|^2\leq \sum_{c=1}^Nw_c (t-t_0)\sum_{i=t_0}^t(\frac{1}{c_1}\mathbb{E}\|(K\eta_i)^2\text{(I)}_i\|^2+\frac{1}{c_2}\mathbb{E}\|K\eta_i\widetilde{\text{(I)}}_i\|^2).
 \]
 Taking $(c_1,c_2)=(\frac{4}{9},\frac{5}{9})$, to prove the claim of the Lemma, it suffices to prove that
 \begin{align}
       \sum_{c=1}^Nw_c\mathbb{E}\|\text{(I)}_i\|^2\leq&\sum_{c=1}^Nw_c(1-w_c)^2\big(\frac{5}{4}\mathbb{E}\|\nabla f^{(c)}(\theta^{(c)}_{i,0})\|^2+\frac{L\sigma_t^{(c)2}d}{36}+ \frac{\sigma_g^2d}{288K}\big),\label{pflem:dislocb:1}\\
     \sum_{c=1}^Nw_c\mathbb{E}\|\widetilde{\text{(I)}}_i\|^2\leq& \sum_{c=1}^Nw_c\Bigl((1-w_c)\frac{K\eta_i^2}{4}\sigma_g^2d+(\sum_{c=1}^Nw_c\sigma_t^{(c)2}-1) d\Bigr)\label{pflem:dislocb:2}.
 \end{align}
We consider $\sum_{c=1}^Nw_c\mathbb{E}\|\widetilde{\text{(I)}}_i\|^2$ first since it is easier to bound.   Similar to the argument as that in (\ref{pfthm:disFLHMC:2}), we have
 \begin{align*}
    \sum_{c=1}^Nw_c\mathbb{E}\|\widetilde{\text{(I)}}_i\|^2
   \leq& \sum_{c=1}^Nw_c\Bigl(\sum_{s=1}^N\big(w_s(\frac{w_s-1}{w_s})^{\mathbb{I}_{\{s=c\}}}\big)^2\frac{K\eta_i^2}{4}\sigma_g^2d+\mathbb{E}\|\sum_{s=1}^Nw_s(\frac{w_s-1}{w_s})^{\mathbb{I}_{\{s=c\}}}p^{(s)}_i\|^2\Bigr)\\
        = & \sum_{c=1}^Nw_c(1-w_c)\frac{K\eta_i^2}{4}\sigma_g^2d+(\sum_{c=1}^Nw_c\sigma_t^{(c)2}-1)d,
\end{align*}
where the first term in the second equality is  by $\sum_{c=1}^Nw_c \sum_{s=1}^N\Bigl(w_s(\frac{w_s-1}{w_s})^{\mathbb{I}_{\{s=c\}}}\Bigr)^2=\sum_{c=1}^Nw_c(1-w_c)$ and the second term is by normality of $\{p^{(c)}_i\}_c$ and $\mathbb{E}[p^{(c_1)}_ip^{(c_2)}_i]=\rho_id$, for any $c_1,c_2\in [N]$.
This proves (\ref{pflem:dislocb:2}).

Next, consider $\sum_{c=1}^Nw_c\mathbb{E}\|\text{(I)}_i\|^2$. 
By direct calculations, we have
\begin{align*}
\sum_{c=1}^N&w_c\mathbb{E}\|\text{(I)}^{(c)}_i\|^2\\
     \leq  & \sum_{c=1}^Nw_c\sum_{s=1}^N\frac{w_s}{2(1-w_c)}(\frac{1-w_s}{w_s})^{\mathbb{I}_{\{s=c\}}}\big(\mathbb{E}\|\frac{2(1-w_c)}{2}\nabla f^{(c)}(\theta^{(c)}_{i,0})\|^2\\
     &+\sum_{k=1}^{K-1}\frac{6(K-k)k}{K^3-K}\mathbb{E}\big\|\frac{2(1-w_c)(K^3-K)}{6kK^2}(\nabla f^{(s)}(\theta^{(s)}_{i,0})-\nabla f^{(s)}(\theta^{(s)}_{i,k}))\big\|^2\big)\\
  \leq & \sum_{c=1}^Nw_c(1-w_c)^2\big(\mathbb{E}\|\nabla f^{(c)}(\theta^{(c)}_{i,0})\|^2+\sum_{k=1}^{K-1}\frac{2(K-k)}{3Kk}\mathbb{E}\|\nabla f^{(c)}(\theta^{(c)}_{i,0})-\nabla f^{(c)}(\theta^{(c)}_{i,k})\|^2\big)\\
\leq& \sum_{c=1}^Nw_c(1-w_c)^2\big(\mathbb{E}\|\nabla f^{(c)}(\theta^{(c)}_{i,0})\|^2+\sum_{k=1}^{K-1}\frac{2(K-k)}{3Kk} L^2(2(k\eta_i)^2\mathbb{E}\|p^{(c)}_i\|^2\\
&+(k\eta_i)^4\mathbb{E}\|\nabla f^{(c)}(\theta^{(c)}_{i,0})\|^2+\frac{(k\eta_i)^4}{k}\sigma_g^2d)\\
     \leq &\sum_{c=1}^Nw_c(1-w_c)^2\big((1+\frac{2L^2(K\eta_i)^4}{3*20})\mathbb{E}\|\nabla f^{(c)}(\theta^{(c)}_{i,0})\|^2+\frac{2L^2(K\eta_i)^2}{3*6}\sigma_t^{(c)2}d+ \frac{2L^2(K\eta_i)^4}{3*12K}\sigma_g^2d\big)
\end{align*}
where the first  inequality is by Lemma \ref{lem:convex} and noting that $\sum_{s=1}^N\frac{w_s}{2(1-w_c)}(\frac{1-w_s}{w_s})^{\mathbb{I}_{\{s=c\}}}=1$ and $\sum_{k=1}^{K-1}\frac{(K-k)k}{(K^3-K)/6}=1$;  the equality is by noting that  $\sum_{c=1}^Nw_c\sum_{s=1}^N\frac{w_s}{2(1-w_c)}(1-w_c)^2(\frac{1-w_s}{w_s})^{\mathbb{I}_{\{s=c\}}}=\sum_{c=1}^Nw_c(1-w_c)^2$;   the second inequality is by (\ref{eqn:disthetab:a2}) in Lemma \ref{lem:disthetab};  the third inequality is  by $\sum_{j=1}^{k-1}(k-j)j\leq \frac{k^3}{6}$, $\sum_{j=1}^{k-1}(k-j)j^2\leq \frac{k^4}{12}$ and  $\sum_{j=1}^{k-1}(k-j)j^3\leq\frac{k^5}{20}$. Now by assumption $K\eta\leq 1/\sqrt{L}$ and elementary numerical calculations, we get (\ref{pflem:dislocb:1})  and finish the proof.
\end{proof}
\subsection{Proof of Lemma \ref{lem:FLorab}}\label{pflem:FLorab}
\lemFLorab*
\begin{proof}
To prove \eqref{eqn:FLorab:a}, we notice that
\begin{align*}
   \mathbb{E}\|\theta^\pi_t(u)-\theta^\pi_t-p_tu\|^2
    = &\frac{u^4}{4}\mathbb{E}\Bigl\|\int_0^u\int_0^s\frac{2}{u^2}\nabla f(\theta^\pi_t(r))drds\Bigr\|^2\\
    \leq&\frac{u^4}{4}\mathbb{E}\int_0^u\int_0^s\frac{2}{u^2}\Bigl\|\nabla f(\theta^\pi_t(r))\Bigr\|^2drds
    =\frac{u^2}{2}\int_0^u\int_0^s\mathbb{E}\|\nabla f(\theta^\pi_t(r))\|^2drds\\
    \leq&\frac{u^2}{2}\int_0^u\int_0^sL^2\mathbb{E}\|\theta^\pi_t(r)-\theta^*\|^2drds
    \leq\frac{L^2d}{4\mu}u^4,
\end{align*}
where the first inequality is due to Lemma \ref{lem:convex}; the second inequality is by Assumption \ref{assum:smooth} and the fact that $\nabla f(\theta^*)=0$; the third inequality is by Lemma \ref{lem:Oracle}.

Next, we show \eqref{eqn:FLorab:b} holds. By direct calculations
\begin{align*}
\mathbb{E}\|\theta^\pi_t(u)-\theta^\pi_t\|^2\|p_t\|_{\infty}^2\leq& \mathbb{E}\|Cu p_t(\frac{u}{2})\|^2\|p_t\|_{\infty}^2\leq \frac{1}{3}u^2c_dd
\end{align*}
where  the first inequality is by noting that $\theta^\pi_t(u)$, $\theta^\pi_t$ are positions under HMC system ${\cal H}$ at time $\frac{u}{2}$ with the same initial $\theta^\pi_t(\frac{u}{2})$ but with opposite momentum $p_t(\frac{u}{2})$ and by Lemma \ref{lem:GuD3}; the second inequality is by Lemma \ref{lem:MaxGau}.

Next consider \eqref{eqn:FLorab:c},  
by  $\nabla f(\theta^\pi_t(u))=\nabla f(\theta^\pi_t)+\int_0^1\nabla^2 f(\theta^\pi_t+s(\theta^\pi_t(u)-\theta^\pi_t))(\theta^\pi_t(u)-\theta^\pi_t)ds$
\[
\nabla f(\theta^\pi_t(u))-\nabla f(\theta^\pi_t)-\nabla^2 f(\theta^\pi_t)p_tu=\int_0^1((a)+(b))ds
\]
where
\begin{align*}
    (a)&=\nabla^2 f(\theta^\pi_t+s(\theta^\pi_t(u)-\theta^\pi_t))(\theta^\pi_t(u)-\theta^\pi_t-p_tu),\\
    (b)&=\Bigl(\nabla^2 f(\theta^\pi_t+s(\theta^\pi_t(u)-\theta^\pi_t))-\nabla^2 f(\theta^\pi_t)\Bigr)p_tu.
\end{align*}
By Lemma \ref{lem:convex}, for any $c_1,c_2>0$ such that $c_1+c_2=1$, we have
\begin{align*}
    \mathbb{E}\|\nabla& f(\theta^\pi_t(u))-\nabla f(\theta^\pi_t)-\nabla^2 f(\theta^\pi_t)p_tu\|^2\\
    \leq &\int_0^1(\frac{1}{c_1}\mathbb{E}\|(a)\|^2+\frac{1}{c_2}\mathbb{E}\|(b)\|^2)ds\\
    \leq & \int_0^1(\frac{1}{c_1}L^2\mathbb{E}\|\theta^\pi_t(u)-\theta^\pi_t-p_tu\|^2+\frac{1}{c_2}L_H^2\mathbb{E}\|s(\theta^\pi_t(u)-\theta^\pi_t)\|^2\|p_tu\|_{\infty}^2)ds\\
    \leq & \int_0^1(\frac{1}{c_1}\frac{L^4d}{4\mu}u^4\sigma_{\rho_t}^4+\frac{1}{c_2}L_H^2s^2\frac{1}{3}u^2c_ddu^2)ds\\
    \leq &\frac{5}{4}L^3\big(\frac{L^4}{\mu}+L_H^2c_d\big)u^4d
\end{align*}
where in the second inequality, the first term is by Assumption \ref{assum:smooth} and the second term is by Assumption \ref{assum:Hsmooth}; the third inequality is by \eqref{eqn:FLorab:a} and \eqref{eqn:FLorab:b} and the last inequality is by direct calculations with $(c_1,c_2)=(\frac{1}{5},\frac{4}{5})$.
\end{proof}
\subsection{Proof of Lemma \ref{lem:MaxGau}}\label{pflem:MaxGau}
\lemMaxGau*
\begin{proof}
First, we introduce some preliminary results, for any $t>0$ and for any gaussian vector $p\sim N(0,\sigma^2\mathbb{I}_d)$
\[
\text{Borell–TIS inequality: }\qquad \mathbb{P}\big(\Bigl|\|p\|_{\infty}-\mathbb{E}\|p\|_{\infty}\Bigr|\geq t\big)\leq 2\exp(-\frac{t^2}{2\sigma^2}),
\]
and
\[
\mathbb{E}\|p\|_{\infty}\leq \sqrt{2\log(2d)}\sigma.
\]
Now, by direct calculations,
\begin{align*}
     \mathbb{E}[\|p\|_{\infty}-\mathbb{E}\|p\|_{\infty}]^2= & \int_0^{\infty} 2t\mathbb{P}\Bigl(\Bigl|\|p\|_{\infty}-\mathbb{E}\|p\|_{\infty}\Bigr|\geq t\Bigr)dt\leq \int_0^{\infty} 2t2\exp(-\frac{t^2}{2\sigma^2})dt=4\sigma^2,\\
     \mathbb{E}[\|p\|_{\infty}-\mathbb{E}\|p\|_{\infty}]^4= & \int_0^{\infty} 4t^3\mathbb{P}\Bigl(\Bigl|\|p\|_{\infty}-\mathbb{E}\|p\|_{\infty}\Bigr|\geq t\Bigr)dt\leq \int_0^{\infty}4t^32\exp(-\frac{t^2}{2\sigma^2})dt=16\sigma^4.
\end{align*}
and noting that
\begin{align*}
    \mathbb{E}\|p\|_{\infty}^2=&\mathbb{E}^2\|p\|_{\infty}+\mathbb{E}[\|p\|_{\infty}-\mathbb{E}\|p\|_{\infty}]^2,\\
     \mathbb{E}\|p\|_{\infty}^4=& \mathbb{E}[\|p\|_{\infty}-\mathbb{E}\|p\|_{\infty}+\mathbb{E}\|p\|_{\infty}]^4\leq 8\big(\mathbb{E}\|p\|_{\infty}^4+\mathbb{E}^4\|p\|_{\infty}\big),
\end{align*}
It follows that
\[
\mathbb{E}\|p\|_{\infty}^2\leq (2\log(2d)+4)\sigma^2, \qquad      \mathbb{E}\|p\|_{\infty}^4\leq  \big(128+32\log^2(2d)\big)\sigma^4.
\]
Moreover, by Cauchy-Schwarz inequality
\[
\mathbb{E}\|p\|^2\|p\|_{\infty}^2\leq \mathbb{E}\big[\frac{\|p\|^4}{2d}+\frac{d\|p\|_{\infty}^4}{2}\big]\leq \big(d+2+64d+16\log^2(2d)d\big)\sigma^4.
\]
Let $c_d:=128+32\log^2(2d)$. Combining the above results, we finish the proof.
\end{proof}
\section{Proof of Theorem \ref{thm:LBconFLHMC}}\label{pfthm:LBconFLHMC}
We consider a special case. Without loss of generality, we assume $N$ is even. The loss function $f^{(c)}(\theta)$ are as follow, for some constants $C_1,C_2>0$
\[
f^{(c)}(\theta)=\bigg\{\begin{array}{ll}
    L\|\theta-\theta^*_{L}\mathbb{I}_d\|^2/2+C_1, & c= 1,2,\ldots,\frac{N}{2}  \\
     \mu\|\theta-\theta^*_{\mu}\mathbb{I}_d\|^2/2+C_2,& c= \frac{N}{2}+1,\ldots ,N. 
\end{array}
\]
and weights are
\[
w_c=\frac{1}{N},\qquad c=1,2,\ldots,N
\]
Then we have the following result on the tightness of our analysis on the upper bound.
\begin{restatable}[Dimensional tight lower bound for ideal HMC process]{theorem}{}
Under Assumptions \ref{assum:convex}-\ref{assum:Hsmooth},  if we initialize $\theta_0\in N(0,\sigma^2\mathbb{I}_d)$ for any $\sigma^2> 0$, suppose $\theta^*_L>\theta^*_\mu>0$, then we have
\begin{align*}
    \theta_t=&\theta_0\gamma^t+\frac{1-\gamma^t}{1-\gamma}\frac{1}{2}(\frac{(1-\cos^2(\sqrt{L}K\eta))}{\sqrt{L}}\\
    &+\frac{(1-\cos^2(\sqrt{\mu}K\eta))}{\sqrt{\mu}})\mathbb{I}_d+\gamma \sqrt{\frac{1-\gamma^{2t}}{1-\gamma^2}}p
\end{align*}
 where $p\sim N(0,\mathbb{I}_d)$, $\gamma=\frac{1}{2}\big(\cos^T(\sqrt{LK\eta})+\cos^T(\sqrt{\mu K\eta})\big)$. 
 
 As a result, the least $t$ for ${\cal W}(\theta_t,\theta^\pi)^2\leq  \epsilon^2$  is taken when
\[
t=\Omega(\frac{\sqrt{d}T\log(d/\epsilon)}{\epsilon}).
\]
\end{restatable}
By definitions and elementary calculus, a parameter $q_0$ that follows HMC system ${\cal H}:f(q)=\frac{L}{2}\|q-q^*\|^2$ with length $\widetilde{\eta}:=K\eta$ with momentum $\widetilde{p}_0$ will give
\[
{\cal H}_{K\eta}(q_0)=q_0\cos(\sqrt{L\widetilde{\eta}})+q^*\big(1-\cos(\sqrt{L\widetilde{\eta}})\big)+\frac{\widetilde{p}_0}{\sqrt{L}}\sin(\sqrt{L\widetilde{\eta}})
\]
By induction, a parameter $q_0$ that follows HMC system ${\cal H}$ with independent momentum $\widetilde{p}_t$, $t=0,1,\ldots T-1$ will give
\begin{align*}
    {\cal H}^T_{K\eta}(q_0)=&q_0\cos^T(\sqrt{L\widetilde{\eta}})+q^*\big(1-\cos^T(\sqrt{L\widetilde{\eta}})\big)+\sum_{t=0}^{T-1}\frac{\widetilde{p}_t}{\sqrt{L}}\sin(\sqrt{L\widetilde{\eta}})\cos^{T-t-1}(\sqrt{L\widetilde{\eta}})\\
    =&q_0\cos^T(\sqrt{L\widetilde{\eta}})+q^*\big(1-\cos^T(\sqrt{L\widetilde{\eta}})\big)+p_0\frac{\big(1-\cos^{2T}(\sqrt{L\widetilde{\eta}})\big)^{1/2}}{\sqrt{L}}
\end{align*}
where the last equality is by the Gaussian property that the sum of $T$ Gaussian's is still a Gaussian random variable.

Let $\gamma=\frac{1}{2}\big(\cos^T(\sqrt{L\widetilde{\eta}})+\cos^T(\sqrt{\mu\widetilde{\eta}})\big)$ be the contraction factor and $\theta^*_{\widetilde{\eta}}=\frac{1}{2(1-\gamma)}\Bigl(\theta^*_L\big(1-\cos^T(\sqrt{L\widetilde{\eta}})\big)+\theta^*_\mu\big(1-\cos^T(\sqrt{\mu\widetilde{\eta}})\big)\Bigr)$. It follows that for FA continuous HMC, we have
\begin{align*}
    \theta_{t+1}=&\theta_t\frac{1}{2}\big(\cos^T(\sqrt{L\widetilde{\eta}})+\cos^T(\sqrt{\mu\widetilde{\eta}})\big)+\frac{1}{2}\Bigl(\theta^*_L\big(1-\cos^T(\sqrt{L\widetilde{\eta}})\big)+\theta^*_\mu\big(1-\cos^T(\sqrt{\mu\widetilde{\eta}})\big)\Bigr)\\
    &\cdot\mathbb{I}_d+p_t\frac{1}{2}\Bigl(\frac{\big(1-\cos^{2T}(\sqrt{L\widetilde{\eta}})\big)^{\frac{1}{2}}}{\sqrt{L}}+\frac{\big(1-\cos^{2T}(\sqrt{\mu\widetilde{\eta}})\big)^{\frac{1}{2}}}{\sqrt{\mu}}\Bigr)\\
    =&\theta_t\gamma+\theta^*_{\widetilde{\eta}}(1-\gamma)\mathbb{I}_d+p_t\frac{1}{2}\Bigl(\frac{\big(1-\cos^{2T}(\sqrt{L\widetilde{\eta}})\big)^{\frac{1}{2}}}{\sqrt{L}}+\frac{\big(1-\cos^{2T}(\sqrt{\mu\widetilde{\eta}})\big)^{\frac{1}{2}}}{\sqrt{\mu}}\Bigr)\\
    =&\cdots\\
    =&\theta_0\gamma^t+\theta^*_{\widetilde{\eta}}(1-\gamma^t)\mathbb{I}_d+\sum_{i=0}^{t-1}\gamma^i p_i\frac{1}{2}\Bigl(\frac{\big(1-\cos^{2T}(\sqrt{L\widetilde{\eta}})\big)^{\frac{1}{2}}}{\sqrt{L}}+\frac{\big(1-\cos^{2T}(\sqrt{\mu\widetilde{\eta}})\big)^{\frac{1}{2}}}{\sqrt{\mu}}\Bigr)\\
        =&\theta_0\gamma^t+\theta^*_{\widetilde{\eta}}(1-\gamma^t)\mathbb{I}_d+ p\sqrt{\frac{1-\gamma^{2t}}{1-\gamma^2}}\frac{1}{2}\Bigl(\frac{\big(1-\cos^{2T}(\sqrt{L\widetilde{\eta}})\big)^{\frac{1}{2}}}{\sqrt{L}}+\frac{\big(1-\cos^{2T}(\sqrt{\mu\widetilde{\eta}})\big)^{\frac{1}{2}}}{\sqrt{\mu}}\Bigr)
\end{align*}
whereby the property of Gaussian, we combine $t$ Gaussians into one Gaussian $p$. This proves the first claim for $t$ being the multiples of $T$.

By the closed form solution ${\cal W}(N(\mu_1,\Sigma_1),N(\mu_2,\Sigma_2))^2=\|\mu_1-\mu_2\|^2+\text{tr}\big(\Sigma_1+\Sigma_2-2(\Sigma_2^{\frac{1}{2}}\Sigma_2\Sigma_2^{\frac{1}{2}})^{\frac{1}{2}}\big)\geq \|\mu_1-\mu_2\|^2$, and recall that we initialize $\theta_0\sim N(0,\sigma^2\mathbb{I}_d)$, we get
\[
    {\cal W}(\theta_t,\theta^\pi)^2\geq \|(\theta^*_{\widetilde{\eta}}(1-\gamma^t)-\theta^*)\mathbb{I}_d\|^2=\gamma^{2t}(\theta^*_{\widetilde{\eta}})^2d-2\gamma^t\theta^*_{\widetilde{\eta}}(\theta^*_{\widetilde{\eta}}-\theta^*) d+(\theta^*_{\widetilde{\eta}}-\theta^*)^2d,
\]
where $\theta^*=\frac{L\theta_L^*+\mu\theta_\mu^*}{L+\mu}$ is the global minimum of the FA-HMC system (This can be easily checked).

We claim that by the conditions $L\geq \mu$ and $\theta_L^*\geq \theta_\mu^*>0$, we have $0<\theta^*_{\widetilde{\eta}}\leq \theta^*$, which will be proved after the main proof. Then
\[
    {\cal W}(\theta_t,\theta^\pi)^2\geq (\theta^*_{\widetilde{\eta}})^2\gamma^{2t}d+(\theta^*_{\widetilde{\eta}}-\theta^*)^2d,
\]
At the same time, by elementary calculus, for $\widetilde{\eta}=K\eta\leq \frac{1}{\sqrt{L}}$, there exist $C_1,C_2$ such that
\[
\gamma\geq 1-C_1\frac{T(L+\mu)}{2}\widetilde{\eta}^2,\qquad (\theta^*_{\widetilde{\eta}}-\theta^*)^2\geq C_2T^2(\frac{L^2\theta_L^*+\mu^2\theta_\mu^*}{L+\mu})^2\widetilde{\eta}^4.
\]
We have
\[
    {\cal W}(\theta_t,\theta^\pi)^2\geq (\theta^*_{\widetilde{\eta}})^2(1-C_1\frac{T(L+\mu)}{2}\widetilde{\eta}^2)^{2t}d+C_2T^2(\frac{L^2\theta_L^*+\mu^2\theta_\mu^*}{L+\mu})^2\widetilde{\eta}^4d,
\]
Let ${\cal W}(\theta_t,\theta^\pi)^2\leq \epsilon$, the least $t/T$ on the RHS is taken when
\[
t/T=\Omega(\frac{\sqrt{d}\log(d/\epsilon)}{\epsilon}).
\]
Last,  recall that
\[
\theta^*_{\widetilde{\eta}}=\frac{\theta^*_L\big(1-\cos^T(\sqrt{L\widetilde{\eta}})\big)+\theta^*_\mu\big(1-\cos^T(\sqrt{\mu\widetilde{\eta}})\big)}{2(1-\gamma)}=\frac{\theta^*_L\big(1-\cos^T(\sqrt{L\widetilde{\eta}})\big)+\theta^*_\mu\big(1-\cos^T(\sqrt{\mu\widetilde{\eta}})\big)}{2-\cos^T(\sqrt{L\widetilde{\eta}})-\cos^T(\sqrt{\mu\widetilde{\eta}})}
\]
To show  $\theta^*_{\widetilde{\eta}}\leq \theta^*$, it suffices to show that for any $a,b>0$ such that $a/(a+b)\leq L/(L+\mu)$ we have
\[
\frac{a}{a+b}\theta^*_L+\frac{b}{a+b}\theta^*_\mu\leq \frac{L}{L+\mu}\theta^*_L+\frac{\mu}{L+\mu}\theta^*_\mu,\qquad \text{and}\qquad \frac{1-\cos^T(\sqrt{L\widetilde{\eta}})}{2-\cos^T(\sqrt{L\widetilde{\eta}})-\cos^T(\sqrt{\mu\widetilde{\eta}})} \leq \frac{L}{L+\mu}
\]
The first inequality is implied by $\theta_L^*\geq \theta_\mu^*$. It is left to show the second inequality, which is equivalent to show
\[
 \frac{\mu\big(1-\cos^T(\sqrt{L\widetilde{\eta}})\big)}{L\big(1-\cos^T(\sqrt{\mu\widetilde{\eta}})\big)} \leq 1
\]
Note that
\[
 \frac{1-\cos^T(\sqrt{L\widetilde{\eta}})}{1-\cos^T(\sqrt{\mu\widetilde{\eta}})} =\frac{1-\cos(\sqrt{L\widetilde{\eta}})}{1-\cos(\sqrt{\mu\widetilde{\eta}})}\frac{\sum_{i=0}^{T-1}\cos(\sqrt{L\widetilde{\eta}})^i}{\sum_{i=0}^{T-1}\cos(\sqrt{\mu\widetilde{\eta}})^i}\leq \frac{1-\cos(\sqrt{L\widetilde{\eta}})}{1-\cos(\sqrt{\mu\widetilde{\eta}})}\cdot 1
\]
It suffices to show  for all $\eta \leq \frac{1}{2\sqrt{L}}$
\[
 \frac{ \mu\big(1-\cos(\sqrt{L\widetilde{\eta}})\big)}{L\big(1-\cos(\sqrt{\mu\widetilde{\eta}})\big)} \leq 1,\qquad \text{or}\qquad \mu\big(1-\cos(\sqrt{L\widetilde{\eta}})\leq L\big(1-\cos(\sqrt{\mu\widetilde{\eta}})\big)
\]
Note that by elementary calculus, $g(\eta):=\mu\big(1-\cos(\sqrt{L\eta})-L\big(1-\cos(\sqrt{\mu\eta})\big)$ satisfies that $g(0)=0$, $g'(0)$, $g''(\eta)=L\mu\big(\cos(\sqrt{L\eta})-\cos(\sqrt{\mu\eta})\big)\leq 0$, for all $\eta \leq \frac{1}{2\sqrt{L}}$, which implies that $g(\eta)\leq 0$  for all $\eta \leq \frac{1}{2\sqrt{L}}$ and proves the above inequality. This finishes the proof.


\section{Proof of Proposition 4.5 
}\label{sec:pfpropdynamic}
We show a more general  version of  Proposition 4.5 
as follow
\begin{proposition}[Dynamic stepsize]Suppose for the $\{\theta_t\}_t$ yielded by Algorithm 3, 
we have that  
for some $C>1$,
\begin{align*}
  \mathbb{E}&\|\theta_{t+1}-\theta^\pi_{t+1}\|^2
     \leq(1-\frac{\mu(K\eta_t)^2}{4})^tD+C(K\eta_t)^2c_d\widetilde{\Delta}
\end{align*}
with
\[
\widetilde{\Delta}=d(T^2(\gamma+(1-\rho)N)+\sum_{c=1}^Nw_c^2\sigma_g^2/K).
\]
If we set
\[
(K\eta)^2=\frac{D}{8c_d\widetilde{\Delta}L},\qquad t_1=\lceil\frac{-\log (8)}{T\log(1-\mu(K\eta)^2/4)}\rceil T
\]
and
\[
\eta_t=\eta,\quad \text{for }0\leq t\leq t_1-1.
\]
 and for $j=2,3,\ldots$
and for $s=1,2,\cdots$
\[
t_{s+1}=2t_s,\eta_{t''}=\frac{\eta_{t'}}{\sqrt{2}},\qquad
\]
$\text{ for any  }\sum_{j=1}^{s-1}t_j\leq t'\leq \sum_{j=1}^st_j-1,\;\sum_{j=1}^st_j\leq t''\leq \sum_{j=1}^{s+1}t_j-1.$ then we have
\[
\mathbb{E}\|\theta_t-\theta^\pi_t\|^2\leq \epsilon^2,\qquad \text{at some}\quad t\leq \frac{Cc_d\widetilde{\Delta}}{\epsilon^2}.
\]
\end{proposition}

\begin{proof}

The procedure follows the proof of the Theorem 14 in \cite{ccbj18}. For completeness, we give the proof of a stronger result below, for any $a>0$
\begin{equation}\label{pfprop:dynamic:1}
   \mathbb{E}\|\theta_t-\theta^\pi_t\|^2\leq \epsilon^2,\qquad \text{at}\quad t=\frac{4\log (2a)}{\mu}\frac{c_d\widetilde{\Delta}}{D}(\frac{D}{\epsilon^2})^{\log_a(2)}.
\end{equation}
if we change to 
\[
(K\eta)^2=\frac{D}{2ac_d\widetilde{\Delta}L},\qquad t_1=\lceil\frac{-\log (2a)}{T\log(1-\mu(K\eta)^2/4)}\rceil T\qquad  \text{and}\qquad \eta_t=\eta,\quad \text{for }0\leq t\leq t_1-1.
\]
and for $s=1,2,\cdots$
\[
t_{s+1}=2t_s,\eta_{t''}=\frac{\eta_{t'}}{a^{1/4}},\qquad\text{ for any  }\sum_{j=1}^{s-1}t_j\leq t'\leq \sum_{j=1}^st_j-1,\;\sum_{j=1}^st_j\leq t''\leq \sum_{j=1}^{s+1}t_j-1.
\]
Then the claim of this proposition follows by setting $a=2$.

Note that  by the definitions of $t_j$ and $\eta_j$ assumptions,
\begin{align*}
    (1-\frac{\mu(K\eta_{t_j})^2}{4})^{t_j}=&(1-\frac{\mu(K\eta_{t_1})^2}{4\cdot 2^{j-1}})^{2^{j-1}t_1}\leq(1-\frac{\mu(K\eta_{t_1})^2}{4\cdot 2^{j-1}})^{2^{j-1}\frac{-\log (2a)}{\log(1-\mu(K\eta)^2/4)}}\\
    =&(\frac{1}{2a})^{\frac{2^{j+1}\log(1-\mu(K\eta)^2/2^{j+1})}{\log(1-\mu(K\eta)^2/4)}}\leq \frac{1}{2a},
\end{align*}
where the first inequality is by inserting the definition of $t_1$; the last inequality is by $y\log(1-x/y)/\log(1-x)\geq 1$ for any $x>0,y\geq 1$.

Based on this, (let $t^{(s)}=\sum_{j=1}^st_j$), it is seen that 
\begin{align*}
  \mathbb{E}\|\theta_{t^{(s)}}-\theta^\pi_{t^{(s)}}\|^2\leq& \frac{1}{2a}\mathbb{E}\|\theta_{t^{(s-1)}}-\theta^\pi_{t^{(s-1)}}\|^2+(K\eta_{t_s})^4Lc_d\widetilde{\Delta}\\
    =& \frac{1}{2a}\mathbb{E}\|\theta_{t^{(s-1)}}-\theta^\pi_{t^{(s-1)}}\|^2+\frac{D}{2a^s}\\
    \leq &\cdots \\
    \leq&(\frac{1}{2a})^sD+\sum_{j=1}^{s-1}\frac{D}{(2a)^j\cdot 2^{s-j}}+\frac{D}{2a^s}=\frac{D}{a^s}.
\end{align*}
 Choosing $s$ such that $\frac{D}{a^s}>\epsilon^2$, $\frac{D}{a^{s+1}}\leq \epsilon^2$, then we get
\[
  \mathbb{E}\|\theta_{t^{(s)}}-\theta^\pi_{t^{(s)}}\|^2\leq\epsilon^2,\quad t^{(s)}=\sum_{s=1}^st_j=t_1\cdot \sum_{j=0}^{s-1}2^j\leq t_1\cdot 2^s\leq 2\cdot\frac{-\log (2a)}{\log(1-\mu(K\eta)^2/4)} 2^{\log_a(D/\epsilon^2)}.
\]
Further by $-\log(1-x)\geq x$ for any $x<1$, and $(K\eta)^2=\frac{D}{2ac_d\widetilde{\Delta}L}$
\[
t^{(s)}\leq \frac{-8\log (2a)}{\mu(K\eta)^2} 2^{\log_a(D/\epsilon^2)}=\frac{32\log (2a)}{\mu}\frac{ac_d\widetilde{\Delta}L}{D}\frac{D}{\epsilon^2}.
\]
This is equivalent to (\ref{pfprop:dynamic:1}).



\end{proof}
\section{Relaxation of Variance of Stochastic Gradients}\label{sec:relaxation}
If we relax

\begin{customassu}{3.4}[$\sigma_g$-Bounded Variance]For local device $c=1,2,\ldots,N$, and leapfrog step $k=1,2,\ldots,K$, $t=1,2,\ldots$, we have $\max_{x=k-1/2,k}\mathrm{tr(Var}(\nabla \widetilde{f}^{(c)}(\theta^{(c)}_{t,k},\xi^{(c)}_{t,x})|\theta^{(c)}_{t,k}))\leq \sigma_g^2Ld$,
 for some $\sigma_g>0$.
\end{customassu}
\noindent to assumption:
\[
\max_{x=k-\frac{1}{2},k}\Bigl\{\mathbb{E}_{\xi}\|\nabla \widetilde{f}^{(c)}(\theta^{(c)}_{t,k},\xi^{(c)}_{t,x})-\nabla f^{(c)}(\theta^{(c)}_{t,k})\|^2\Bigr\}\leq \sigma_g^2\big[\mathbb{E}_{\xi}\|\nabla f^{(c)}(\theta^{(c)}_{t,k})\|^2+d\big],
\]
similar results w.r.t. dimensionality will still hold as long as we can show that 
\begin{equation}\label{pfrelax:1}
   \mathbb{E}\|\nabla f^{(c)}(\theta^{(c)}_{t,k})\|^2\leq CLd,\qquad \text{ for some $C$}.  
\end{equation}
The above goal is equivalent to generalizing the results of  Lemma \ref{lem:disunib} to that adapted to  (\ref{pfrelax:1}) and Lemma \ref{lem:disthetab} which is used in the proof of Lemma \ref{lem:disunib}.

For  Lemma~\ref{lem:disthetab} with the weaker Assumption \ref{assum:boundvar}, if we repeat the procedure in Section \ref{pflem:disthetab}, we can get (for some constants $C_1,C_2,C_3>0$
\begin{align*}
      & \mathbb{E}_{\xi}\| q(k)-q(0)\|^2\leq C_1(k\eta)^2\|p(0)\|^2+C_2(k\eta)^4(1+\frac{\sigma_g^2}{K})\|\nabla F(q(0))\|^2+C_3(k\eta)^4\sigma_g^2d \\
      & \mathbb{E}_{\xi}\| \nabla F (q(k))-\nabla F(q(0))\|^2\leq L^2(C_1(k\eta)^2\|p(0)\|^2+C_2(k\eta)^4(1+\frac{\sigma_g^2}{K})\|\nabla F(q(0))\|^2+C_3(k\eta)^4\sigma_g^2d)
\end{align*}
The difference between this result and  Lemma~\ref{lem:disthetab} is that the coefficient of the term $|\nabla F(q(0))\|^2$ is related to $\sigma^2_g$. This won't affect the use of Lemma~\ref{lem:disthetab} as long as we have $\sigma_g^2=O(K)$ which is a mild assumption.

Further base on this, if we repeat the procedure of the proof of  Lemma \ref{lem:disunib} in  Section \ref{pflem:disthetab}, we can get (for some constants $C_4,C_5,C_6,C_7>0$)
\[
 \mathbb{E}\|\theta^{(c)}_{t+1}-\theta^*\|^2
   \leq(1-\frac{\mu(K\eta_t)^2}{4}+C_4(1+\frac{\sigma_g^2}{K})\frac{(K\eta_t)^6L^4}{\mu}\mathbb{I}_{\{K\geq 2\}})\mathbb{E}\|\theta^{(c)}_{t,0}-\theta^*\|^2+\frac{(K\eta_t)^2L}{\mu}\widetilde{B}^{(c)}_a, 
\]
where
\begin{align*}
\widetilde{B}^{(c)}_a=&C_5\frac{L}{\mu}d+C_6(1+\frac{\sigma_g^2}{K})\frac{\|\nabla f^{(c)}(\theta^*)\|^2}{\mu}+C_7\frac{\sigma_g^2}{K\mu}d.
\end{align*}
As long as we have $\sigma_g^2=O(K)$, then there is not much significant difference between this result and Lemma \ref{lem:disunib}.
\section{Supplementary Figures to Main Text}\label{secton:figures}
For a robust comparison, we add standard deviation information to some figures in the main text, based on multiple runs. Specifically, we present the shaded error: [average-standard deviation, average+standard deviation] based on independent runs and add them to the figures. 


In Figures \ref{fig:HMC-accu_eb} - \ref{fig:HMC-NLL_eb}, we add shaded error based on 5 independent runs to Figures 3(a) - 3(b). 

In Figures \ref{fig:local-accu_eb} - \ref{fig:local-NLL_eb}, we add shaded error based on 5 independent runs to Figures 4(a) - 4(b). 


\begin{figure*}[htbp]
    \centering
    \subfigure[Accuracy]{
    \begin{minipage}[t]{0.24\linewidth}
    \centering
    \label{fig:HMC-accu_eb}
    \includegraphics[width=\linewidth]{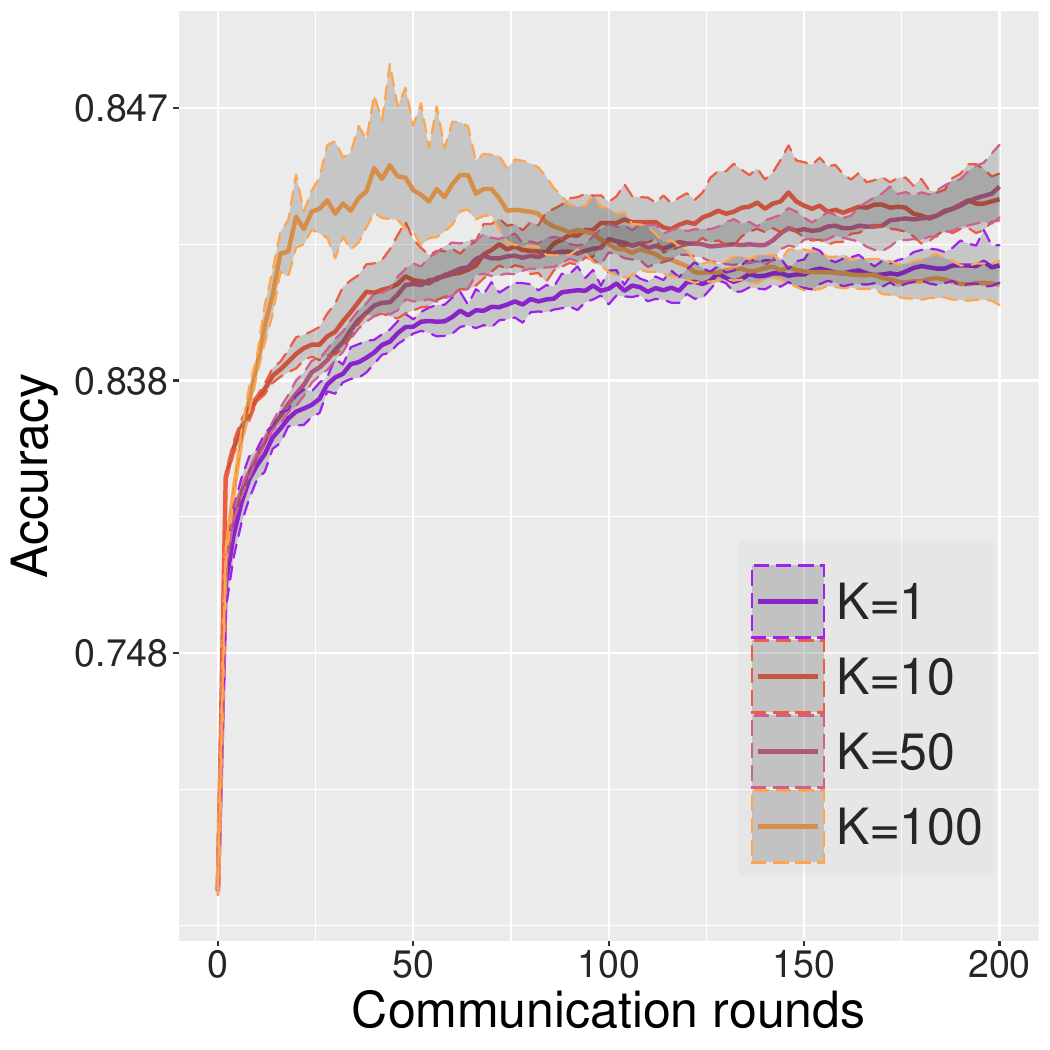}
    \end{minipage}%
    }%
    \subfigure[BS]{
    \begin{minipage}[t]{0.24\linewidth}
    \centering
    \label{fig:HMC-brier_eb}
    \includegraphics[width=\linewidth]{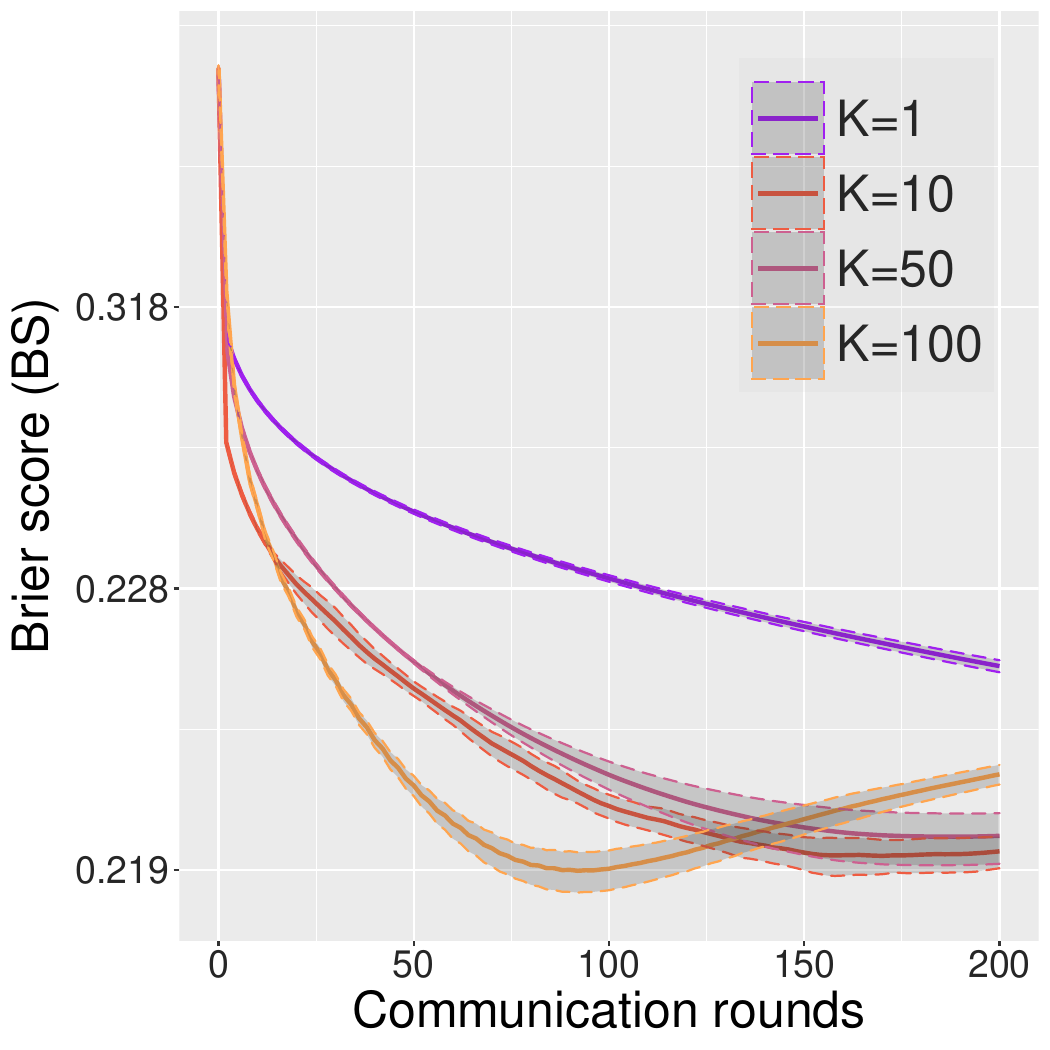}
    \end{minipage}%
    }%
    \subfigure[ECE]{
    \begin{minipage}[t]{0.24\linewidth}
    \centering
    \label{fig:HMC-ECE_eb}
    \includegraphics[width=\linewidth]{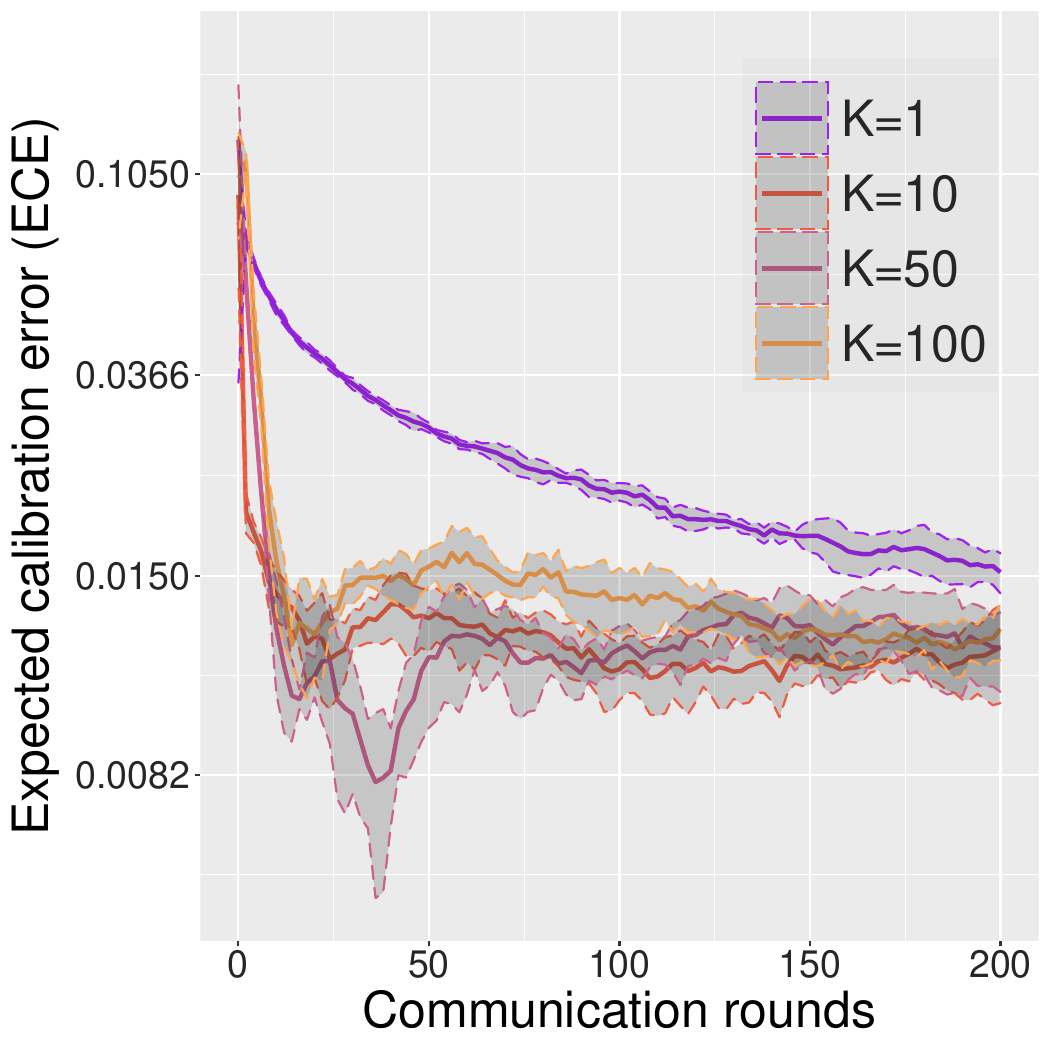}
    \end{minipage}%
    }%
    \subfigure[NLL]{
    \begin{minipage}[t]{0.24\linewidth}
    \centering
    \label{fig:HMC-NLL_eb}
    \includegraphics[width=\linewidth]{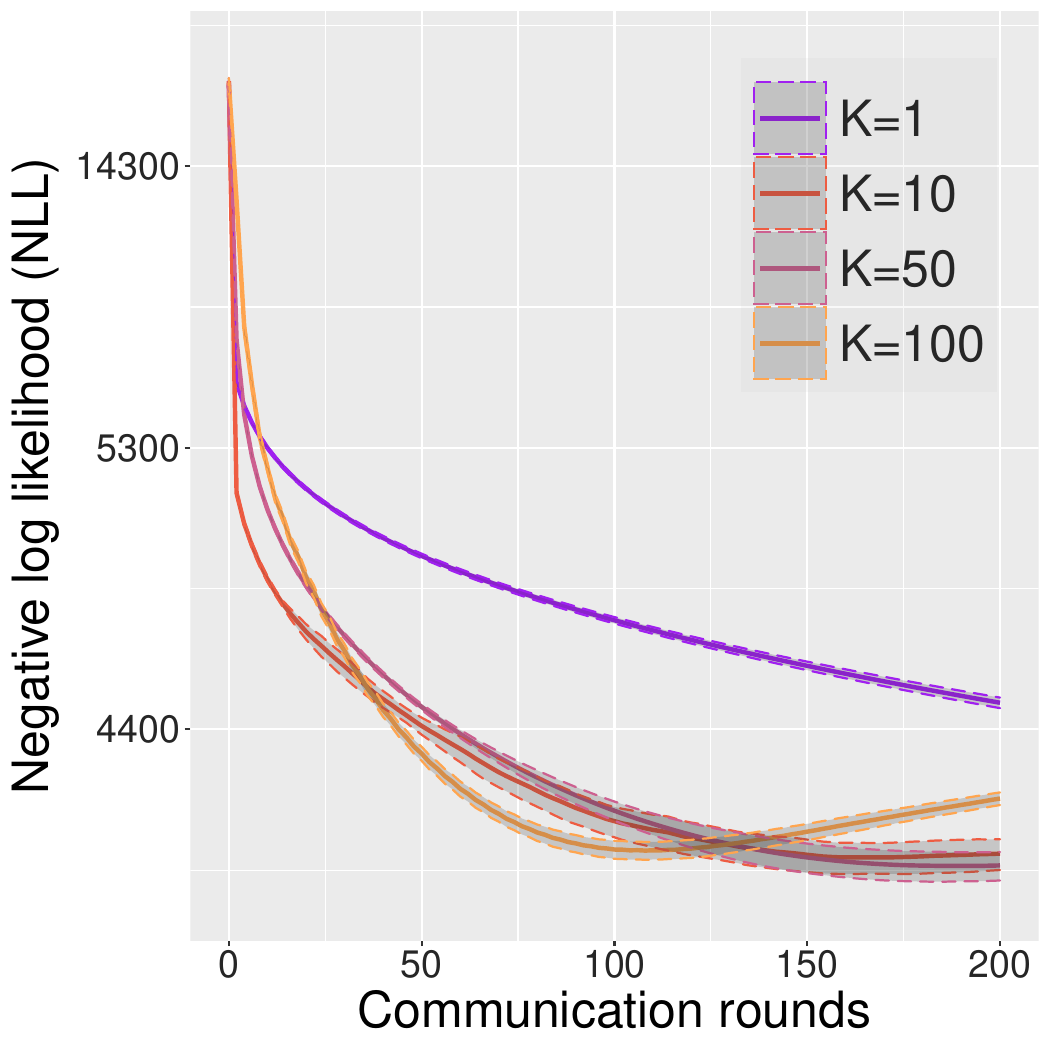}
    \end{minipage}%
    }%
  \vskip -0.1in
  \caption{The impact of leapfrog steps $K$ on FA-HMC applied on the Fashion-MNIST dataset. The shaded error represents the standard deviation based on 5 independence runs.}
  \label{figure:Fashion_HMC_eb}
\end{figure*}

\begin{figure*}[htbp]
    \centering
    \subfigure[Accuracy]{
    \begin{minipage}[t]{0.24\linewidth}
    \centering
    \label{fig:local-accu_eb}
    \includegraphics[width=\linewidth]{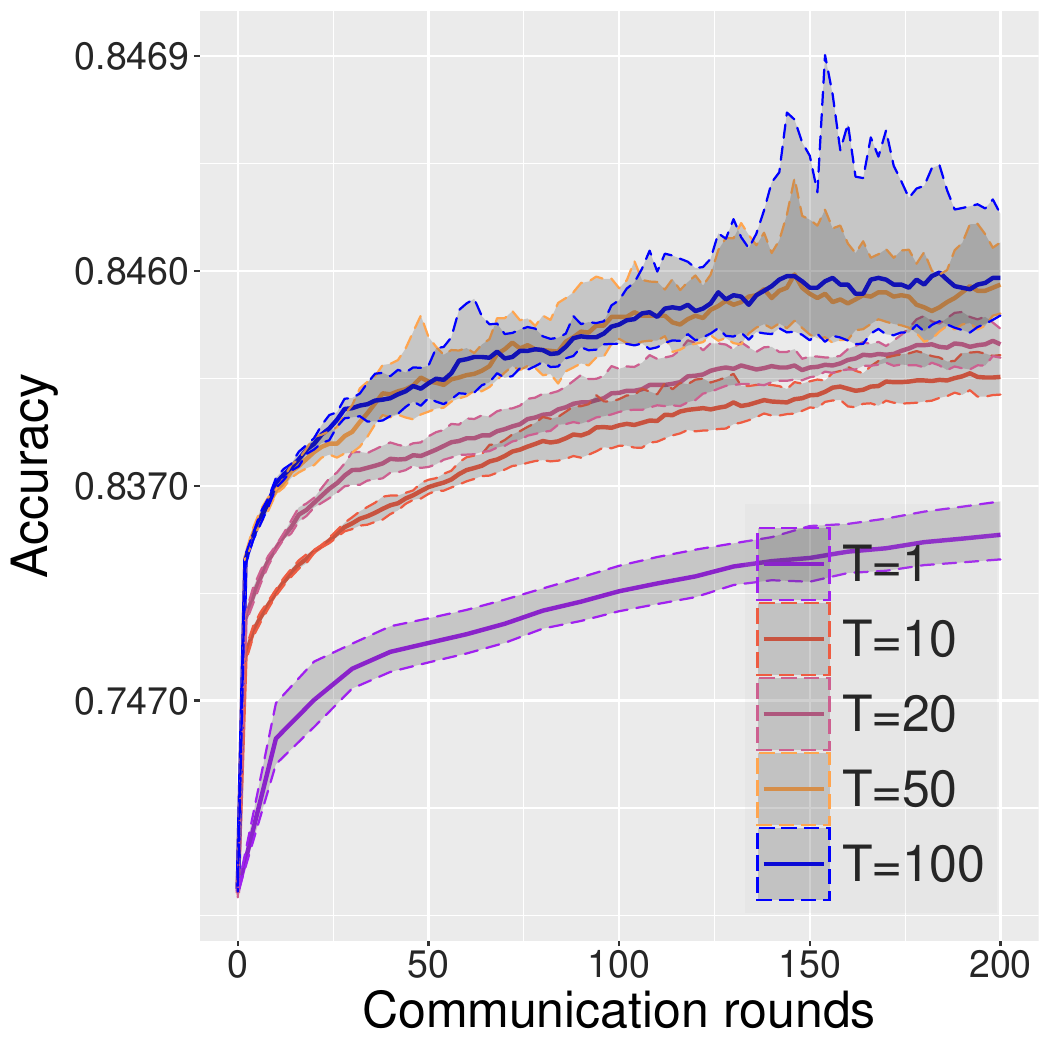}
    \end{minipage}%
    }%
    \subfigure[BS]{
    \begin{minipage}[t]{0.24\linewidth}
    \centering
    \label{fig:local-brier_eb}
    \includegraphics[width=\linewidth]{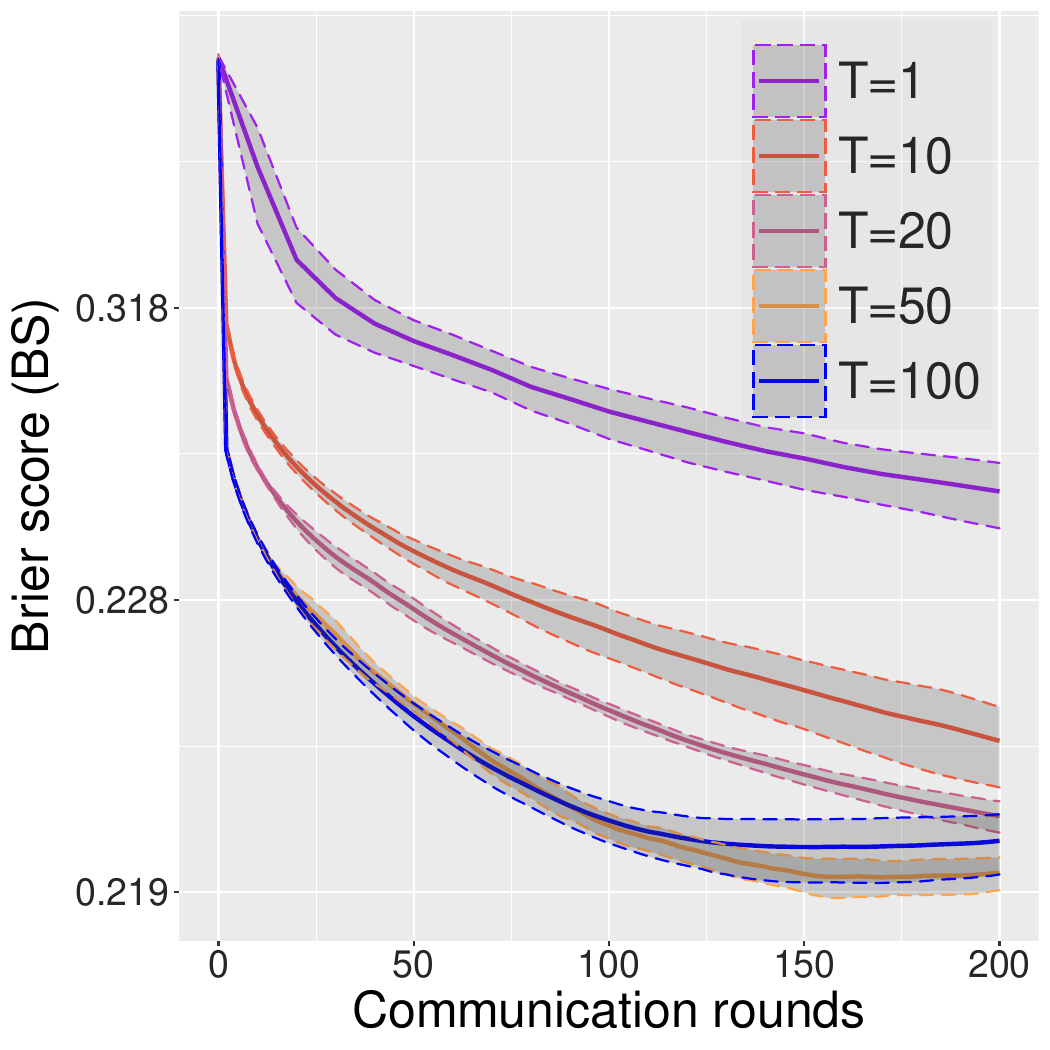}
    \end{minipage}%
    }%
    \subfigure[ECE]{
    \begin{minipage}[t]{0.24\linewidth}
    \centering
    \label{fig:local-ECE_eb}
    \includegraphics[width=\linewidth]{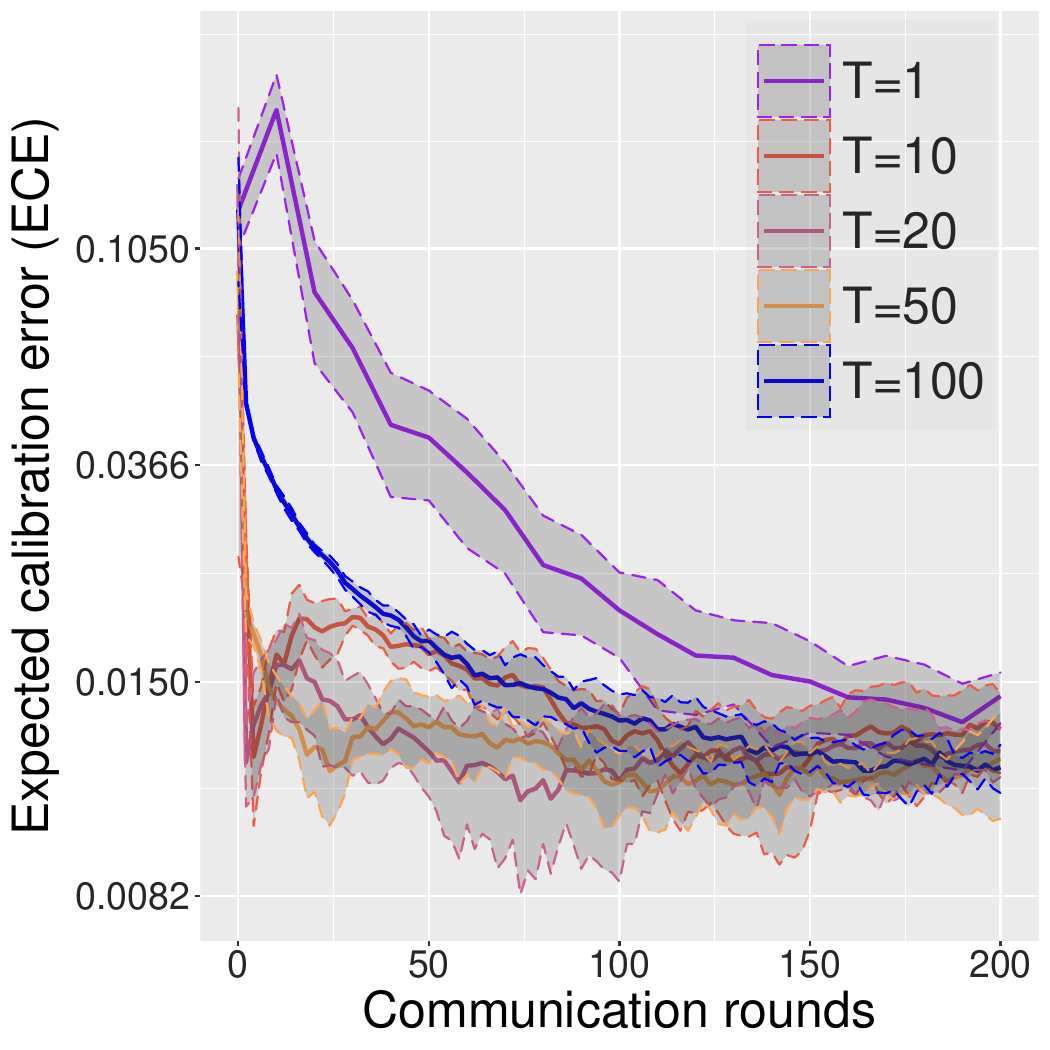}
    \end{minipage}%
    }%
    \subfigure[NLL]{
    \begin{minipage}[t]{0.24\linewidth}
    \centering
    \label{fig:local-NLL_eb}
    \includegraphics[width=\linewidth]{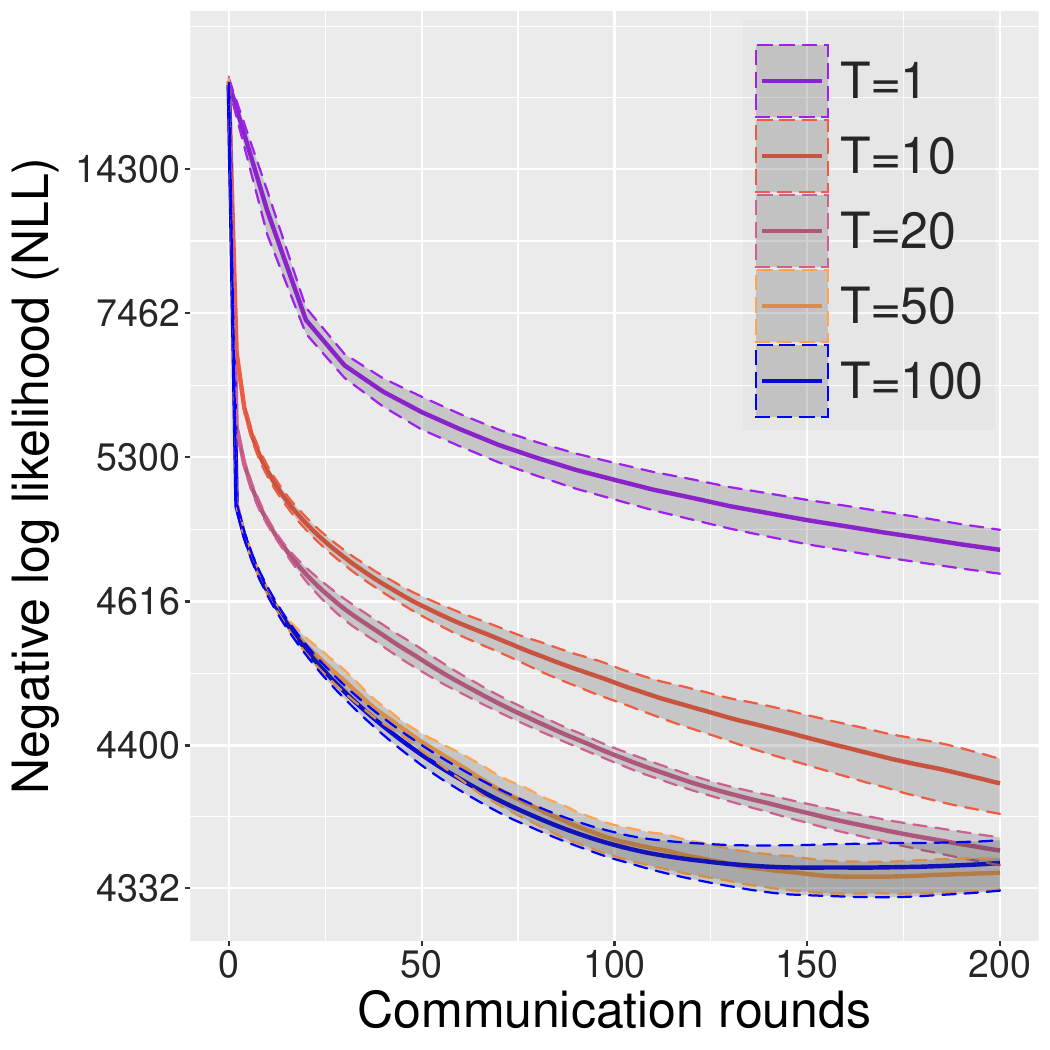}
    \end{minipage}%
    }%
  \vskip -0.1in
  \caption{The impact of local steps $T$ on FA-HMC applied on the Fashion-MNIST dataset. The shaded error represents the standard deviation based on 5 independence runs}
  \label{figure:Fashion_local_eb}
\end{figure*}

\bibliographystyle{jasa3}
\bibliography{reference}
